\documentclass{article}

\usepackage[preprint]{neurips_2026}

\usepackage[utf8]{inputenc} % allow utf-8 input
\usepackage[T1]{fontenc}    % use 8-bit T1 fonts
\usepackage{hyperref}       % hyperlinks
\usepackage{url}            % simple URL typesetting
\usepackage{booktabs}       % professional-quality tables
\usepackage{amsfonts}       % blackboard math symbols
\usepackage{nicefrac}       % compact symbols for 1/2, etc.
\usepackage{microtype}      % microtypography
\usepackage{xcolor}         % colors
\usepackage{enumitem}
\usepackage{multirow}
\usepackage{placeins}

\usepackage{siunitx}
\usepackage{pgfplots}
\pgfplotsset{compat=1.18}

\usepackage{bm}
\newcommand{\br}{\bm{r}}

\newcommand{\bI}{\bm{I}}
\newcommand{\bu}{\bm{u}}
\newcommand{\bz}{\bm{z}}
\newcommand{\bx}{\bm{x}}
\newcommand{\by}{\bm{y}}
\newcommand{\bw}{\bm{w}}
\newcommand{\bb}{\bm{b}}
\newcommand{\bp}{\bm{p}}
\newcommand{\bP}{\bm{P}}
\newcommand{\bR}{\bm{R}}
\newcommand{\bA}{\bm{A}}

\newcommand{\bg}{\bm{g}}
\newcommand{\bk}{\bm{K}}
\newcommand{\bQ}{\bm{Q}}

\newcommand{\bDelta}{\bm{\Delta}}

\newcommand{\melo}{\textup{\textsc{MELO}}}

\usepackage{amsmath, amssymb, amsthm}

\newtheorem{theorem}{Theorem}

\newtheorem{lemma}[theorem]{Lemma}
\newtheorem{corollary}[theorem]{Corollary}
\newtheorem{remark}{Remark}

\newcommand{\appendixtheorems}{%
  \setcounter{theorem}{0}%
  \setcounter{assumption}{0}%
  \setcounter{remark}{0}%
  \renewcommand{\thetheorem}{A.\arabic{theorem}}%
  \renewcommand{\thelemma}{A.\arabic{theorem}}%
  \renewcommand{\thecorollary}{A.\arabic{theorem}}%
  \renewcommand{\theassumption}{A.\arabic{assumption}}%
  \renewcommand{\theremark}{A.\arabic{remark}}%
}

\usepackage{algorithm}
\usepackage{algorithmic}

\usepackage{graphicx}

\usepackage{tikz}
\usetikzlibrary{positioning}
\usetikzlibrary{calc}

\title{Hedging Memory Horizons for Non-Stationary Prediction via Online Aggregation}

\author{%
  Yutong Wang \\
  Department of Statistics, LSE \\
  London, UK \\
  \texttt{y.wang487@lse.ac.uk} \\
  \And
  Yannig Goude \\
  \'{E}lectricit\'{e} de France R\&D \\
  Palaiseau, France \\
  \texttt{yannig.goude@edf.fr} \\
  \And
  Qiwei Yao \\
  Department of Statistics, LSE \\
  London, UK \\
  \texttt{q.yao@lse.ac.uk} \\
}

\begin{document}

\maketitle

\begin{abstract}
    We study online prediction under distribution shift, where
    inputs arrive chronologically and outcomes are revealed only after prediction. In this setting, 
    predictors must remain stable in quiet regimes yet adapt when regimes shift,
    and the right adaptation memory is unknown in advance. % the problem
    We propose \melo{} (Memory-hedged Exponentially Weighted Least-Squares Online
    aggregation), a model-agnostic method that hedges across adaptation scales: it wraps any non-anticipating base-predictor pool
    with exponentially weighted least-squares (EWLS) adaptation experts at multiple
    forgetting factors, and aggregates raw and EWLS-adapted forecasts with MLpol which is 
    a parameter-free online aggregation rule. % method
    Under boundedness conditions, we establish deterministic oracle inequalities showing that it competes with both the best raw predictor and the best bounded, time-varying affine combinations of the base predictions, up to a path-length-dependent tracking cost and a sublinear aggregation overhead.  % theory
    We evaluate \melo{} on French national
    electricity-load forecasting through the COVID-19 lockdown using no
    regime indicators, lockdown dates, or policy covariates. 
    \melo{} reduces overall
    RMSE by \(34.7\%\) relative to base-only MLpol 
    and achieves lower overall RMSE than a TabICL reference supplied with an external COVID
    policy-response covariate.  
    \melo{} requires only lightweight per-step recursive updates without model retraining. % empirical performance and computation efficiency
\end{abstract}

% -----------------------------------------------------------------------------
\section{Introduction}
% -----------------------------------------------------------------------------
Operational forecasters can fail abruptly when reality changes faster than
their update mechanism. The COVID-19 lockdown sharply disrupted French national
electricity load within days. Models calibrated on years of pre-pandemic data
incurred large short-run errors, while reliable retraining required enough
post-shift labels and therefore could not provide an immediate correction.
This pattern recurs across deployed forecasting settings, including short-term
energy load, retail demand, financial quantities, and user behaviour. In these
settings, inputs arrive chronologically and models are evaluated on future
periods; the deployment distribution may also shift after training. Throughout,
\emph{causal} means non-anticipating in time: the prediction at time \(t\) may
use only inputs and labels revealed before \(t\), and not future outcomes. The
appropriate adaptation memory is therefore both unknown in advance and itself
non-stationary: a horizon well tuned to one regime can be badly miscalibrated
for the next.

Existing approaches address only part of this trade-off.
Offline-trained models, such as tree ensembles \citep{chen2016xgboost,ke2017lightgbm}, GAMs \citep{hastie1986generalized, wood2017generalized}, deep tabular
networks \citep{gorishniy2021revisiting}, and tabular foundation models
\citep{hollmann2023tabpfn,qu2025tabicl},
encode rich historical structure but cannot adapt without explicit retraining
or reconditioning, which requires labels not yet available immediately after a
regime shift.
Adaptive filters update online, including recursive least-squares methods \citep{hayes1996statistical,sayed2003fundamentals}, and Kalman-style forecast corrections for load forecasting \citep{obst2021adaptive,de2022state}.
However, a single adaptation memory typically governs their behaviour: long
memory is stable but slow at breaks, whereas short memory adapts quickly but is
noisy in quiet periods.
Adaptive Kalman variants
\citep{huang2018vbakf,devilmarest2024viking} attempt to infer adaptation
scales online, but each does so within its own state-space parametrisation and
update mechanism.
Classical forecast combination and online expert aggregation
\citep{bates1969combination,stock2004combination,cesa2006prediction,
gaillard2014second,devaine2013forecasting} weight heterogeneous forecasters by
realised loss. This gives robustness to the available expert pool, but the aggregate cannot move outside their convex hull with
simplex weights over raw forecasts. A shared post-shift bias is therefore inherited rather than
corrected.

We propose \melo{} (Memory-hedged Exponentially Weighted Least-Squares Online
aggregation), a causal adaptation layer for online prediction under
deployment-time shift. Rather than committing to a single memory horizon,
\melo{} exposes several horizons as competing experts and lets realised
forecasting loss decide how much to trust each. Concretely, the entire
base-prediction vector is fed into a time-varying affine combination layer.
Unlike convex aggregation over the raw forecasts, this layer can move outside
their convex hull when post-shift bias must be corrected. It maintains online
corrections at multiple adaptation speeds in parallel, while keeping the raw
forecasts as conservative fallbacks. The choice of which correction speed to
trust is then deferred to a loss-driven outer aggregation rule. The resulting
procedure requires no regime labels, change-point annotations, or base-model
retraining.

\paragraph{Contributions.}
\begin{itemize}[leftmargin=*]
    \item We introduce \melo{}, a model-agnostic causal adaptation layer motivated by residual-aware correction. It treats base forecasts as nonlinear representations with potentially complementary errors, wraps them
    with multi-scale EWLS experts, and aggregates raw and corrected forecasts
    with MLpol using lightweight recursive updates.

    \item We prove deterministic oracle inequalities, under explicit boundedness/stability conditions, showing that \melo{} competes with both the best raw predictor and the best correction memory scale in the EWLS grid, balancing prediction robustness with the cost of adapting to a changing affine comparator. % finite-memory error against comparator drift.

    \item On COVID-era French electricity-load forecasting, \melo{} outperforms base-only aggregation and matched adaptive-filter baselines, and also beats a stronger-information TabICL+GRI reference on overall RMSE.
\end{itemize}

% -----------------------------------------------------------------------------
\section{Related Work}
% -----------------------------------------------------------------------------

\paragraph{Online aggregation and prediction with expert advice.}
Prediction with expert advice gives an online-learning perspective on classical
forecast combination~\citep{bates1969combination,timmermann2006forecast,cesa2006prediction}:
a forecaster combines a finite pool of experts sequentially and is evaluated by
regret against hindsight benchmarks, from the best fixed expert~\citep{cesa2006prediction},
to shifting experts~\citep{herbster1998tracking,bousquet2002tracking} and, more
broadly, a time-varying comparator sequence in dynamic
regret~\citep{zinkevich2003online,cesabianchi2012mirror,zhang2018adaptive}. Second-order
aggregation methods such as BOA~\citep{wintenberger2017optimal} and
MLpol~\citep{gaillard2014second} use data-dependent updates that adapt to the
realised loss sequence; we use MLpol as a parameter-free aggregation rule.
Conceptually closer to our multi-scale design, \citet{bousquet2002tracking}
introduce mixing past posteriors as an implicit form of memory-scale hedging,
and \citet{cesabianchi2012mirror} provide a unified mirror-descent analysis of
fixed-share, weight-sharing, and related tracking schemes; MELO makes the
memory hedging explicit by exposing a grid of EWLS adaptation horizons as
competing experts. Online expert aggregation has also been applied to financial
nonstationarity~\citep{miao2023online}, where a multiplicative-weights
aggregator over a heterogeneous model pool (linear, tree, neural) achieves
robust performance through COVID; MELO differs in hedging across memory
horizons of the same learner family rather than across model classes, and in
maintaining recursive expert state rather than refitting on rolling windows.
Within aggregation theory, Q-aggregation~\citep{dai2012deviation,dai2014aggregation}
and $\phi$-divergence FTRL~\citep{alquier2021non} sharpen the
deviation behaviour and unboundedness handling of EWA in the batch and
unbounded-loss settings respectively; MLpol's polynomial potential plays an
analogous role online. Our contribution lies at the level of the expert pool:
we augment raw predictors with a multi-scale family of online residual-correction
experts, so the aggregate can choose, by realised squared loss, between stable
historical predictors and rapidly adapting combinations.

\paragraph{Tabular prediction under temporal distribution shift.}
Modern tabular prediction is dominated by strong tree ensembles
\citep{chen2016xgboost,ke2017lightgbm}, with growing interest in deep tabular
models and tabular foundation models
\citep{gorishniy2021revisiting,hollmann2023tabpfn,qu2025tabicl}. Much of this 
literature evaluates on random splits, where training and test samples are 
approximately exchangeable; the recent chronological benchmark TabReD
\citep{rubachev2024tabred} shows that this can hide large performance gaps 
under temporal distribution shift. Our goal is orthogonal to tabular 
architecture design: any causal tabular predictor can be used as a base 
expert, while the EWLS expert learns online combinations from realised 
forecast errors.

\paragraph{Concept drift and deployment-time adaptation.}
Streaming and concept-drift methods address non-stationarity through sliding 
windows, drift detection, online ensembles, or adaptive retraining
\citep{gama2014survey,bifet2007learning,gomes2017adaptive}. Test-time and 
deployment-time adaptation methods modify model statistics, representations, or 
predictions during inference, with or without labels
\citep{kim2025battling,liu2025improving,lee2025lightweight}. We study the 
label-revealed-after-prediction setting, where adaptation can be driven directly 
by realised forecast errors. The proposed EWLS layer therefore provides a 
model-agnostic supervised adaptation mechanism, without committing to a single 
drift detector, retraining rule, or hand-chosen window schedule. Unlike test-time adaptation methods that modify base-model internals, \melo{} leaves base predictors untouched and operates entirely as an outer aggregation layer.

\paragraph{Adaptive filtering and adaptive Kalman methods.}
EWLS/RLS adaptive filters track time-varying linear relationships through
least-squares recursions \citep{hayes1996statistical,sayed2003fundamentals}, while Kalman filtering
gives the classical recursive prediction--correction framework for linear
state-space models \citep{kalman1960new}.
Adaptive Kalman-style variants can estimate process uncertainty or adaptation scales online
\citep{huang2018vbakf,devilmarest2024viking}, and related state-space corrections 
have also been developed for electricity-load forecasting through 
Kalman-style updates of forecasting-model coefficients 
\citep{obst2021adaptive,de2022state}. These approaches are usually tied to a 
particular state-space or model-update parametrisation. We instead combine 
the entire base-prediction pool, run several fixed-scale filters as experts 
in parallel, and use an outer loss-driven aggregator to hedge the unknown
adaptation time scale.

% -----------------------------------------------------------------------------
\section{Methodology} 
\label{sec:method}
% -----------------------------------------------------------------------------

\begin{figure}[ht]
    \setlength{\abovecaptionskip}{0pt}   
    \setlength{\belowcaptionskip}{0pt}
    \centering
    \includegraphics[width=0.95\linewidth]{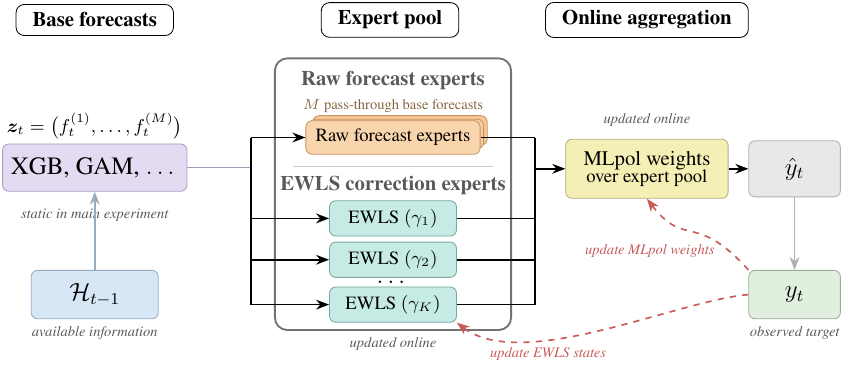}
    \caption{Overview. Causal base predictions \(\bz_t\) enter the
    expert pool directly and through \(K\) EWLS experts with different forgetting
    factors. MLpol aggregates the \(M+K\) raw and adapted candidates into
    \(\hat y_t\); after \(y_t\) is observed, both EWLS states and MLpol weights are
    updated online.}
    \label{fig:architecture}
\end{figure}

We study sequential one-step-ahead prediction. Let \(\mathcal H_{t-1}\) denote all information
available to the learner before $y_t$ is revealed.
By convention, $\mathcal{H}_{t-1}$ subsumes the entire past (previous covariates, predictions, and revealed outcomes) as well as any current period inputs known prior to predicting $y_t$, such as calendar variables, lagged outcomes, or exogenous covariates whose values for time $t$ are already observed or forecast.

At round $t$, the learner uses $\mathcal{H}_{t-1}$ to form a prediction $\hat y_{t}$, then observes \(y_t\) and
incurs the one-step squared loss $\ell_t(\hat y_t) = (\hat y_t-y_t)^2$. For the prediction sequence \(\hat y=(\hat y_1,\ldots,\hat y_T)\), the objective is to minimize the cumulative loss 
\begin{equation}
    \label{eq:loss}
    L_T(\hat y)=\sum_{t=1}^T \ell_t(\hat y_t)
    = \sum_{t=1}^T(\hat y_t-y_t)^2 .
\end{equation}

Let \( z_{t,m} = f^{(m)}_t(\mathcal H_{t-1}) \) be the prediction produced by base forecaster $m \in [M]$,
and collect the base predictions as
$\bz_t \equiv (z_{t,1},\ldots,z_{t,M})^\top \in \mathbb{R}^M$.
This formulation is deliberately general:
a rolling-retrain model corresponds to an $f^{(m)}_t$ that is refitted
each round using training data contained in $\mathcal H_{t-1}$;
a static model corresponds to a fixed map $f^{(m)}_t \equiv f^{(m)}$
evaluated on the current-period inputs encoded in $\mathcal H_{t-1}$;
and physics-based or rule-based forecasters arise as deterministic
mappings of $\mathcal H_{t-1}$.

The base forecasters are treated as black-box causal prediction streams. We only require their forecasts to be available before $y_t$. They may be fixed or causally updated, provided the expert identities remain coherent across rounds. In the main experiments they are trained once and kept fixed, isolating MELO's adaptation from any benefit of base retraining.

\paragraph{Step 1. EWLS experts.} Put \(\tilde{\bz}_t=(\bz_t^\top,1)^\top\). For a 
forgetting factor $\gamma \in [1/2,1)$, the EWLS expert gives a prediction $\check{y}_{t}^{(\gamma)} = \tilde{\bz}_{t}^\top \bw_{t}^{(\gamma)} $, where the coefficient vector is the exponentially weighted ridge solution
based only on past observations:
\begin{equation}
    \bw_{t}^{(\gamma)} \in \arg\min_{\bw} 
    \Big\{ \sum_{s=1}^{t-1} \gamma^{(t-1)-s} (y_s - \bw^\top \tilde{\bz}_s)^2 + \gamma^{t-1} \delta {\|\bw\|}_2^2  \Big\}.
    \label{eq:ewls}
\end{equation}
Here \(\delta>0\) is a small ridge ensuring the discounted Gram matrix is invertible.

Thus each EWLS expert learns a time-varying affine combination of the base
predictions, with an intercept. 
The constant $1$ appended to \(\bz\) gives each EWLS expert a time-varying intercept, allowing it to 
absorb level shifts in \(y_t\), such as a sudden drop in baseline demand, without forcing the 
adaptation through the base coefficients.
The forgetting factor controls the adaptation scale via nominal memory length
$h(\gamma)={1}/{(1-\gamma)}$: smaller \(h(\gamma)\) reacts faster but is
noisier, whereas larger \(h(\gamma)\) is more stable but adapts more slowly.
Since the appropriate adaptation scale is unknown and may itself change over
time, \melo{} maintains \(K\) EWLS experts with forgetting factors
$\Gamma=\{\gamma_1,\ldots,\gamma_K\}$.

\paragraph{Step 2. MLpol aggregation over base and EWLS experts.} At time $t$, \melo{} forms a pool of $N=M+K$ candidate predictions by concatenating the raw base predictors and the EWLS experts:
\[
\tilde \by_{t} \equiv (\tilde y_{t,1}, \cdots, \tilde y_{t,N})^\top
=\big( z_{t,1}, \cdots, z_{t,M},
\check{y}_{t}^{(\gamma_1)}, \cdots,
\check{y}_{t}^{(\gamma_K)}\big)^\top.
\]
The final predictor is MLpol aggregation~\citep{gaillard2014second}
$\hat y_{t} = \bp_t^\top \tilde \by_{t}$. 
MLpol assigns weights according to positive cumulative pseudo-regrets: 
\begin{equation} \label{eq:mlpol-update}
    p_{t,j} = \frac{[R_{t-1,j}]_+}{\sum_{i=1}^{N} [R_{t-1,i}]_+}, \qquad 
    R_{t-1,j} = \sum_{s=1}^{t-1} 2(\hat y_s - y_s) (\hat y_s - \tilde y_{s,j}),
\end{equation}
where \([x]_+=\max\{x,0\}\). If all positive parts vanish, uniform weights apply. 
This is the gradient-trick form for squared loss: an expert receives positive pseudo-regret
when it would have moved the aggregate prediction toward the observed target.
Raw base predictors are included directly in the pool, so the aggregate can
fall back to them when online correction is unnecessary.
Each candidate prediction is an affine function of $\tilde \bz_t$, so the aggregate can equivalently be analysed against time-varying affine combinations.

\paragraph{Grid design.}
We place forgetting factors on a geometric grid in nominal memory length,
\(h_k=1/(1-\gamma_k)\), over \([h_{\min},h_{\max}]\). This uses only
\(K=O(\log(h_{\max}/h_{\min}))\) correction experts while ensuring that, if the
oracle memory length lies in the range, some grid point approximates it up to a
constant factor. Section~\ref{sec:theory} makes this precise.

\paragraph{Implementation and Complexity.} EWLS solutions are updated by the standard recursive
least-squares (RLS) recursion, with causal initialization and a small covariance
inflation for numerical stability, see Algorithm~\ref{alg:main} in Appendix~\ref{app:algo}. Let \(d=M+1\). Each EWLS expert costs $O(d^2)$ per step via the RLS recursion, giving 
total per-step cost $O(Kd^2+N)$. This is independent of $T$ and of the base predictors' own training cost.

% -----------------------------------------------------------------------------
\section{Theoretical guarantee}
\label{sec:theory}
% -----------------------------------------------------------------------------

We summarise the main deterministic, sequence-wise guarantees that motivate
the two design choices: maintaining several EWLS time scales, and aggregating
raw base forecasts together with their online corrections. The statements in
this section are informal main-text versions; the formal oracle inequalities,
including constants, clipping conventions, and proofs, are given in
Appendix~\ref{app:deterministic-theory}.

\paragraph{Setup.}
Write \(L_T(a_{1:T})=\sum_{t=1}^T(a_t-y_t)^2\) for any 
prediction sequence $a_{1:T}=(a_t)_{t=1}^T$. For a time-varying affine combination path 
\(\bu_{1:T}=(\bu_1,\ldots,\bu_T)\), define
\[
    L_T(\bu_{1:T})
    = \sum_{t=1}^T \bigl(\tilde{\bz}_t^\top \bu_t-y_t\bigr)^2,
    \qquad
    P_T(\bu_{1:T})
    = \sum_{t=2}^T \|\bu_t-\bu_{t-1}\|_2 .
\]
The path length \(P_T\) measures the total movement of the comparator path and 
accommodates both gradual drift and abrupt changes. Write 
\(\mathcal U_T(R)=\{\bu_{1:T}:\|\bu_t\|_2\le R\text{ for all }t\}\).

\paragraph{Assumptions.}
We assume bounded data, \(\|\tilde{\bz}_t\|_2\le B_z\) and \(|y_t|\le B_y\), and 
bounded EWLS iterates along the realised sequence: for the EWLS expert with 
forgetting factor \(\gamma\), the unprojected forgotten-RLS coefficients satisfy
$\|\bw_t^{(\gamma)}\|_2 \le R$ for $t\in [T]$, applied to the relevant EWLS expert(s) in each statement below. This is an explicit stability condition on the realised recursion, not a consequence of forgetting alone.

\paragraph{Single-scale tracking.}
Let \(\check y^{(\gamma)}_{1:T} = (\check y^{(\gamma)}_{t})_{t=1}^T\) denote the EWLS prediction sequence generated by the 
idealised recursion in \eqref{eq:ewls}. Denote $h_{\gamma} = 1/(1-\gamma)$.

\begin{theorem}[Single-scale tracking; informal, see Theorem~\ref{thm:rls-dynamic-app}]
\label{thm:rls-dynamic}
Under the assumptions above, for any \(\gamma\in[1/2,1)\) and any comparator 
path \(\bu_{1:T}\in\mathcal U_T(R)\),
\[
    L_T\bigl(\check y^{(\gamma)}_{1:T}\bigr)
    \;\le\;
    L_T(\bu_{1:T})
    \;+\;
    \frac{C_1\,dT}{h_\gamma}
    \;+\;
    C_2\,h_\gamma\,P_T(\bu_{1:T})
    \;+\;
    \text{l.o.t.},
\]
where \(C_1,C_2\) hide only boundedness and initialisation constants, and l.o.t. collects $\delta$-initialisation and  logarithmic-in-$h$ terms.
\end{theorem}

The two leading terms display a stability--adaptivity trade-off: short memory 
reduces the tracking term \(h_\gamma P_T\) but inflates the finite-memory term 
\(dT/h_\gamma\); long memory does the opposite. When \(P_T(\bu_{1:T})>0\), the 
balancing scale is
\(h^\star \asymp \sqrt{dT/P_T(\bu_{1:T})}\),
which depends on the \emph{unknown} amount of comparator movement; when 
\(P_T(\bu_{1:T})=0\), the bound reduces to $ L_T(\bu_{1:T}) + C_1 dT/h_\gamma
$, which favours the longest available memory. 

\paragraph{Multi-scale aggregation.}
\melo{} sidesteps the choice of \(h\) by exposing a grid of EWLS scales as 
competing experts and letting MLpol decide by realised loss. 
Let \(\Gamma\subset[1/2,1)\) be a finite-forgetting grid with
\(|\Gamma|=K\), and let
\(\mathcal E_\Gamma=\{f^{(1)},\ldots,f^{(M)}\}\cup\{\check y^{(\gamma)}:\gamma\in\Gamma\}\).
Let \(f^{(j)}_{1:T}\) denote the raw prediction sequence of base expert \(j\) and let \(\hat y_{1:T}\) denote the resulting MLpol aggregate over
\(\mathcal E_\Gamma\).

\begin{theorem}[\melo{} oracle inequality; informal, see Theorem~\ref{thm:mlpol-oracle-app}]
\label{thm:mlpol-oracle}
Under the same assumptions and the bounded-loss MLpol convention formalised in
Appendix~\ref{app:deterministic-theory}, suppressing logarithmic and initialisation terms, the MLpol aggregate over \(\mathcal E_\Gamma\) satisfies, for every 
comparator path \(\bu_{1:T}\in\mathcal U_T(R)\),
\[
    L_T(\hat y_{1:T})
    \;\le\;
    \min\!\left\{
        \underbrace{\min_{1\le j\le M} L_T(f_{1:T}^{(j)})}_{\text{best raw predictor}},
        \;\;
        \underbrace{L_T(\bu_{1:T})
        + \min_{\gamma\in\Gamma}\!\left(
            \tfrac{C_1 dT}{h_\gamma}
            + C_2\,h_\gamma\,P_T(\bu_{1:T})
        \right)}_{\text{best EWLS scale tracking } \bu_{1:T}}
    \right\}
    \;+\;
    C_3\sqrt{(M+K)T},
\]
where \(C_1,C_2,C_3\) hide only boundedness and initialisation constants.
\end{theorem}

The aggregate competes with two benchmarks simultaneously. The first is a safety comparison: the bounded-loss MLpol aggregate over the
enlarged pool remains competitive with the best individual raw predictor up to
the sublinear aggregation overhead. The 
second captures distribution shifts through the best online affine correction at 
the best memory scale in the grid.

\paragraph{Geometric grids.}
A geometric grid over \([h_{\min},h_{\max}]\) with ratio \(\rho\) contains a 
point within factor \(\rho\) of \(h^\star\) whenever \(h^\star\) lies in the 
range. Therefore, up to the suppressed logarithmic and initialisation terms, the
inner minimum in Theorem~\ref{thm:mlpol-oracle} is of order
\(C_\rho\sqrt{dT\,P_T(\bu_{1:T})}\) when \(P_T(\bu_{1:T})>0\), using only 
\(K=O(\log(h_{\max}/h_{\min}))\) experts. The precise corollary is in
Appendix~\ref{app:deterministic-theory}.

% -----------------------------------------------------------------------------
\section{Forecasting French daily electricity load}
\label{sec:experiments}
% -----------------------------------------------------------------------------

\paragraph{Task and protocol.}
We evaluate one-step-ahead forecasting of French national electricity load, a
standard operational short-term load forecasting task
\citep{devaine2013forecasting,obst2021adaptive,de2022state}. We use the public
RTE (French TSO) daily load dataset from 2012-01-01 to 2021-01-15 and refer to 
this benchmark as \textbf{RTE-FR Load}. Base models are trained on 2012--2018. 
Online methods are then evaluated sequentially from 2019-01-01 to 2021-01-15 
($T=746$ days), with each update using only true values available after the 
corresponding prediction. We report RMSE overall and over three regimes: 
\emph{pre-lockdown} (2019-01-01 to 2020-03-16, 441 days), \emph{lockdown} 
(2020-03-17 to 2020-05-11, 56 days), and the \emph{post-first-lockdown period} 
(2020-05-12 to 2021-01-15, 249 days).

\paragraph{Models and tuning.}
The base pool contains $M=7$ heterogeneous forecasters:
Lag-1, Linear(Ridge) \citep{hoerl1970ridge},
XGBoost \citep{chen2016xgboost},
LightGBM-GBDT \citep{ke2017lightgbm},
GAM \citep{hastie1986generalized,wood2017generalized},
ResNet \citep{he2016resnet,gorishniy2021revisiting},
and FT-Transformer \citep{gorishniy2021revisiting}.
All are trained once on 2012--2018 data with a shared feature set including lagged 
load, temperature, calendar, and holiday variables. To isolate algorithmic 
adaptation from weather-forecast quality, all methods use realised-temperature
covariates in this controlled benchmark; an operational deployment would replace 
them by temperature forecasts, see Appendix~\ref{app:data-splits}. Hyperparameters 
use published defaults or tail-of-training validation, with no tuning on the 
test period, see Appendix~\ref{app:hp-tune}. 

The EWLS layer uses $K=16$ correction experts: \(15\) finite-forgetting factors 
are placed on a geometric grid over $\gamma\in[0.950,0.9998]$, equivalently 
nominal scales $h(\gamma)=1/(1-\gamma)\in[20,5000]$, and one no-forgetting 
endpoint $\gamma=1$ is appended. The theory applies to the finite-forgetting 
experts; the \(\gamma=1\) endpoint is included empirically as a static 
long-memory reference. Numerical-stability parameters, including the 
covariance-inflation scale $\varepsilon_0=10^{-8}$, are selected by
walk-forward validation on 2018 only, see Appendices~\ref{app:ewls-impl} and~\ref{app:eps_selection}.

\paragraph{Comparisons.}
Our main ablations are MLpol on the base pool only, MLpol on the EWLS pool only,
and \melo{}, which is MLpol on the combined Base+EWLS pool. We also report the best
single EWLS expert selected in hindsight on overall test RMSE as a within-method
oracle baseline. 
External references include individual static baselines, 
tabular foundation-model baselines run daily with an expanding in-context training window (TabPFN and TabICL), stronger-information variants augmented with the Oxford Government Response Index (GRI; \citealt{hale2021global}), and adaptive Kalman-filter alternatives tuned under the same 2018 walk-forward protocol.

% -----------------------------------------------------------------------------
\subsection{Main performance and external baselines}
\label{sec:main_result}
% -----------------------------------------------------------------------------

\begin{table}[ht]
    \caption{
    Per-regime and overall RMSE (MW) on the French electricity-load test period.
    \textbf{Bold}: best per column among non-hindsight methods;
    \underline{underline}: second best among non-hindsight methods.
    Hindsight rows (italicised) are non-deployable EWLS reference benchmarks and
    do not compete for highlighting. ``Best single EWLS (hindsight)'' selects one
    forgetting factor using overall test RMSE. ``Per-regime best EWLS
    (hindsight)'' selects the in-segment best forgetting factor separately for
    each regime ($\gamma=0.9990, 0.9773, 0.9997$), using regime boundaries and
    future losses, and is therefore not realisable \emph{a priori}.
    The expanded table with all individual base baselines, all single EWLS
    scales, and additional external references is in
    Appendix~\ref{app:rte_detailed_results}. ``+GRI'' denotes an additional
    regime-informed covariate; \melo{} uses no such signal.
    }
    \label{tab:main_results}
    \centering
    \small
    \setlength{\tabcolsep}{5pt}
    \begin{tabular}{lcccc}
        \toprule
        Method & Pre-lockdown & Lockdown & Post-lockdown & Overall \\
        \textit{Regime length (days)} & \textit{441} & \textit{56} & \textit{249} & \textit{746} \\
        \midrule
        GAM (best static overall)        & 946.6 & 3201.3 & 1076.2 & 1298.3 \\
        TabICL (online)                  & 670.3 & 2943.6 &  687.4 & 1036.2 \\
        TabICL (online, +GRI)            & 668.0 & 1132.6 &  699.5 &  723.4 \\
        \midrule
        MLpol on Base only               & 690.5 & 2452.7 &  907.0 & 1004.0 \\
        \textit{Best single EWLS ($\gamma=0.9986$, hindsight)}
                                         & \textit{657.9} & \textit{1168.3}
                                         & \textit{585.4} & \textit{687.5} \\
        \textit{Per-regime best EWLS (hindsight)}
                                         & \textit{656.2} & \textit{998.8}
                                         & \textit{573.6} & \textit{662.8} \\
        MLpol on EWLS only               & \underline{653.9} & \textbf{1073.6}
                                         & \underline{584.1} & \underline{673.2} \\
        \textbf{\melo{} (Base+EWLS)}
                                         & \textbf{623.1} & \underline{1086.1}
                                         & \textbf{579.3} & \textbf{655.8} \\
        \midrule
        RMSE reduction vs. Base only (MW)       & +67.4  & +1366.6 & +327.8 & +348.3 \\
        RMSE reduction vs. Base only (\%)       & +9.8\% & +55.7\%  & +36.1\% & +34.7\% \\
        \bottomrule
    \end{tabular}
\end{table}

Table~\ref{tab:main_results} shows the main empirical message: adding the EWLS
experts to the raw base-forecast pool substantially improves online
aggregation under distribution shift. Relative to MLpol on the base pool alone,
\melo{} reduces RMSE before lockdown ($+9.8\%$), by the largest margin during
the COVID-19 regime break ($+55.7\%$), and still substantially after the first
lockdown period ($+36.1\%$), yielding a $+34.7\%$ overall improvement. A paired
moving-block bootstrap on the daily squared-loss sequences confirms that the
overall margins over both base-only MLpol and TabICL(+GRI) are significant at the
$5\%$ level; details are in Appendix~\ref{app:rte_bootstrap}.

The ablations separate the two roles of the Base+EWLS pool. Base-only
aggregation is reasonably competitive before lockdown, but deteriorates sharply
under the break. EWLS-only aggregation adapts especially well during the short
lockdown segment, where it slightly outperforms the full pool, but the raw base
forecasts remain useful in quieter periods. The full Base+EWLS aggregation
therefore gives the best overall balance: it is the best non-hindsight method
before lockdown, after lockdown, and overall. Within the EWLS family alone, no single forgetting factor matches \melo{}'s overall RMSE. Even a non-deployable per-regime EWLS oracle remains slightly worse overall.

The comparison with TabICL is a stress test rather than a claim that the EWLS
correction layer replaces foundation-model forecasts. The fair same-information
comparison is TabICL (online), which uses the same causal information available to
\melo{}. TabICL (online, +GRI) is instead a stronger-information reference, since the
Oxford Government Response Index changes contemporaneously with the policy
regime. TabICL (online) is substantially worse than \melo{}, whereas TabICL (online, +GRI)
narrows the gap by adding this regime-informed covariate. \melo{} still achieves
lower overall RMSE than this stronger-information reference while using no such
exogenous regime signal.

Finally, the framework is extensible rather than tied to a particular base
pool. Adding the TabICL (online, +GRI) forecast as an additional causal expert improves
\melo{} from $655.8$ to $621.1$ RMSE, and including the same forecast in the
EWLS design vector further improves RMSE to $601.9$, see Appendix~\ref{app:extensibility}.

\begin{table}[t]
    \caption{
    Comparison with adaptive Kalman-filter alternatives under the same 
    tuning protocol. All nuisance hyperparameters are selected on
    2018 out-of-sample data; the 2019--2021 test period is never used for
    selection. Each adaptive-KF layer is combined with the same direct base
    experts by the same outer MLpol aggregator. Full diagnostics and
    filter-only ablations are in Appendix~\ref{app:adaptive_kf}.
    }
    \label{tab:adaptive_kf_main}
    \centering
    \small
    \setlength{\tabcolsep}{5pt}
    \begin{tabular}{lcccc}
        \toprule
        Method & Pre-lockdown & Lockdown & Post-lockdown & Overall \\
        \midrule
        \melo{} (Base+EWLS)        & \textbf{623.1} & \textbf{1086.1} & \textbf{579.3} & \textbf{655.8} \\
        MLpol on Base+VIKING       & 634.0 & 1185.3 & 599.7 & 680.5 \\
        MLpol on Base+IMM-KF       & 636.5 & 1276.7 & 648.6 & 708.7 \\
        MLpol on Base+VBAKF        & 665.1 & 1227.4 & 665.4 & 722.8 \\
        \bottomrule
    \end{tabular}
\end{table}

\paragraph{Adaptive-filter alternatives.}
A natural alternative to a fixed grid of correction scales is a filter that
estimates its own adaptation scale online. We compare against three adaptive
Kalman-filter baselines with complementary process-uncertainty
parameterisations: VIKING~\citep{devilmarest2024viking} (isotropic
log-variance random walk), VBAKF~\citep{huang2018vbakf} (variational
covariance/noise adaptation), and IMM-KF~\citep{blom1988imm} (Bayesian mixing
over a fixed grid of \(Q\)-scales). VIKING and VBAKF are run as small
pools over responsiveness settings and aggregated with the base experts by MLpol; IMM-KF uses
its internal Markov mixture over fixed-\(Q\) modes and is then also exposed to
the same outer MLpol comparison with the direct base experts. All nuisance
hyperparameters are selected by the same 2018 walk-forward protocol used for
\(\varepsilon_0\); the 2019--2021 test period is never used for selection.

Among the matched Base+adaptive-filter pools in Table~\ref{tab:adaptive_kf_main},
\melo{} achieves the lowest RMSE in each of the four reported regimes.
The pairwise differences in overall-test-period RMSE are significant under the paired moving-block bootstrap, 
see Appendix~\ref{app:adaptive_kf}, which also reports the full bootstrap
table, filter-only ablations, and diagnostics helping explain why these
adaptive-\(Q\) mechanisms do not match the explicit multi-scale EWLS +
regret-aggregation separation on this benchmark.

% -----------------------------------------------------------------------------
\subsection{Adaptation diagnostics}
\label{sec:adaptation_diagnostics}
% -----------------------------------------------------------------------------

\begin{figure}[ht]
    \centering
    \setlength{\abovecaptionskip}{0pt}   
    \setlength{\belowcaptionskip}{0pt}
    \includegraphics[width=0.85\linewidth]{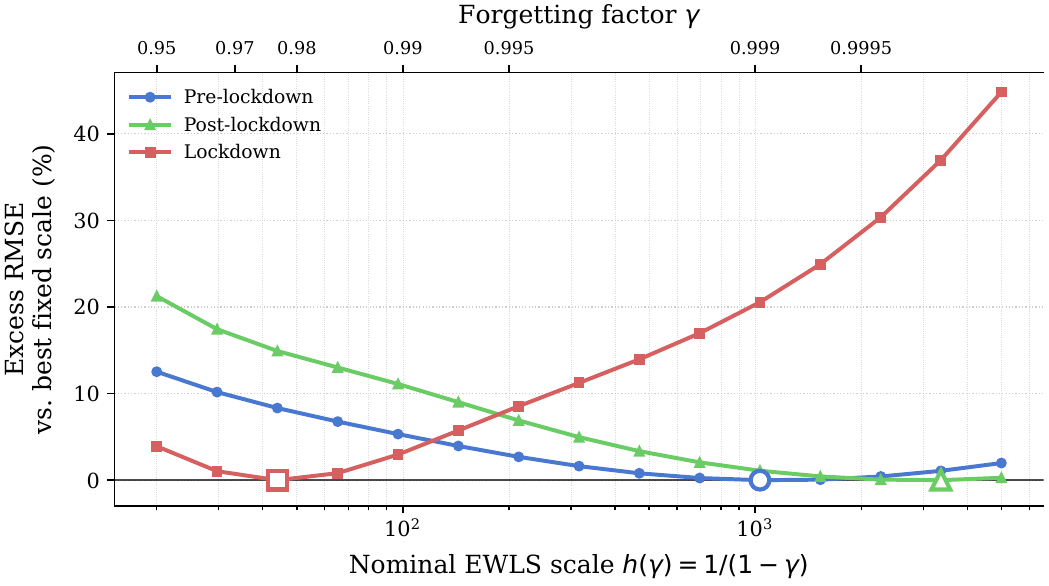}
    \caption{
    \textbf{No single forgetting scale is optimal across regimes.}
    Each curve shows the excess RMSE of a standalone EWLS correction with fixed
    forgetting factor $\gamma$, relative to the best
    fixed scale within that regime. 
    The horizontal axis is the nominal scale
    $h(\gamma)=1/(1-\gamma)$, used 
    as an ordered forgetting-scale index because the implemented experts also include small
    covariance inflation.
    Empty markers indicate the best fixed scale in each regime: lockdown favours
    a much shorter scale than the non-lockdown regimes.
    }
    \label{fig:gamma_sensitivity}
\end{figure}

\begin{figure}[t!]
    \centering
    \includegraphics[width=\linewidth]{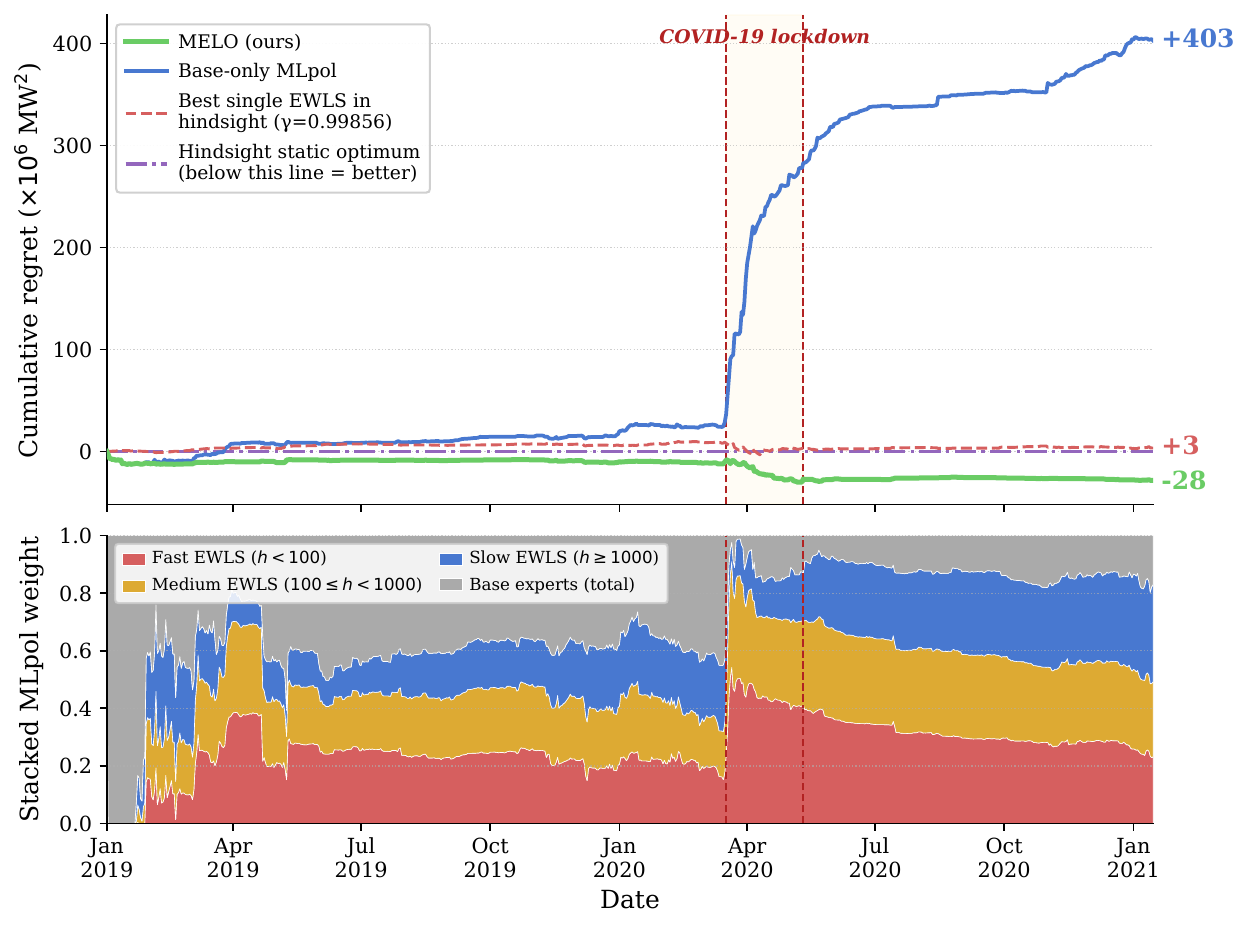}
    \caption{
    \textbf{Cumulative regret and MLpol weights over the test period.}
    \textit{Top:} Cumulative excess squared loss of three aggregators
    relative to the best static convex
    combination of all $N=M+K=23$ experts, fitted in hindsight. Values below zero mean that the online aggregate beats this static hindsight benchmark. Endpoint labels are in $10^6\, \mathrm{MW}^2$.
    \textit{Bottom:} MLpol weights stack with EWLS experts aggregated into three buckets
    by nominal scale $h_{\gamma}$: Fast ($h<100$, red), Medium ($100\le h\le 1000$, gold), and Slow ($h \ge 1000$, blue); grey on top is the total weight on the
    $M=7$ raw base experts. 
    }
    \label{fig:regret_and_weights}
\end{figure}

Figure~\ref{fig:gamma_sensitivity} isolates this scale dependence: lockdown is minimized at much shorter nominal scales than the surrounding regimes, consistent with the trade-off in Theorem~\ref{thm:rls-dynamic}. Figure~\ref{fig:regret_and_weights} shows that MLpol exploits this online: at lockdown onset, weight reallocates sharply into the Fast EWLS bucket, and away from the base experts, so \melo{} ends at cumulative regret $-28 \times 10^{6}\,\mathrm{MW}^2$ relative to the best static convex combination fitted in hindsight. Base-only MLpol diverges to $+403 \times 10^{6}\,\mathrm{MW}^2$ over the same window.
On two TabReD datasets with much higher base-residual correlation the same framework yields gains below $0.5\%$, see Appendix~\ref{app:tabred}, suggesting a boundary condition: when base residuals are nearly collinear, affine combination has little additional signal to exploit.

\FloatBarrier
%==============================================================
\section{Conclusion}
\label{sec:conclusion}
%==============================================================

We introduced \textsc{MELO}, a causal, model-agnostic adaptation layer for time-indexed tabular prediction under distribution shift. 
The method keeps raw base predictors as conservative competitors, augments them with multi-scale EWLS correction experts, and uses MLpol to select among raw and corrected forecasts by realised loss. 
Our deterministic oracle inequalities formalise the two intended safeguards: competitiveness with both the best raw predictor protects against unnecessary correction and the best EWLS memory scale gives a path-length-dependent guarantee for bounded affine comparator paths when residual drift is exploitable.
Empirically, on the RTE-FR load stress test, \textsc{MELO} improves substantially across the COVID-19 lockdown, outperforming base-only aggregation and the adaptive Kalman-filter alternatives considered under matched validation tuning. 

The limitation is that large gains require both exploitable
non-stationarity and a base pool with residual diversity. When base predictors
make highly correlated errors, the EWLS combination space has limited room
to help, as seen in the TabReD experiments, motivating a simple
residual-correlation diagnostic on the validation period. Future work includes
broader evaluations with actively retrained base models and extensions to
multivariate targets, structured losses, and online expansion of the
memory-scale grid.

\clearpage

{
\small

\bibliographystyle{abbrvnat}
\bibliography{reference}

}

\newpage

\appendix

\appendixtheorems

% -----------------------------------------------------------------------------
\section{Proofs for the pathwise online guarantees}
\label{app:deterministic-theory}
% -----------------------------------------------------------------------------

This appendix proves the deterministic guarantees stated in
Section~\ref{sec:theory}: the dynamic-regret bound for a single forgotten-RLS
correction expert, the MLpol oracle inequality over the enlarged Base+EWLS
pool, and the geometric-grid corollary. The analysis is sequence-wise: 
apart from the explicitly stated bounded-iterate (or, equivalently, uniform residual/separation) condition,
we impose no
distributional, stationarity, or persistence-of-excitation assumptions on the realised sequence.

The proof has five ingredients. First, a one-step potential identity for
forgotten RLS under squared loss. Second, a log-determinant inequality adapted
to the non-monotone recursion
\(\bA_t=\gamma\bA_{t-1}+\tilde{\bz}_t\tilde{\bz}_t^\top\). Third, a
dynamic-comparator telescoping argument that converts the fixed-comparator
identity into a path-length bound. Fourth, an MLpol potential bound for the
linearised excess losses. Fifth, a local-band refinement showing that, for dense
EWLS grids, the effective aggregation cost can depend on the inverse mass of a
useful time-scale band rather than on the raw number of grid points.

% -----------------------------------------------------------------------------
\subsection{Setup and notation}
\label{app:det-setup}
% -----------------------------------------------------------------------------

We use the same notation as in the main text. At round \(t\), the EWLS
correction expert observes the augmented base-prediction vector
\(\tilde{\bz}_t=(\bz_t^\top,1)^\top\in\mathbb R^d\), where \(d=M+1\), predicts
\(\tilde{\bz}_t^\top \bw_t\), then observes \(y_t\) and incurs squared loss
\[
    \ell_t(\bw)=\bigl(\tilde{\bz}_t^\top\bw-y_t\bigr)^2 .
\]
Here \(\bw_t\) denotes the pre-update coefficient used to predict at round
\(t\). The proof below uses only boundedness of \(\tilde{\bz}_t\), not any
property of the raw covariates or base models.

% -----------------------------------------------------------------------------
\subsection{Map to the informal statements in Section~\ref{sec:theory}}
\label{app:formal-restatement}
% -----------------------------------------------------------------------------

The main text states Theorems~\ref{thm:rls-dynamic} and \ref{thm:mlpol-oracle} 
in informal form, suppressing constants and lower-order terms for readability. 
The corresponding formal statements are 
Theorem~\ref{thm:rls-dynamic-app} (single-scale tracking) and 
Theorem~\ref{thm:mlpol-oracle-app} (\melo{} oracle inequality), proved in 
Appendices~\ref{app:rls-dynamic-proof}~and~\ref{app:oracle-proof}. The 
geometric-grid statement of Section~\ref{sec:theory} is formalised as 
Corollary~\ref{cor:grid-app}. We restate the correspondence here so that the 
notation in the formal statements is unambiguous.

\paragraph{The function $\Psi_h$.}
The single-scale tracking penalty appearing implicitly in the informal 
theorems is, with all constants made explicit,
\begin{equation}
\label{eq:psi-h-defined}
    \Psi_h(T,P)
    \;=\;
    \underbrace{\delta R^2}_{\text{initialisation}}
    \;+\;
    \underbrace{D^2 d\log\!\left(1+\tfrac{B_z^2 h}{\delta}\right)}_{\text{log term}}
    \;+\;
    \underbrace{\frac{2 D^2 dT}{h}}_{\text{finite-memory term}}
    \;+\;
    \underbrace{4R(\delta+B_z^2 h)\,P}_{\text{tracking term}} ,
\end{equation}
where $D=B_y+B_z R$ and $d=M+1$. The "log and initialisation terms" referenced 
in the informal Theorem~\ref{thm:rls-dynamic} are the first two summands; the 
constants $C_1, C_2$ in the informal statement correspond to
\[
    C_1 = 2D^2,
    \qquad
    C_2 = 8RB_z^2
    \quad\text{(under the mild regime $\delta\le B_z^2 h$).}
\]
The unsimplified $C_2$ is $4R(\delta+B_z^2 h)/h$, which reduces to the above 
once $\delta\le B_z^2 h$; see \eqref{eq:rls-dynamic-simplified}.

\paragraph{Equivalent affine view of the aggregate.}
The informal Theorem~\ref{thm:mlpol-oracle} compares $\hat y_t$ to time-varying 
affine combinations of the base-prediction vector. The justification is that 
every candidate in the \melo{} pool is itself affine in $\tilde\bz_t$: each raw 
base expert is a coordinate of $\bz_t$, and each EWLS expert is 
$\tilde\bz_t^\top \bw_t^{(\gamma)}$. Hence the MLpol aggregate satisfies
\[
    \hat y_t \;=\; \tilde\bz_t^\top \bar\bw_t,
    \qquad
    \bar\bw_t \;=\; \sum_{j=1}^{N} p_{t,j}\,\bw_{t,j},
\]
for coefficients $\bar\bw_t$ determined causally by past losses and EWLS 
states, where $\bw_{t,j}$ is the affine representation of expert $j$ (a 
canonical basis vector for raw base experts, $\bw_t^{(\gamma_k)}$ for EWLS 
experts). The comparator class $\mathcal U_T(R)$ in 
Theorem~\ref{thm:rls-dynamic-app} is therefore the natural benchmark: it is the 
same affine class that $\hat y_t$ itself ranges over.

\paragraph{Clipping convention for the named base forecasts.}
In the formal oracle inequality \eqref{eq:oracle-inequality-app}, $L_T(f_j)$ 
denotes the loss of the \emph{clipped} raw prediction sequence, since clipping 
is applied to all experts before they enter MLpol. When the unclipped raw 
predictions are the deployed base forecasts and outcomes lie in $[-B,B]$, 
Lemma~\ref{lem:clipping} shows that clipping cannot increase squared loss 
against the target. Hence the bound stated against the clipped $L_T(f_j)$ also 
holds against the unclipped raw-forecast loss in any practical reading where 
$|y_t|\le B$. This is the convention used throughout 
Section~\ref{sec:experiments}.

\paragraph{Forgotten-RLS recursion.}
Fix a forgetting factor \(\gamma\in(0,1)\) and define the nominal scale
\(h=(1-\gamma)^{-1}\). Let \(\bA_0=\delta \bI\) with \(\delta>0\),
\(\bb_0=\bm 0\), and \(\bw_1=\bm 0\). The forgotten-RLS recursion is
\begin{equation}
\label{eq:app-rls-recursion}
    \bA_t=\gamma \bA_{t-1}+\tilde{\bz}_t\tilde{\bz}_t^\top,\qquad
    \bb_t=\gamma \bb_{t-1}+y_t\tilde{\bz}_t,\qquad
    \bw_{t+1}=\bA_t^{-1}\bb_t .
\end{equation}
Equivalently,
\[
    \bA_t
    =
    \gamma^t\delta \bI
    +
    \sum_{s=1}^t\gamma^{t-s}\tilde{\bz}_s\tilde{\bz}_s^\top .
\]
Since \(\bA_t\succ0\), the update is well-defined. Moreover, because
\(\bA_{t-1}\bw_t=\bb_{t-1}\), a direct calculation gives
\begin{equation}
\label{eq:app-newton-form}
    \bw_{t+1}
    =
    \bw_t-\frac12\bA_t^{-1}\bg_t,
    \qquad
    \bg_t:=\nabla\ell_t(\bw_t)
    =
    2(\tilde{\bz}_t^\top\bw_t-y_t)\tilde{\bz}_t .
\end{equation}

\paragraph{Boundedness assumptions.}
We use the following deterministic boundedness conditions:
\begin{enumerate}[leftmargin=1.6em,itemsep=1pt,topsep=2pt]
    \item[(B1)] \(\|\tilde{\bz}_t\|_2\le B_z\) for all \(t\).
    \item[(B2)] \(|y_t|\le B_y\) for all \(t\).
    \item[(B3)] The comparator path satisfies \(\|\bu_t\|_2\le R\) for all \(t\).
    \item[(B4)] The unprojected forgotten-RLS iterates are stable along the
    realised sequence: \(\|\bw_t\|_2\le R\) for all \(t\).
\end{enumerate}

Condition~(B4) is an explicit stability assumption on the unprojected RLS
recursion, not a consequence of forgetting alone. Without a conditioning or
persistence-of-excitation condition, \(\bA_t\) may be ill-conditioned in
directions weakly observed by the realised sequence, and
\(\bw_t=\bA_{t-1}^{-1}\bb_{t-1}\) need not be uniformly bounded. We state the
pathwise tracking bound conditionally on this stability event, thereby
avoiding any stochastic excitation assumption.

\paragraph{Empirical diagnostic for (B4).}
Appendix~\ref{app:bounded-iterate-diag} provides a sanity check for this
condition on the RTE-FR Load experiment. Since the augmented EWLS coefficient
contains an intercept on the response scale, the diagnostic separates the
full augmented norm from the slope-only norm. Along the realised test
trajectory, the slope component remains uniformly bounded across the
\(K-1=15\) finite-forgetting EWLS experts and all \(T=746\) test days, with
\[
    \sup_{t,\gamma}\|\bm\beta_t^{(\gamma)}\|_2 = 1.69 .
\]

Alternatively, the proof can be read under the direct bounds
\[
    |\tilde{\bz}_t^\top\bw_t-y_t|\le D
    \qquad\text{and}\qquad
    \|\bw_t-\bu_t\|_2\le 2R
\]
for the comparator class under consideration. We do not claim that Euclidean
projection preserves the exact recursion \eqref{eq:app-rls-recursion}; the
proof below is for the unprojected recursion. In the implementation,
predictions are clipped before being passed to MLpol, which ensures bounded
losses for the aggregation layer; see Lemma~\ref{lem:clipping} and
Remark~\ref{rem:implementation-rls}.

Set
\[
    D:=B_y+B_zR .
\]
Under (B1), (B2), and (B4), the realised RLS residuals satisfy
\[
    |\tilde{\bz}_t^\top\bw_t-y_t|\le D.
\]
Under (B1), (B2), and (B3), any comparator \(\bu_t\) also satisfies
\[
    |\tilde{\bz}_t^\top\bu_t-y_t|\le D.
\]

For a comparator path \(\bu_{1:T}=(\bu_1,\ldots,\bu_T)\), define
\begin{equation}
\label{eq:app-path-length}
    L_T(\bu_{1:T})
    =
    \sum_{t=1}^T(\tilde{\bz}_t^\top\bu_t-y_t)^2,
    \qquad
    P_T(\bu_{1:T})
    =
    \sum_{t=2}^T {\|\bu_t-\bu_{t-1}\|}_2 .
\end{equation}
We write
\[
    \check y_t^{(\gamma)}
    =
    \tilde{\bz}_t^\top\bw_t,
    \qquad
    L_T(\check y^{(\gamma)})
    =
    \sum_{t=1}^T
    \bigl(\check y_t^{(\gamma)}-y_t\bigr)^2
\]
for the prediction sequence and cumulative loss of the forgotten-RLS expert
with nominal scale \(h=(1-\gamma)^{-1}\).

% -----------------------------------------------------------------------------
\subsection{A one-step potential identity}
\label{app:one-step}
% -----------------------------------------------------------------------------

\begin{lemma}[One-step identity for forgotten RLS]
\label{lem:bregman-step}
For every \(\bu\in\mathbb R^d\) and every \(t\ge1\),
\begin{equation}
\label{eq:bregman-step-correct}
    \ell_t(\bw_t)-\ell_t(\bu)
    =
    \gamma\|\bw_t-\bu\|_{\bA_{t-1}}^2
    -
    \|\bw_{t+1}-\bu\|_{\bA_t}^2
    +
    \frac14 \bg_t^\top \bA_t^{-1}\bg_t .
\end{equation}
\end{lemma}

\begin{proof}
For squared loss, expanding around \(\bw_t\) gives the exact identity
\begin{equation}
\label{eq:exact-square-expansion}
    \ell_t(\bw_t)-\ell_t(\bu)
    =
    \bg_t^\top(\bw_t-\bu)
    -
    \bigl(\tilde{\bz}_t^\top(\bw_t-\bu)\bigr)^2 .
\end{equation}
From \eqref{eq:app-newton-form},
\[
    \bg_t=2\bA_t(\bw_t-\bw_{t+1}).
\]
Therefore
\begin{align}
\label{eq:polarization-correct}
    \bg_t^\top(\bw_t-\bu)
    &=
    2(\bw_t-\bw_{t+1})^\top \bA_t(\bw_t-\bu) \notag\\
    &=
    \|\bw_t-\bu\|_{\bA_t}^2
    -
    \|\bw_{t+1}-\bu\|_{\bA_t}^2
    +
    \|\bw_t-\bw_{t+1}\|_{\bA_t}^2 \notag\\
    &=
    \|\bw_t-\bu\|_{\bA_t}^2
    -
    \|\bw_{t+1}-\bu\|_{\bA_t}^2
    +
    \frac14 \bg_t^\top \bA_t^{-1}\bg_t .
\end{align}
The recursion
\(\bA_t=\gamma\bA_{t-1}+\tilde{\bz}_t\tilde{\bz}_t^\top\) implies
\[
    \|\bw_t-\bu\|_{\bA_t}^2
    =
    \gamma\|\bw_t-\bu\|_{\bA_{t-1}}^2
    +
    \bigl(\tilde{\bz}_t^\top(\bw_t-\bu)\bigr)^2 .
\]
Substituting this expression and \eqref{eq:polarization-correct} into
\eqref{eq:exact-square-expansion} cancels the curvature term
\(\bigl(\tilde{\bz}_t^\top(\bw_t-\bu)\bigr)^2\), yielding
\eqref{eq:bregman-step-correct}.
\end{proof}

% -----------------------------------------------------------------------------
\subsection{A forgetting-adapted log-determinant inequality}
\label{app:logdet}
% -----------------------------------------------------------------------------

The classical online-Newton log-determinant argument relies on monotonicity of
\(\bA_t=\bA_{t-1}+\tilde{\bz}_t\tilde{\bz}_t^\top\). With forgetting,
\(\bA_t=\gamma\bA_{t-1}+\tilde{\bz}_t\tilde{\bz}_t^\top\) is no longer monotone
in the PSD order. The next lemma is the corresponding log-determinant bound for
forgotten RLS.

\begin{lemma}[Forgetting-adapted log-determinant inequality]
\label{lem:logdet}
Let \(\bA_t=\gamma\bA_{t-1}+\tilde{\bz}_t\tilde{\bz}_t^\top\) with
\(\bA_0=\delta\bI\), \(\delta>0\), \(\gamma=1-1/h\), and \(h\ge2\). Under (B1),
\begin{equation}
\label{eq:logdet}
    \sum_{t=1}^T \tilde{\bz}_t^\top\bA_t^{-1}\tilde{\bz}_t
    \le
    d\log\!\left(1+\frac{B_z^2h}{\delta}\right)
    +
    \frac{2dT}{h}.
\end{equation}
\end{lemma}

\begin{proof}
Define
\[
    q_t
    :=
    \tilde{\bz}_t^\top(\gamma\bA_{t-1})^{-1}\tilde{\bz}_t
    =
    \gamma^{-1}\tilde{\bz}_t^\top\bA_{t-1}^{-1}\tilde{\bz}_t
    \ge0.
\]
By the Sherman--Morrison formula,
\[
    \bA_t^{-1}
    =
    (\gamma\bA_{t-1})^{-1}
    -
    \frac{
    (\gamma\bA_{t-1})^{-1}
    \tilde{\bz}_t\tilde{\bz}_t^\top
    (\gamma\bA_{t-1})^{-1}
    }
    {1+\tilde{\bz}_t^\top(\gamma\bA_{t-1})^{-1}\tilde{\bz}_t}.
\]
Contracting with \(\tilde{\bz}_t\) gives
\begin{equation}
\label{eq:key-identity}
    \tilde{\bz}_t^\top \bA_t^{-1}\tilde{\bz}_t
    =
    \frac{q_t}{1+q_t}.
\end{equation}
The matrix-determinant lemma gives
\[
    \det \bA_t
    =
    \det(\gamma\bA_{t-1})\bigl(1+q_t\bigr)
    =
    \gamma^d\det(\bA_{t-1})(1+q_t).
\]
Thus
\[
    \log(1+q_t)
    =
    \log\det \bA_t-\log\det \bA_{t-1}+d\log(1/\gamma).
\]
Summing over \(t\) and using \(\log(1+q)\ge q/(1+q)\) for \(q\ge0\),
\begin{align}
    \sum_{t=1}^T \tilde{\bz}_t^\top\bA_t^{-1}\tilde{\bz}_t
    &=
    \sum_{t=1}^T\frac{q_t}{1+q_t}
    \le
    \sum_{t=1}^T\log(1+q_t) \notag\\
    &=
    \log\frac{\det \bA_T}{\det \bA_0}
    +
    dT\log(1/\gamma).
\label{eq:logdet-prefinal}
\end{align}
Next,
\[
    \bA_T
    =
    \gamma^T\delta\bI
    +
    \sum_{s=1}^T\gamma^{T-s}
    \tilde{\bz}_s\tilde{\bz}_s^\top
    \preceq
    \left(\delta+B_z^2\sum_{j=0}^{T-1}\gamma^j\right)\bI
    \preceq
    (\delta+B_z^2h)\bI.
\]
Therefore
\[
    \log\frac{\det \bA_T}{\det \bA_0}
    \le
    d\log\!\left(1+\frac{B_z^2h}{\delta}\right).
\]
Finally, since \(h\ge2\), \(\gamma=1-1/h\ge1/2\), and
\[
    \log(1/\gamma)
    =
    -\log(1-1/h)
    \le
    \frac{2}{h}.
\]
Substituting into \eqref{eq:logdet-prefinal} proves
\eqref{eq:logdet}.
\end{proof}

\begin{remark}[Origin of the \(T/h\) term]
\label{rem:where-T-over-h}
When \(\gamma=1\), the second term in \eqref{eq:logdet-prefinal} vanishes and
one recovers the classical log-determinant control.  Forgetting contributes
\(dT\log(1/\gamma)\approx dT/h\).  This is the explicit cost of using a finite
memory scale rather than retaining all past curvature.
\end{remark}

% -----------------------------------------------------------------------------
\subsection{Dynamic regret of a single forgotten-RLS expert}
\label{app:rls-dynamic-proof}
% -----------------------------------------------------------------------------

\begin{theorem}[Dynamic regret of forgotten RLS]
\label{thm:rls-dynamic-app}
Assume (B1)--(B4), let \(h=(1-\gamma)^{-1}\ge2\), and let
\(\bu_{1:T}\) be any comparator path satisfying (B3). Then
\begin{equation}
\label{eq:rls-dynamic-general}
    L_T(\check y^{(h)})-L_T(\bu_{1:T})
    \le
    \delta R^2
    +
    D^2 d\log\!\left(1+\frac{B_z^2h}{\delta}\right)
    +
    \frac{2D^2dT}{h}
    +
    4R(\delta+B_z^2h)P_T(\bu_{1:T}).
\end{equation}
In particular, if \(\delta\le B_z^2h\), then
\begin{equation}
\label{eq:rls-dynamic-simplified}
    L_T(\check y^{(h)})-L_T(\bu_{1:T})
    \le
    \delta R^2
    +
    D^2 d\log\!\left(1+\frac{B_z^2h}{\delta}\right)
    +
    \frac{2D^2dT}{h}
    +
    8RB_z^2h\,P_T(\bu_{1:T}).
\end{equation}
Ignoring logarithmic and initialization terms, the \(h\)-dependent part is
\[
    O\!\left(\frac{dT}{h}+hP_T(\bu_{1:T})\right).
\]
\end{theorem}

\begin{proof}
Apply Lemma~\ref{lem:bregman-step} with \(\bu=\bu_t\) and sum over \(t\):
\begin{align}
    L_T(\check y^{(h)})-L_T(\bu_{1:T})
    &=
    \sum_{t=1}^T \bigl(\ell_t(\bw_t)-\ell_t(\bu_t)\bigr) \notag\\
    &\le
    \underbrace{
    \sum_{t=1}^T
    \left[
        \gamma\|\bw_t-\bu_t\|_{\bA_{t-1}}^2
        -
        \|\bw_{t+1}-\bu_t\|_{\bA_t}^2
    \right]}_{S_1}
    +
    \underbrace{
    \frac14\sum_{t=1}^T \bg_t^\top \bA_t^{-1}\bg_t}_{S_2}.
\label{eq:regret-s1-s2}
\end{align}

\paragraph{Bounding \(S_2\).}
By (B1), (B2), and (B4),
\[
    |\tilde{\bz}_t^\top\bw_t-y_t|\le D,
    \qquad
    \bg_t=2(\tilde{\bz}_t^\top\bw_t-y_t)\tilde{\bz}_t.
\]
Hence
\[
    \bg_t^\top \bA_t^{-1}\bg_t
    \le
    4D^2 \tilde{\bz}_t^\top \bA_t^{-1}\tilde{\bz}_t .
\]
Using Lemma~\ref{lem:logdet},
\begin{equation}
\label{eq:s2-bound}
    S_2
    \le
    D^2
    \sum_{t=1}^T
    \tilde{\bz}_t^\top \bA_t^{-1}\tilde{\bz}_t
    \le
    D^2 d\log\!\left(1+\frac{B_z^2h}{\delta}\right)
    +
    \frac{2D^2dT}{h}.
\end{equation}

\paragraph{Bounding \(S_1\).}
Shift the second sum in \(S_1\):
\begin{align}
S_1
&=
\gamma\|\bw_1-\bu_1\|_{\bA_0}^2
-
\|\bw_{T+1}-\bu_T\|_{\bA_T}^2 \notag\\
&\quad+
\sum_{t=2}^T
\left[
    \gamma\|\bw_t-\bu_t\|_{\bA_{t-1}}^2
    -
    \|\bw_t-\bu_{t-1}\|_{\bA_{t-1}}^2
\right].
\label{eq:s1-shifted}
\end{align}
The terminal term is non-positive and can be dropped. Since \(\bw_1=\bm 0\),
\(\bA_0=\delta\bI\), \(\gamma\le1\), and \(\|\bu_1\|\le R\),
\[
    \gamma\|\bw_1-\bu_1\|_{\bA_0}^2
    \le
    \delta R^2.
\]
For the bracketed term, let
\[
    \bx_t:=\bw_t-\bu_t,
    \qquad
    \bDelta_t:=\bu_t-\bu_{t-1}.
\]
Then \(\bw_t-\bu_{t-1}=\bx_t+\bDelta_t\), and
\begin{align}
    \gamma\|\bx_t\|_{\bA_{t-1}}^2
    -
    \|\bx_t+\bDelta_t\|_{\bA_{t-1}}^2
    &=
    -(1-\gamma)\|\bx_t\|_{\bA_{t-1}}^2
    -
    2\bDelta_t^\top\bA_{t-1}\bx_t
    -
    \|\bDelta_t\|_{\bA_{t-1}}^2
    \notag\\
    &\le
    2|\bDelta_t^\top\bA_{t-1}\bx_t| .
\label{eq:bracket-bound}
\end{align}
Moreover,
\[
    \bA_{t-1}
    =
    \gamma^{t-1}\delta\bI
    +
    \sum_{s=1}^{t-1}
    \gamma^{t-1-s}\tilde{\bz}_s\tilde{\bz}_s^\top
    \preceq
    (\delta+B_z^2h)\bI.
\]
By (B3)--(B4), \(\|\bx_t\|\le2R\). Therefore
\[
    2|\bDelta_t^\top\bA_{t-1}\bx_t|
    \le
    2\|\bDelta_t\|\,\|\bA_{t-1}\|_{\mathrm{op}}\|\bx_t\|
    \le
    4R(\delta+B_z^2h)\|\bu_t-\bu_{t-1}\|.
\]
Summing \eqref{eq:bracket-bound} over \(t=2,\ldots,T\),
\[
    S_1
    \le
    \delta R^2
    +
    4R(\delta+B_z^2h)P_T(\bu_{1:T}).
\]
Combining this with \eqref{eq:s2-bound} proves
\eqref{eq:rls-dynamic-general}. If \(\delta\le B_z^2h\), then
\(\delta+B_z^2h\le2B_z^2h\), giving
\eqref{eq:rls-dynamic-simplified}.
\end{proof}

\begin{remark}[About bounded iterates and projection]
    \label{rem:bounded-iterates-projection}
    The exact identity \eqref{eq:app-newton-form} is what makes
    Lemma~\ref{lem:bregman-step} possible. A Euclidean projection after the RLS
    update would generally destroy this identity. One can instead analyse an
    \(\bA_t\)-metric projection, as in online Newton methods, but the implementation
    in this paper is closer to the unprojected RLS recursion with prediction
    clipping. We therefore state the deterministic bound under the bounded-iterate
    condition (B4). The same proof also goes through under direct bounds
    \(|\tilde{\bz}_t^\top\bw_t-y_t|\le D\) and
    \(\|\bw_t-\bu_t\|\le 2R\) for the comparator class under consideration.
\end{remark}

% -----------------------------------------------------------------------------
\subsection{Clipping and bounded expert losses}
\label{app:clipping}
% -----------------------------------------------------------------------------

The MLpol bound below requires bounded outcomes and bounded expert
predictions.  We use the standard clipping device.

\begin{lemma}[Clipping does not increase squared loss]
\label{lem:clipping}
Let \(|y|\le B\), and define
\[
    \mathrm{clip}_B(a):=\max\{-B,\min\{a,B\}\}.
\]
Then
\[
    \bigl(\mathrm{clip}_B(a)-y\bigr)^2\le (a-y)^2
    \qquad
    \text{for all }a\in\mathbb R.
\]
\end{lemma}

\begin{proof}
If \(a\in[-B,B]\), the claim is equality.  If \(a>B\), then
\(\mathrm{clip}_B(a)=B\ge y\), so
\[
    |B-y|\le |a-y|.
\]
The case \(a<-B\) is symmetric.
\end{proof}

Thus the clipped forgotten-RLS expert has cumulative squared loss no larger
than the un-clipped expert analysed in Theorem~\ref{thm:rls-dynamic-app},
provided the target lies in \([-B,B]\).

% -----------------------------------------------------------------------------
\subsection{MLpol regret over a finite pool}
\label{app:mlpol-regret}
% -----------------------------------------------------------------------------

We next record the MLpol potential bound used for the aggregation layer. We
state a distributional form against any fixed convex combination of experts:
the point-expert oracle inequality is the special case where the distribution
is a point mass, and the local-band refinement below uses the uniform
distribution over a band of nearby forgetting factors.

Let \(N\) be the number of experts. At round \(t\), write
\(\tilde y_{t,j}\) for the prediction of expert \(j\), and let
\[
    \hat y_t=\sum_{j=1}^N p_{t,j}\tilde y_{t,j}
\]
be the aggregate prediction. We use the linearised excess loss
\[
    r_{t,j}
    =
    2(\hat y_t-y_t)(\hat y_t-\tilde y_{t,j}),
    \qquad
    R_{t,j}=\sum_{s=1}^t r_{s,j}.
\]
The MLpol weights are
\[
    p_{t,j}
    =
    \frac{[R_{t-1,j}]_+}{\sum_{i=1}^N [R_{t-1,i}]_+},
\]
with the uniform distribution used when all positive parts vanish.

\begin{lemma}[MLpol finite-pool regret, distributional form]
\label{lem:mlpol-regret}
Assume that all expert predictions and outcomes lie in \([-B,B]\). For any
fixed distribution \(q\in\Delta_N\), define the corresponding convex-mixture
prediction
\[
    y_t^q=\sum_{j=1}^N q_j\tilde y_{t,j}.
\]
Then there exists a universal constant \(C>0\) such that
\begin{equation}
\label{eq:mlpol-distributional-regret}
    \sum_{t=1}^T(\hat y_t-y_t)^2
    \le
    \sum_{t=1}^T(y_t^q-y_t)^2
    +
    C B^2\sqrt{T\,N\|q\|_2^2}.
\end{equation}
In particular, taking \(q\) to be a point mass gives
\begin{equation}
\label{eq:mlpol-finite-regret}
    \sum_{t=1}^T(\hat y_t-y_t)^2
    \le
    \min_{1\le j\le N}
    \sum_{t=1}^T(\tilde y_{t,j}-y_t)^2
    +
    C B^2\sqrt{TN}.
\end{equation}
Moreover, if \(S\subseteq\{1,\ldots,N\}\) is non-empty and \(q^S\) is uniform on
\(S\), then
\begin{equation}
\label{eq:mlpol-band-regret}
    \sum_{t=1}^T(\hat y_t-y_t)^2
    \le
    \sum_{t=1}^T(\bar y_t^S-y_t)^2
    +
    C B^2\sqrt{T\,\frac{N}{|S|}},
    \qquad
    \bar y_t^S=\frac{1}{|S|}\sum_{j\in S}\tilde y_{t,j}.
\end{equation}
\end{lemma}

\begin{remark} 
A tighter, second-order bound of the form $O(B \sqrt{(N \cdot L_T)} )$ can be obtained via the MLpol potential analysis of \citep{gaillard2014second}. We use the simpler first-order form here because the leading rate-determining term in Theorem \ref{thm:mlpol-oracle} is the EWLS tracking term $\Psi_h(T, P_T)$, not the MLpol overhead.
\end{remark}

\begin{proof}
Define the potential
\[
    \Phi_t
    =
    \frac{1}{2N}
    \sum_{j=1}^N [R_{t,j}]_+^2 .
\]
For any scalar \(x,a\), the inequality
\[
    [x+a]_+^2-[x]_+^2
    \le
    2[x]_+a+a^2
\]
holds.  Therefore
\[
    \Phi_t-\Phi_{t-1}
    \le
    \frac{1}{N}\sum_{j=1}^N [R_{t-1,j}]_+ r_{t,j}
    +
    \frac{1}{2N}\sum_{j=1}^N r_{t,j}^2 .
\]
The first term is zero.  Indeed, if some positive part is non-zero, then
\[
\begin{aligned}
    \sum_{j=1}^N [R_{t-1,j}]_+ r_{t,j}
    &=
    \left(\sum_{j=1}^N [R_{t-1,j}]_+\right)
    \sum_{j=1}^N p_{t,j}r_{t,j}  \\
    &=
    \left(\sum_{j=1}^N [R_{t-1,j}]_+\right)
    2(\hat y_t-y_t)
    \left(\hat y_t-\sum_{j=1}^N p_{t,j}\tilde y_{t,j}\right)
    =
    0 .
\end{aligned}
\]
If all positive parts vanish, the same term is trivially zero.  Since
\(\hat y_t,y_t,\tilde y_{t,j}\in[-B,B]\),
\[
    |r_{t,j}|
    =
    2| \hat y_t-y_t |\,|\hat y_t-\tilde y_{t,j}|
    \le
    8B^2 .
\]
Thus
\[
    \Phi_t-\Phi_{t-1}
    \le
    C B^4,
    \qquad
    \text{and hence}
    \qquad
    \Phi_T\le C B^4T .
\]

Now fix \(q\in\Delta_N\).  For each round \(t\), by the elementary identity
\(a^2-b^2=2a(a-b)-(a-b)^2\), with
\(a=\hat y_t-y_t\) and \(b=y_t^q-y_t\), we have
\[
    (\hat y_t-y_t)^2-(y_t^q-y_t)^2
    \le
    2(\hat y_t-y_t)(\hat y_t-y_t^q)
    =
    \sum_{j=1}^N q_j r_{t,j}.
\]
Summing over \(t\),
\[
    L_T(\hat y)-L_T(y^q)
    \le
    \sum_{j=1}^N q_j R_{T,j}
    \le
    \sum_{j=1}^N q_j [R_{T,j}]_+ .
\]
By Cauchy--Schwarz,
\[
\begin{aligned}
    \sum_{j=1}^N q_j [R_{T,j}]_+
    &\le
    \left(N\sum_{j=1}^N q_j^2\right)^{1/2}
    \left(
        \frac{1}{N}\sum_{j=1}^N [R_{T,j}]_+^2
    \right)^{1/2}  \\
    &=
    \sqrt{N\|q\|_2^2}\,\sqrt{2\Phi_T}
    \le
    C B^2\sqrt{T\,N\|q\|_2^2}.
\end{aligned}
\]
This proves \eqref{eq:mlpol-distributional-regret}.  The point-mass and
uniform-on-\(S\) special cases follow from
\(\|e_j\|_2^2=1\) and \(\|q^S\|_2^2=1/|S|\), respectively.
\end{proof}

% -----------------------------------------------------------------------------
\subsection{Oracle inequality for the Base+EWLS pool}
\label{app:oracle-proof}
% -----------------------------------------------------------------------------

Let \(\mathcal H=\{h_1,\ldots,h_K\}\) be a finite set of nominal EWLS scales.
For each \(h\in\mathcal H\), let \(\check y^{(h)}\) denote the corresponding
forgotten-RLS correction sequence, with predictions clipped to \([-B,B]\)
before aggregation. The enlarged pool is
\[
    \mathcal I_{\mathcal H}
    =
    \{f_1,\ldots,f_M\}\cup\{\check y^{(h)}:h\in\mathcal H\}.
\]
Let \(N=M+K\).

Define
\begin{equation}
\label{eq:psi-h-full}
    \Psi_h(T,P)
    :=
    \delta R^2
    +
    D^2 d\log\!\left(1+\frac{B_z^2h}{\delta}\right)
    +
    \frac{2D^2dT}{h}
    +
    4R(\delta+B_z^2h)P .
\end{equation}

\begin{theorem}[Oracle inequality for multi-scale online correction]
\label{thm:mlpol-oracle-app}
Assume the conditions of Theorem~\ref{thm:rls-dynamic-app} hold for each
\(h\in\mathcal H\), and assume outcomes lie in \([-B,B]\). All predictions
entering MLpol are clipped to \([-B,B]\). Then, writing \(f_j\) and
\(\check y^{(h)}\) for the clipped raw and EWLS prediction sequences, for every
comparator path \(\bu_{1:T}\) satisfying (B3),
\begin{equation}
\label{eq:oracle-inequality-app}
    L_T(\hat y)
    \le
    \min\left\{
        \min_{1\le j\le M}L_T(f_j),
        \;
        L_T(\bu_{1:T})
        +
        \min_{h\in\mathcal H}\Psi_h(T,P_T(\bu_{1:T}))
    \right\}
    +
    CB^2\sqrt{T(M+K)}.
\end{equation}
\end{theorem}

\begin{proof}
By the point-mass special case of Lemma~\ref{lem:mlpol-regret} applied to the
enlarged pool,
\[
    L_T(\hat y)
    \le
    \min\left\{
        \min_{1\le j\le M}L_T(f_j),
        \min_{h\in\mathcal H}L_T(\check y^{(h)})
    \right\}
    +
    CB^2\sqrt{T(M+K)}.
\]
By Lemma~\ref{lem:clipping}, clipping cannot increase squared loss. Therefore
Theorem~\ref{thm:rls-dynamic-app} gives, for every \(h\in\mathcal H\),
\[
    L_T(\check y^{(h)})
    \le
    L_T(\bu_{1:T})+\Psi_h(T,P_T(\bu_{1:T})).
\]
Taking the minimum over \(h\in\mathcal H\) proves
\eqref{eq:oracle-inequality-app}.
\end{proof}

% -----------------------------------------------------------------------------
\subsection{Geometric grids and unknown correction scale}
\label{app:grid-proof}
% -----------------------------------------------------------------------------

The leading \(h\)-dependent terms in \(\Psi_h\) are
\[
    \frac{C_1dT}{h}+C_2hP,
\]
where \(P=P_T(\bu_{1:T})\). If \(P=0\), the leading expression is minimised by
the largest available memory scale, corresponding to the no-drift case. For
\(P>0\), the continuous minimizer is
\[
    h^\star
    =
    \sqrt{\frac{C_1dT}{C_2P}},
\]
with minimum value \(2\sqrt{C_1C_2dTP}\).

\begin{corollary}[Geometric-grid adaptation]
\label{cor:grid-app}
Let \(P=P_T(\bu_{1:T})>0\), and let \(\mathcal H\) be a geometric grid over
\([h_{\min},h_{\max}]\) with ratio \(\rho>1\). Suppose
\(h^\star\in[h_{\min},h_{\max}]\). Then there exists \(h_k\in\mathcal H\) such
that
\begin{equation}
\label{eq:grid-approx}
    \frac{C_1dT}{h_k}+C_2h_kP
    \le
    C_\rho\sqrt{dTP},
\end{equation}
where \(C_\rho\) depends only on \(\rho,C_1,C_2\). Consequently, ignoring
logarithmic and initialization terms,
\begin{equation}
\label{eq:grid-oracle-final}
    L_T(\hat y)
    \le
    \min\left\{
        \min_j L_T(f_j),
        \;
        L_T(\bu_{1:T})
        +
        C_\rho\sqrt{dT\,P_T(\bu_{1:T})}
    \right\}
    +
    O\!\left(B^2\sqrt{T(M+K)}\right).
\end{equation}
\end{corollary}

\begin{proof}
Since \(\mathcal H\) is geometric and contains \(h^\star\) within its range,
there is a grid point \(h_k\) with
\[
    h^\star\le h_k\le \rho h^\star
    \quad\text{or}\quad
    h^\star/\rho\le h_k\le h^\star .
\]
In either case,
\[
    \frac{h_k}{h^\star}\le \rho,
    \qquad
    \frac{h^\star}{h_k}\le \rho.
\]
Therefore
\begin{align*}
    \frac{C_1dT}{h_k}+C_2h_kP
    &\le
    \rho\left(\frac{C_1dT}{h^\star}+C_2h^\star P\right)  \\
    &=
    2\rho\sqrt{C_1C_2dTP}.
\end{align*}
Absorbing \(2\rho\sqrt{C_1C_2}\) into \(C_\rho\) proves
\eqref{eq:grid-approx}. Substituting this bound into
Theorem~\ref{thm:mlpol-oracle-app} gives \eqref{eq:grid-oracle-final}.
\end{proof}

% -----------------------------------------------------------------------------
\subsection{Local-band refinement for dense EWLS grids}
\label{app:local-band-proof}
% -----------------------------------------------------------------------------

The oracle inequality above compares MLpol with the best single EWLS scale in a
finite grid, which gives a worst-case aggregation overhead of order
\(\sqrt{T(M+K)}\). This point-oracle view can be pessimistic for dense EWLS
grids. Nearby forgetting factors often generate very similar prediction
sequences, and the useful comparator may be a local band of time scales rather
than a single isolated expert. The distributional form of
Lemma~\ref{lem:mlpol-regret} makes this precise.

Let the enlarged pool contain \(M\) base experts and \(K\) EWLS experts, so that
\(N=M+K\). For a non-empty subset \(S\subseteq\{1,\ldots,K\}\) of EWLS scales,
define the band-averaged EWLS prediction
\[
    \bar y_t^S
    =
    \frac{1}{|S|}
    \sum_{k\in S} \check y_t^{(h_k)} .
\]

\begin{corollary}[Band-oracle bound for EWLS grids]
\label{cor:band-oracle}
Assume all base predictions, EWLS predictions after clipping, and outcomes lie
in \([-B,B]\). Then for every non-empty EWLS band
\(S\subseteq\{1,\ldots,K\}\),
\begin{equation}
\label{eq:band-oracle}
    L_T(\hat y)
    \le
    L_T(\bar y^S)
    +
    C B^2
    \sqrt{
        T\,\frac{M+K}{|S|}
    } .
\end{equation}
\end{corollary}

\begin{proof}
Apply Lemma~\ref{lem:mlpol-regret} with \(N=M+K\) to the distribution \(q^S\)
that assigns mass \(1/|S|\) to the EWLS experts in \(S\) and zero mass to all
other experts, including the \(M\) base experts. Then
\[
    \|q^S\|_2^2=\frac{1}{|S|},
\]
and the convex-mixture prediction induced by \(q^S\) is precisely
\(\bar y_t^S\). Substituting this into
\eqref{eq:mlpol-distributional-regret} gives \eqref{eq:band-oracle}.
\end{proof}

We next relate the band-average comparator to a single central forgetting scale.
The following continuity statement is for the idealised forgotten-RLS recursion
analysed above, without covariance inflation. It is used only to justify why
nearby forgetting factors in a dense grid should not be viewed as unrelated
experts. It is convenient to parametrize forgetting factors by
\[
    \theta=\log h,
    \qquad h=(1-\gamma)^{-1}.
\]
Write \(\check y^{(\theta)}\) for the clipped EWLS prediction sequence with
\(h=e^\theta\).

\begin{lemma}[Continuity of finite-horizon EWLS predictions]
\label{lem:ewls-continuity}
Fix \(T<\infty\), \(\delta>0\), and a compact interval
\(\Theta\subset\mathbb R\). For every \(t\le T\), the idealised EWLS coefficient
\(\bw_t^{(\theta)}\) and prediction \(\check y_t^{(\theta)}\) are continuous
functions of \(\theta\in\Theta\). If \(\Theta\) is compact and
\(\gamma(\theta)=1-e^{-\theta}\) stays in \((0,1)\), then the map is Lipschitz:
there exists \(L_\theta<\infty\), depending on the realised finite sequence and
on \(\Theta\), such that
\begin{equation}
\label{eq:ewls-lipschitz}
    \max_{1\le t\le T}
    |\check y_t^{(\theta)}-\check y_t^{(\theta')}|
    \le
    L_\theta |\theta-\theta'|
    \qquad
    \text{for all }\theta,\theta'\in\Theta .
\end{equation}
\end{lemma}

\begin{proof}
For the idealised EWLS recursion,
\[
    \bA_t(\gamma)
    =
    \gamma^t\delta\bI
    +
    \sum_{s=1}^t
    \gamma^{t-s}\tilde{\bz}_s\tilde{\bz}_s^\top,
    \qquad
    \bb_t(\gamma)
    =
    \sum_{s=1}^t
    \gamma^{t-s}y_s\tilde{\bz}_s .
\]
Both \(\bA_t(\gamma)\) and \(\bb_t(\gamma)\) are polynomial functions of
\(\gamma\). Moreover \(\bA_t(\gamma)\succeq \gamma^t\delta\bI\succ0\) for
\(\gamma\in(0,1]\), so \(\bA_t(\gamma)^{-1}\) is continuous wherever
\(\gamma>0\). Hence
\[
    \bw_{t+1}^{(\gamma)}
    =
    \bA_t(\gamma)^{-1}\bb_t(\gamma)
\]
is continuous in \(\gamma\), and therefore in \(\theta\) through the continuous
map \(\gamma(\theta)=1-e^{-\theta}\). On a compact interval contained in
\((0,1)\), the derivative is bounded, so the map is Lipschitz. The prediction
\(\check y_t^{(\theta)}=\tilde{\bz}_t^\top\bw_t^{(\theta)}\), followed by
clipping, preserves continuity and Lipschitzness.
\end{proof}

\begin{corollary}[Local-band control for dense EWLS grids]
\label{cor:local-band}
Assume the Lipschitz condition \eqref{eq:ewls-lipschitz}. Fix a target scale
\(\theta^\star\in\Theta\) and a radius \(r>0\). Let
\[
    S_r(\theta^\star)
    =
    \left\{
        k\in\{1,\ldots,K\}:
        |\theta_k-\theta^\star|\le r
    \right\},
    \qquad
    \theta_k=\log h_k,
\]
and assume \(S_r(\theta^\star)\neq\emptyset\). Then
\begin{equation}
\label{eq:local-band-control}
    L_T(\hat y)
    \le
    L_T(\check y^{(\theta^\star)})
    +
    4B L_\theta T r
    +
    C B^2
    \sqrt{
        T\,\frac{M+K}{|S_r(\theta^\star)|}
    } .
\end{equation}
Consequently, combining with Theorem~\ref{thm:rls-dynamic-app}, for
\(h^\star=e^{\theta^\star}\),
\begin{equation}
\label{eq:local-band-dynamic}
    L_T(\hat y)
    \le
    L_T(\bu_{1:T})
    +
    \Psi_{h^\star}(T,P_T(\bu_{1:T}))
    +
    4B L_\theta T r
    +
    C B^2
    \sqrt{
        T\,\frac{M+K}{|S_r(\theta^\star)|}
    } .
\end{equation}
\end{corollary}

\begin{proof}
For every \(k\in S_r(\theta^\star)\), the Lipschitz condition gives
\[
    \max_{1\le t\le T}
    |\check y_t^{(h_k)}-\check y_t^{(\theta^\star)}|
    \le
    L_\theta r .
\]
Therefore the band average satisfies
\[
    \max_{1\le t\le T}
    |\bar y_t^{S_r(\theta^\star)}-\check y_t^{(\theta^\star)}|
    \le
    L_\theta r .
\]
For clipped predictions and outcomes in \([-B,B]\), squared loss is
\(4B\)-Lipschitz in the prediction:
\[
    |(a-y)^2-(b-y)^2|
    \le
    4B|a-b|,
    \qquad
    a,b,y\in[-B,B].
\]
Hence
\[
    L_T(\bar y^{S_r(\theta^\star)})
    \le
    L_T(\check y^{(\theta^\star)})
    +
    4B L_\theta T r .
\]
Combining this inequality with Corollary~\ref{cor:band-oracle} proves
\eqref{eq:local-band-control}. The final display follows by applying
Theorem~\ref{thm:rls-dynamic-app} to the EWLS scale \(h^\star=e^{\theta^\star}\).
\end{proof}

% -----------------------------------------------------------------------------
\subsection{Implementation remarks: covariance inflation and nominal scale}
\label{app:implementation-remarks}
% -----------------------------------------------------------------------------

\begin{remark}[Covariance inflation]
\label{rem:implementation-rls}
The deterministic proof above analyses the idealised forgotten-RLS recursion
\eqref{eq:app-rls-recursion}. The implementation used in the experiments adds
a small covariance-inflation term to the \(\bP_t=\bA_t^{-1}\) recursion for
numerical stability and responsiveness, with the inflation scale selected on
the 2018 walk-forward validation period. This modification is an implementation
device rather than the source of the theoretical stability condition: adding
\(\varepsilon\bI\) to the covariance caps overconfident precision directions but
does not by itself impose a uniform lower bound on
\(\lambda_{\min}(\bA_t)\). We therefore treat \(h(\gamma)=1/(1-\gamma)\) in the
experiments as a \emph{nominal} forgetting scale rather than a calibrated number
of days of memory.
\end{remark}

\newpage

% -----------------------------------------------------------------------------
\section{Experimental protocol and implementation details}
\label{app:implementation}
% -----------------------------------------------------------------------------

% -----------------------------------------------------------------------------
\subsection{The algorithm}
\label{app:algo}
% -----------------------------------------------------------------------------

\begin{algorithm}[t]
\caption{Multi-Time-Scale EWLS with MLpol Aggregation}
\label{alg:main}
\begin{algorithmic}[1]
\REQUIRE Base experts ${\{f^{(j)}\}}_{j=1}^M$; forgetting factors
$\Gamma=\{\gamma_1,\ldots,\gamma_K\}$; inflation levels
$\{\varepsilon_k\}_{k=1}^K$; diffuse-prior scale $\delta_0>0$; optional clipping radius $B$
\STATE Set $N \leftarrow M+K$
\STATE \textbf{Initialise:}
$\bw_1^{(k)}=\bm{0}$ and $\bP_0^{(k)}=\delta_0^{-1}\bI_{M+1}$ for $k=1,\ldots,K$;
$\bp_1=\bm{1}/N$; $\bR_0=\bm{0}$
\FOR{$t=1,2,\ldots,T$}
    \STATE Observe $\bx_t$ and form
    $\bz_t=\bigl(f^{(1)}(\bx_t),\ldots,f^{(M)}(\bx_t)\bigr)^\top$;
    set $\tilde{\bz}_t=(\bz_t^\top,1)^\top$
    \STATE \textit{// Expert predictions}
    \STATE $\tilde y_{t,j}=z_{t,j}$ for $j=1,\ldots,M$
    \hfill \COMMENT{raw base experts}
    \STATE $\tilde y_{t,M+k}=\tilde{\bz}_t^\top \bw_{t}^{(k)}$
    for $k=1,\ldots,K$
    \hfill \COMMENT{EWLS correction experts}
    \STATE Optionally clip each $\tilde y_{t,j}$ to $[-B,B]$
    \STATE \textit{// MLpol aggregate prediction}
    \STATE $\hat y_t=\sum_{j=1}^{N}p_{t,j}\tilde y_{t,j}$
    \STATE Observe $y_t$ and set $g_t=2(\hat y_t-y_t)$
    \STATE \textit{// EWLS/RLS updates}
    \FOR{$k=1,\ldots,K$}
        \STATE $s_t^{(k)} \gets \gamma_k+\tilde{\bz}_t^\top\bP_{t-1}^{(k)}\tilde{\bz}_t$
        \STATE $\bk_t^{(k)} \gets \bP_{t-1}^{(k)}\tilde{\bz}_t/s_t^{(k)}$
        \hfill \COMMENT{RLS/Kalman gain}
        \STATE $\bw_{t+1}^{(k)} \gets \bw_{t}^{(k)}
        +\bk_t^{(k)}\bigl(y_t-\tilde{\bz}_t^\top\bw_{t}^{(k)}\bigr)$
        \STATE $\bP_{t|t}^{(k)} \gets \gamma_k^{-1}
        \Bigl(\bP_{t-1}^{(k)}
        -\bP_{t-1}^{(k)}\tilde{\bz}_t\tilde{\bz}_t^\top\bP_{t-1}^{(k)}/s_t^{(k)}\Bigr)$
        \STATE $\bP_t^{(k)} \gets \bP_{t|t}^{(k)}+\varepsilon_k\bI$
        \hfill \COMMENT{small inflation carried to the next round}
    \ENDFOR
    \STATE \textit{// MLpol weight update}
    \STATE $\tilde r_{t,j}\gets g_t(\hat y_t-\tilde y_{t,j})$ for all $j=1,\ldots,N$
    \STATE $\bR_t\gets\bR_{t-1}+\tilde{\br}_t$
    \STATE $p_{t+1,j}\gets [R_{t,j}]_+/\sum_{i=1}^N[R_{t,i}]_+$
    for all $j=1,\ldots,N$, using $p_{t+1,j}=1/N$ if $\sum_i[R_{t,i}]_+=0$
\ENDFOR
\end{algorithmic}
\end{algorithm}

Algorithm~\ref{alg:main} gives the recursive update after the cold-start
initialization used in the experiments. The EWLS state update always uses the
unprojected RLS recursion. The clipping line in the algorithm is optional and
is included only to align with the bounded-loss version of the MLpol oracle
inequality in Appendix~\ref{app:deterministic-theory}; it is not activated in
the reported experiments. Thus the empirical results use the unclipped expert
predictions, while the theoretical aggregation guarantee applies to the
standard clipped variant. The single-scale EWLS tracking bound is stated
separately for the idealised unprojected recursion under the explicit stability
condition.

% -----------------------------------------------------------------------------
\subsection{RTE-FR Load data and splits}
\label{app:data-splits}
% -----------------------------------------------------------------------------

\paragraph{Raw series and daily target.}
We use the public French national electricity-load data described in
Section~\ref{sec:experiments}. The native short-term load-forecasting problem is
intraday: load is observed at half-hourly resolution, and production systems
typically forecast a vector of future half-hourly loads. In this paper we focus
on a scalar daily-load benchmark. For each calendar day \(t\), the target is the
average load over the intraday half-hourly observations,
\[
    y_t = \frac{1}{n_t}\sum_{q=1}^{n_t} \text{Load}_{t,q},
\]
where \(\text{Load}_{t,q}\) denotes the \(q\)-th half-hourly load observation on day \(t\).
On ordinary days \(n_t=48\). Thus the target remains an average power quantity,
and RMSE is reported in MW rather than in daily energy units. This aggregation
removes the need to model the intraday load shape and focuses the evaluation
on non-stationarity in the daily demand level.

\paragraph{Weather covariates.}
All methods are given the same temperature covariates constructed from realised
temperature observations rather than from day-ahead weather forecasts. This is
a controlled-covariate design: it prevents differences in weather-forecast quality
from confounding the comparison between online adaptation methods.
The experiment should
therefore be read as an ex-post benchmark of load forecasting conditional on
common weather information, not as a claim that realised target-day temperature
would be available in a real deployment. In deployment, the same pipeline would
instead receive forecast-temperature covariates.

\paragraph{Feature availability and leakage control.}
The feature vector contains calendar variables, holiday indicators, temperature
variables, and lagged load features. Calendar and holiday variables are known
before the target day. Lagged load features are constructed only from strictly
previous days, so no future load values enter the feature vector. During the test
period, all online updates are sequential: each method first predicts \(y_t\),
then observes \(y_t\), and only then updates its EWLS/RLS states and MLpol
weights. Thus the online adaptation step is causal with respect to load
labels. The only ex-post covariate convention is the realised-temperature choice above, which is shared by all compared methods.

\paragraph{Train, validation, and test periods.}
The data span 2012-01-01 to 2021-01-15. Hyperparameters are selected only within
the pre-test period: 2018 is used as a walk-forward validation window when
needed, and selected static base models are then refit on the full
2012--2018 window. All reported online results are evaluated sequentially on
2019-01-01 to 2021-01-15. The test period is split into pre-lockdown,
lockdown, and post-lockdown recovery regimes as described in
Section~\ref{sec:experiments}. No observation from the 2019--2021 test period is
used for base-model tuning, EWLS-grid selection, or nuisance-parameter
selection.

% \paragraph{Relation to operational day-ahead forecasting.}
% The daily-average benchmark is a simplified evaluation target. In practical
% short-term load forecasting, one often needs half-hourly forecasts for the next
% day, either through rolling horizon forecasts or through a morning forecast run
% that fixes the next day's full intraday trajectory. Public post-COVID
% day-ahead-load forecasting protocols similarly evaluate horizons on the order of
% 16--40 hours ahead. Our daily target averages the 48 half-hourly values to test
% whether the proposed online correction layer detects and adapts to regime
% changes in aggregate demand. The proposed aggregation framework itself is not
% restricted to daily targets: it can be applied separately by horizon/half-hour,
% or to a pool of vector-valued base forecasts, provided the corresponding base
% forecasts and weather-forecast covariates are available before the operational
% forecast deadline.

% =============================================================================
\subsection{Base-model features and tuning}
\label{app:hp-tune}
% =============================================================================

This appendix documents how base-model hyperparameters are selected
for the main experiment. The guiding rule is that no base-model 
hyperparameter is selected using any part of the test period 
beginning on 2019-01-01.

\paragraph{Scope.}
Of the $M=7$ base predictors, hyperparameters are set as follows:
\begin{itemize}[leftmargin=1.4em,itemsep=1pt,topsep=2pt]
    \item \textit{Lag-1}: no hyperparameters.
    \item \textit{GAM}: structural specification taken from 
    \citet{gaillard2016additive}; fitted with library-level defaults of
    \texttt{mgcv}, with no test-period tuning. The exact formula is
    given below.
    \item \textit{Ridge, XGBoost, LightGBM, ResNet, FT-Transformer}:
    tuned on a held-out 2018 validation window within the training
    period, using the Optuna-based protocol described below.
\end{itemize}
For reproducibility, the released code includes both the selected
configurations used in the main experiment and the tuning routine used
to obtain them. Re-running the tuning stage is therefore optional for
reproducing the reported results.

\paragraph{GAM specification.}
The additive structure follows the semi-parametric load-forecasting
template established by \citet{goude2014local} and refined in
\citet{gaillard2016additive}, which couples thin-plate smooths of
exponentially smoothed temperatures with calendar effects and lagged
load. The exponentially smoothed temperature features
($\mathrm{Temp}^{s95}$ and its daily min/max) and the
$\mathrm{toy}$ time-of-year covariate with cyclic spline basis
follow the conventions of the \texttt{opera} R package
\citep{gaillard2016additive}. Writing $\eta_t$ for the linear
predictor of daily load at day $t$,
\begin{equation*}
\begin{aligned}
\eta_t \;=\;
  & \,\mathrm{te}\!\left(\ell(\mathrm{Load}_{t-1}),\,
                          \mathrm{f}(\mathrm{WeekDay}_t)\right)
   + \ell(\mathrm{Load}_{t-7}) \\
  & + s(\mathrm{Temp}_t)
    + s(\mathrm{Temp}^{s95}_{t})
    + s(\mathrm{Temp}^{s95,\min}_{t})
    + s(\mathrm{Temp}^{s95,\max}_{t}) \\
  & + \mathrm{f}(\mathrm{WeekDay}_t)
    + \mathrm{f}(\mathrm{BH}_t)
    + s_{\mathrm{cp}}(\mathrm{toy}_t)
    + s(\mathrm{Month}_t),
\end{aligned}
\end{equation*}
where $\ell(\cdot)$ is a linear term, $s(\cdot)$ a thin-plate
regression smooth, $s_{\mathrm{cp}}(\cdot)$ a cyclic cubic smooth
(used on \emph{toy}, time-of-year), $\mathrm{f}(\cdot)$ a categorical
factor, and $\mathrm{te}(\cdot,\cdot)$ a tensor-product interaction.
The covariates $\mathrm{Temp}^{s95}$ and its daily min/max variants
are exponentially smoothed temperatures (smoothing factor $0.95$);
$\mathrm{BH}$ flags French bank holidays. All smoothing parameters
are estimated by REML; no covariate, basis dimension, or smoothness
penalty is tuned on the test period.

\paragraph{Train/validation split.}
Each tuned model is fit on 2012-01-01 to 2017-12-31 and scored on a
2018 validation window (2018-01-01 to 2018-12-31) by RMSE on the
one-step-ahead prediction task. After tuning, the selected
configuration is refit on the full 2012--2018 window for evaluation on
the 2019--2021 test period. No test-period information enters either
the search or the final fit.

\paragraph{Separation of convergence and modelling hyperparameters.}
Convergence controls---optimiser \texttt{max\_iter}/\texttt{tol} for
linear models, training \texttt{epochs}/\texttt{patience} for deep
models---are fixed a priori rather than searched. This avoids conflating
optimisation budget with modelling capacity: if \texttt{epochs} is
searched as a free parameter, the search tends to select the largest
budget available. We fix the deep-model training loop at
$5000$ epochs with early-stopping patience of $200$ validation steps,
which is sufficient for every retained configuration to converge on
this data.

\paragraph{Search algorithm and trial budget.}
We use Optuna \citep{akiba2019optuna} with the Tree-structured Parzen
Estimator (TPE) sampler and seed $42$. For the two deep models, we
additionally enable Optuna's \texttt{MedianPruner} (5 startup trials,
20 warmup epochs, interval 1), which terminates a trial early if its
validation RMSE at the current epoch lies above the median of prior
trials at the same epoch. Linear and tree models are run without
pruning. Trial budgets reflect per-trial cost: linear models receive
$2000$ trials, tree ensembles $500$, and deep models $200$ completed
or pruned trials. Extending the search to $2\times$ the budget did not
improve any of the retained configurations.

\paragraph{Search spaces.}
Table~\ref{tab:hp_tuning_search_space} reports the per-model search
spaces used in the reported runs. Ranges are deliberately generous
rather than tight around prior beliefs. Log-scale ranges are indicated
by ``log''; step-discretised integer ranges by ``step $s$''.

\begin{table}[ht]
    \centering
    \small
    \setlength{\tabcolsep}{5pt}
    \caption{Per-model Optuna search spaces used for base-model
    hyperparameter tuning. ``log'': log-uniform sampling. ``step $s$'':
    discretised with stride $s$. Convergence controls
    are fixed a priori and not searched.}
    \label{tab:hp_tuning_search_space}
    \begin{tabular}{llll}
        \toprule
        Model & Hyperparameter & Range & Notes \\
        \midrule
        Ridge & \texttt{alpha} & $[10^{-4},\,10^{4}]$ & log \\
        \midrule
        \multirow{8}{*}{XGBoost}
          & \texttt{n\_estimators}    & $[100,\,2000]$           & step $100$ \\
          & \texttt{learning\_rate}   & $[10^{-3},\,0.3]$        & log \\
          & \texttt{max\_depth}       & $[2,\,10]$               & integer \\
          & \texttt{min\_child\_weight} & $[10^{-2},\,20]$       & log \\
          & \texttt{subsample}        & $[0.3,\,1.0]$            & uniform \\
          & \texttt{colsample\_bytree} & $[0.3,\,1.0]$            & uniform \\
          & \texttt{reg\_alpha}       & $[10^{-8},\,10]$         & log \\
          & \texttt{reg\_lambda}      & $[10^{-8},\,10]$         & log \\
        \midrule
        \multirow{9}{*}{LightGBM}
          & \texttt{n\_estimators}    & $[100,\,2000]$           & step $100$ \\
          & \texttt{learning\_rate}   & $[10^{-3},\,0.3]$        & log \\
          & \texttt{num\_leaves}      & $[7,\,255]$              & log \\
          & \texttt{max\_depth}       & $[-1,\,12]$              & integer ($-1=$ unlimited) \\
          & \texttt{min\_child\_samples} & $[5,\,100]$           & step $5$ \\
          & \texttt{subsample}        & $[0.5,\,1.0]$            & uniform \\
          & \texttt{colsample\_bytree} & $[0.5,\,1.0]$           & uniform \\
          & \texttt{reg\_alpha}       & $[10^{-8},\,10]$         & log \\
          & \texttt{reg\_lambda}      & $[10^{-8},\,10]$         & log \\
        \midrule
        \multirow{8}{*}{ResNet}
          & \texttt{n\_layers}        & $[1,\,8]$                & integer \\
          & \texttt{d}                & $[64,\,512]$             & step $64$ \\
          & \texttt{d\_hidden\_factor} & $[1.0,\,4.0]$           & uniform \\
          & \texttt{hidden\_dropout}  & $[0,\,0.5]$              & uniform \\
          & \texttt{residual\_dropout} & $[0,\,0.5]$             & uniform \\
          & \texttt{d\_embedding}     & $[4,\,64]$               & step $4$ \\
          & \texttt{lr}               & $[10^{-5},\,10^{-2}]$    & log \\
          & \texttt{batch\_size}      & $\{128, 256, 512\}$      & categorical \\
        \midrule
        \multirow{8}{*}{FT-Transformer}
          & \texttt{n\_layers}        & $[1,\,6]$                & integer \\
          & \texttt{d\_token}         & $\{64, 128, 192, 256\}$  & categorical \\
          & \texttt{n\_heads}         & $\{4, 8\}$               & categorical \\
          & \texttt{d\_ffn\_factor}   & $[2/3,\,8/3]$            & uniform \\
          & \texttt{attention\_dropout} & $[0,\,0.5]$            & uniform \\
          & \texttt{ffn\_dropout}     & $[0,\,0.5]$              & uniform \\
          & \texttt{residual\_dropout} & $[0,\,0.2]$             & uniform \\
          & \texttt{lr}               & $[10^{-4},\,10^{-2}]$    & log \\
        \bottomrule
    \end{tabular}
\end{table}

\paragraph{Selected configurations and reproducibility.}
The selected configuration for each model---including tuned modelling
hyperparameters and fixed convergence controls---is included in the
released supplementary code. Reproducing Section~\ref{sec:experiments}
therefore requires only the released configurations and does not require
re-running the hyperparameter search. Re-running the tuning protocol
with the same validation split and random seed reproduces the selected
configurations up to the usual library- and hardware-level numerical
variation. Total tuning cost is approximately $3$~hours for the non-deep
models and $6$~hours for ResNet and FT-Transformer combined; hardware
details are given in the Compute paragraph below.

\paragraph{Foundation-model baselines.}
TabPFN and TabICL are applied with their realised library defaults.
The expanding-window online variants are rerun at each test step using
the available historical sample, with only data observed before the
prediction. We do not tune task-specific hyperparameters for these
models, preserving their role as strong out-of-the-box baselines.

% -----------------------------------------------------------------------------
\subsection{EWLS implementation and covariance inflation}
\label{app:ewls-impl}
% -----------------------------------------------------------------------------

Algorithm~\ref{alg:main} differs from the idealised EWLS recursion in Section~\ref{sec:method} 
in one implementation detail: after the standard forgotten-RLS covariance update, 
we add a small
$\gamma$-dependent covariance inflation term $\varepsilon_k \bI$. 
This term is used for numerical stability and responsiveness in finite samples.
It is not used to justify the bounded-iterate condition in Appendix~\ref{app:deterministic-theory}.

\paragraph{Cold start for EWLS experts.}
The diffuse-prior initialisation in Algorithm~\ref{alg:main} describes the recursive update after the EWLS state has been initialised. In the experiments we use a short cold-start period of $M+5$ test observations for the EWLS experts.
During this period, each EWLS correction expert outputs the uniform average of the raw base predictions,
\[
    \check y_t^{(\gamma_k)}
    = 
    \frac{1}{M} \sum_{j=1}^M f^{(j)}(\bx_t),
    \qquad t\le M+5,
\]
while MLpol treats these predictions like ordinary expert predictions and updates its cumulative pseudo-regrets normally.
After observing the first \(M+5\) labels, for each \(\gamma_k\) we initialise each EWLS state by
solving the corresponding batch EWLS problem on those cold-start observations, obtaining the pre-update coefficient $\bw_{M+6}^{(k)}$ and covariance $\bP_{M+5}^{(k)}$.
From the round $M+6$ onward, the EWLS experts use the RLS recursion in Algorithm~\ref{alg:main}. Thus the cold start affects
only the internal EWLS states; the MLpol layer is active from the first test day and all updates remain causal.

\paragraph{Covariance inflation with $\gamma$-dependent scaling.}
The implemented covariance recursion is
\begin{equation}
    \label{eq:cov-inflation}
    \bP_t^{(k)} \;\leftarrow\; \gamma_k^{-1}\!\left(
      \bP_{t-1}^{(k)} - \frac{\bP_{t-1}^{(k)} \tilde{\bz}_t \tilde{\bz}_t^\top \bP_{t-1}^{(k)}}{s_t^{(k)}}
    \right) + \varepsilon_k\, \bI,
    \qquad 
    s_t^{(k)} = \gamma_k + \tilde{\bz}_t^\top\bP_{t-1}^{(k)}\tilde{\bz}_t.
\end{equation}
Under the Kalman-filter interpretation of RLS, adding \(\varepsilon_k\bI\) to the covariance 
corresponds to a small process-noise covariance in the implicit state equation
$\bw_{t+1} = \bw_t + \boldsymbol{\eta}_t$, with $\mathrm{Cov}(\boldsymbol{\eta}_t) = \varepsilon_k \bI$.
It therefore encodes the
expert's prior scale for how much the combination coefficients may move from one step to the next.

A single shared $\varepsilon_k \equiv \varepsilon$ is unsatisfactory at both
ends of the $\gamma$-grid. For large $\gamma$, the EWLS expert is intended to 
represent a long-memory, stable correction; too much injected covariance would 
dominate the slow forgetting and effectively shorten its memory.
For smaller $\gamma$, past-data contributions are rapidly discounted, and
the recursive covariance can become numerically unstable along directions not recently excited.
The inflation level should therefore be coupled to the forgetting scale.

We use the power-law schedule
\begin{equation}
\varepsilon_k \;=\; \varepsilon_0\, (1-\gamma_k)^{\alpha}, \qquad \alpha=1. 
\label{eq:eps-schedule}
\end{equation}
The exponent $\alpha=1$ is a scale-matching choice. For an EWLS
expert with forgetting factor $\gamma$, the total discounted mass of past observations is 
$\sum_{\ell \ge 0} \gamma^\ell = 1/(1-\gamma)$. Consequently, when the feature directions are of comparable scale,
the weighted Gram matrix has nominal magnitude $(1-\gamma)^{-1}$, so its inverse has
nominal magnitude $(1-\gamma)$. Setting $\varepsilon_k \;=\; \varepsilon_0\, (1-\gamma_k)^{\alpha}$
therefore keeps the added covariance on the same nominal inverse-Gram scale across the forgetting grid. 
Long-memory experts receive less inflation, so that their slow scale is not overwhelmed; short-memory experts receive more inflation, which
helps numerical stability when recent data provide limited curvature.
This argument is only a scale heuristic for the implementation. 
The pathwise guarantees in Appendix~\ref{app:deterministic-theory} are proved for the idealised recursion without covariance
inflation, and $\varepsilon_0$ is chosen by walk-forward validation on 2018 (Appendix~\ref{app:eps_selection}).

\paragraph{Relation to the theoretical analysis.}
The deterministic analysis in Section~\ref{sec:theory} and
Appendix~\ref{app:deterministic-theory} studies the idealised forgotten-RLS
recursion without the \(\varepsilon_k\bI\) term. The inflation in
\eqref{eq:cov-inflation} changes the exact recursion and therefore the
calibrated mapping between \(\gamma\) and a physical effective sample size. It
should not be read as a proof of the bounded-iterate condition (B4): adding
\(\varepsilon_k\bI\) to the covariance caps overconfident precision directions
but does not by itself impose a uniform lower bound on
\(\lambda_{\min}(\bA_t)\). Accordingly, \(h(\gamma)=1/(1-\gamma)\) is treated
throughout the experiments as a \emph{nominal} forgetting-scale index: smaller
\(h\) corresponds to more responsive corrections and larger \(h\) to more
stable corrections, but the absolute value of \(h\) should not be interpreted
as an exact number of days of memory.

% -----------------------------------------------------------------------------
\subsection{Selection of \texorpdfstring{$\varepsilon_0$}{epsilon0} }
\label{app:eps_selection}
% -----------------------------------------------------------------------------

The covariance-inflation scale \(\varepsilon_0\) is the only EWLS implementation
parameter selected by validation. The exponent \(\alpha=1\) and the
forgetting-factor grid are fixed a priori. We select \(\varepsilon_0\) by a
walk-forward validation protocol on a window preceding the main test period;
the main test period (2019-01-01 onwards) is never inspected during selection.

\paragraph{Protocol}
We run a separate walk-forward selection pipeline with all data boundaries 
shifted one year earlier than in the main experiment. All training and 
initialisation stages of this pipeline are isolated from 2018; 2018 is used 
only as the final out-of-sample validation window:
\begin{itemize}[leftmargin=1.4em,itemsep=1pt,topsep=2pt]
  \item \textbf{Base-model hyperparameters.} The base-model hyperparameters
  are re-tuned from scratch under the shifted protocol, using 2012--2016 as
  the training window and 2017 as the tuning validation window. They are not
  reused from the main run of Appendix~\ref{app:hp-tune}, preventing
  2018 information from entering through hyperparameter choices.
  \item \textbf{Base-model fits.} The selected configurations are refit on
  2012--2017. Neural networks use a tail-of-2017 split for early stopping,
  while static models are refit on the full 2012--2017 window. The resulting
  base predictors have never seen 2018.

  \item \textbf{Aggregation initialisation.} The EWLS and MLpol states are
  initialised using predictions and labels from 2012--2017 only.
  
  \item \textbf{Validation window for $\varepsilon_0$.} Candidate values are
  scored on 2018-01-01 to 2018-12-31, a 365-day window of genuine
  out-of-sample predictions for every base model. During 2018 the online
  updates follow the same predict-then-update protocol as in the main
  evaluation.
\end{itemize}
Thus, no 2018 information enters the base-hyperparameter search, base-model
fit, or aggregation initialisation. The 2018 RMSE is used only to choose
$\varepsilon_0$.

For each $\varepsilon_0$ in the logarithmic grid
$\{10^{-13},10^{-12},\ldots,10^{-5}\}$, with $\alpha=1$ fixed, we run the
full EWLS\,+\,MLpol pipeline and record the RMSE on the 2018 validation
window. The minimiser is then fixed for the main run and is not revised using
the 2019+ test period. The selection script enforces that the validation window 
cannot end after 2018-12-31, thereby preventing $2019+$ dates from entering the selection.

\paragraph{Sweep curve}
Figure~\ref{fig:eps_sweep} shows the resulting validation RMSE as a function
of $\varepsilon_0$.

\begin{figure}[ht]
    \centering
    \begin{tikzpicture}
      \begin{axis}[
        width=11cm, height=6.5cm,
        xlabel={$\varepsilon_0$ (log scale)},
        ylabel={RMSE on 2018 out-of-sample (MW)},
        xmode=log, log basis x=10,
        xmin=3e-14, xmax=3e-5,
        ymin=660, ymax=710,
        xtick={1e-13,1e-12,1e-11,1e-10,1e-9,1e-8,1e-7,1e-6,1e-5},
        xticklabels={$10^{-13}$,$10^{-12}$,$10^{-11}$,$10^{-10}$,$10^{-9}$,$10^{-8}$,$10^{-7}$,$10^{-6}$,$10^{-5}$},
        grid=both,
        grid style={gray!15},
        tick align=outside,
        legend style={draw=none, fill=none, at={(0.5,-0.25)},
                      anchor=north, legend columns=3, font=\small},
      ]
        % Plateau shading (1e-9 to 1e-7, within 1% of best)
        \fill[teal!15]
          (axis cs:1e-9,660) rectangle (axis cs:1e-7,710);
        % Base-only reference line
        \addplot[dashed, gray!70, thick]
          coordinates {(3e-14,705.3899) (3e-5,705.3899)};
        \addlegendentry{Base-only (705.3899)}
        % Main curve
        \addplot[color=purple!80!black, thick, mark=*, mark size=1.8pt]
          coordinates {
            (1e-13, 694.8306) (1e-12, 694.8214) (1e-11, 695.2060)
            (1e-10, 690.6253) (1e-9,  681.6095) (1e-8,  677.3405)
            (1e-7,  678.3108) (1e-6,  683.3872) (1e-5,  687.6214)
          };
        \addlegendentry{MLpol on Base+EWLS}
        % Highlight selected point
        \addplot[only marks, mark=*, mark size=4pt, color=red!70!black]
          coordinates {(1e-8, 677.3405)};
        \addlegendentry{Selected $\varepsilon_0=10^{-8}$}
      \end{axis}
    \end{tikzpicture}
    \caption{\textbf{Walk-forward $\varepsilon_0$ sweep on 2018 out-of-sample.}
    The validation curve has an interior minimum at $\varepsilon_0=10^{-8}$
    (RMSE $677.3$). The shaded region marks a conservative two-decade plateau
    $[10^{-9},10^{-7}]$, where the RMSE remains within $1\%$ of the best value.
    The neighbouring point at $10^{-6}$ is also close, but the curve has already
    started to rise. The left tail worsens when the inflation is too small relative
    to the conditioning of the EWLS covariance recursion, while the right tail
    worsens when excessive inflation prevents high-$\gamma$ experts from retaining
    their intended long-memory behaviour. The main test period, starting on
    2019-01-01, is not used in this selection.}
    \label{fig:eps_sweep}
\end{figure}
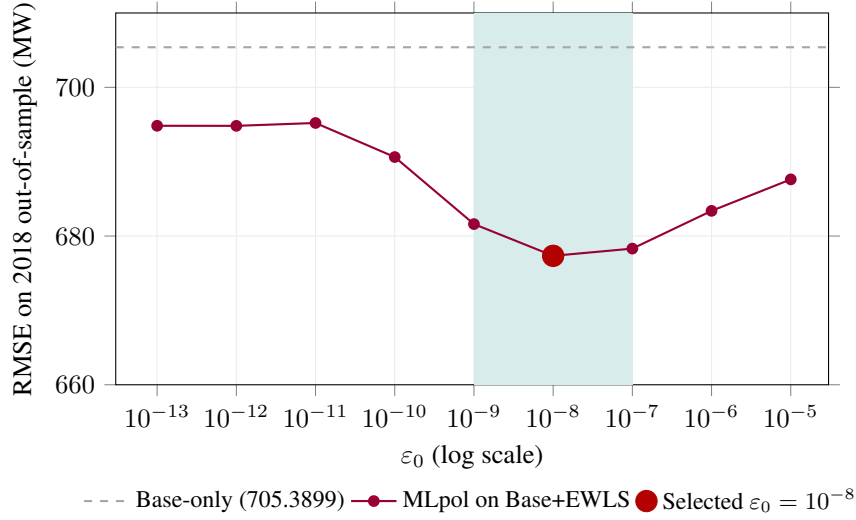

\paragraph{Diagnostics}

\begin{itemize}[leftmargin=1.4em,itemsep=1pt,topsep=2pt]
    \item \textbf{Interior minimum, both sides bracketed.}
    $\varepsilon_0=10^{-8}$ is not at a grid boundary. Performance worsens
    when $\varepsilon_0$ is made much smaller than the useful range
    ($\varepsilon_0 \leq 10^{-11}$), and the curve rises again for larger
    values, with a clear deterioration by $\varepsilon_0=10^{-5}$. Thus the
    grid brackets the useful scale rather than selecting an endpoint.

    \item \textbf{Wide plateau, robust selection.}
    The three grid points $\{10^{-9},10^{-8},10^{-7}\}$, spanning two orders of
    magnitude, are all within $1\%$ of the best RMSE. The neighbouring value
    $10^{-6}$ is also close, but the upward trend has already begun. The
    validation procedure therefore identifies a stable scale rather than a
    finely tuned scalar, mitigating concerns that $\varepsilon_0$ is overfit to
    the particular 2018 validation year.

    \item \textbf{EWLS layer helps even in a non-crisis validation year.}
    MLpol on Base only, an $\varepsilon_0$-independent reference, achieves
    RMSE $705.4$ on 2018. At the selected $\varepsilon_0 = 10^{-8}$, the full
    Base+EWLS pool reaches $677.3$, a $28.0$\,MW ($4.0\%$) improvement even
    in a year without a COVID-scale regime shift. This suggests that the EWLS
    layer can add value beyond the base experts alone, even outside extreme
    regime shifts.
\end{itemize}

\paragraph{Why we do not use in-sample validation}
An alternative validation protocol is to sweep $\varepsilon_0$ using
in-sample base predictions on the training period (2012--2017), evaluated
on the last year of training. We report such a sweep for completeness: it
prefers the smallest $\varepsilon_0$ in the grid and does not produce an
interior minimum over $[10^{-13},\,10^{-5}]$. This behaviour is consistent with
the theoretical role of $\varepsilon_0$: in-sample residuals can under-represent
the out-of-sample variation that an EWLS expert must track, biasing the sweep
toward insufficient injection. The walk-forward protocol above avoids this
failure mode by evaluating on genuine out-of-sample predictions.

\paragraph{Agreement with an unrestricted test-period sweep}
As a post-hoc robustness check, not used for selection, we also swept
$\varepsilon_0$ directly on the main test period under the same diagnostic
setup. The absolute RMSE values from this diagnostic run are not meant to
replace the final main-experiment numbers; they are reported only to compare
the selected and test-optimal $\varepsilon_0$ under a common setup. The
test-optimal value $\varepsilon_0=10^{-9}$ attains overall RMSE $660.57$,
while the walk-forward-selected value $\varepsilon_0=10^{-8}$ attains
$661.61$. The $1.04$\,MW ($0.16\%$) gap is within the plateau already
identified on 2018, and quantifies the small protocol cost of selecting
$\varepsilon_0$ without touching the test period.

% -----------------------------------------------------------------------------
\subsection{Compute and reproducibility}
% -----------------------------------------------------------------------------

\paragraph{Compute.} All experiments were run on a single Apple M4 Pro 
workstation (14-core CPU, 20-core integrated GPU, 48\,GB unified memory). 
Deep-model training (ResNet, FT-Transformer) used the PyTorch MPS backend; 
all other components ran on CPU. Base-model hyperparameter tuning took 
approximately 3 hours for the linear and tree ensembles combined and 
approximately 6 hours for ResNet and FT-Transformer combined. The main 
online-aggregation pipeline over the 746-day test period with $N=23$ 
experts runs end-to-end in under 10\,s on CPU---the recursion is dominated 
by $M\times M$ linear-algebra updates and is not the computational 
bottleneck. The paired moving block bootstrap 
(Appendix~\ref{app:rte_bootstrap}, $R=10{,}000$) runs in under 30\,s.

% -----------------------------------------------------------------------------
\subsection{Empirical bounded-iterate diagnostic}
\label{app:bounded-iterate-diag}
% -----------------------------------------------------------------------------

The pathwise tracking guarantee in Theorem~\ref{thm:rls-dynamic}
(restated as Theorem~\ref{thm:rls-dynamic-app} in
Appendix~\ref{app:deterministic-theory}) is conditional on the
bounded-iterate condition~(B4):
\(\|\bw_t^{(\gamma)}\|_2 \le R\) uniformly along the realised sequence.
This appendix provides an empirical diagnostic for this condition on
RTE-FR Load.

\paragraph{Decomposing slope and intercept.}
The EWLS expert is fitted on the augmented base-prediction vector
\[
    \widetilde \bz_t = (\bz_t^\top,1)^\top ,
\]
so that
\[
    \bw_t^{(\gamma)}
    =
    \bigl(\bm\beta_t^{(\gamma)},\alpha_t^{(\gamma)}\bigr),
\]
where \(\bm\beta_t^{(\gamma)}\in\mathbb{R}^M\) combines the base forecasts
and \(\alpha_t^{(\gamma)}\in\mathbb{R}\) is an intercept. These two components
have different scales. The slope vector is the online combination component:
large oscillations in \(\|\bm\beta_t^{(\gamma)}\|_2\) would indicate unstable
adaptation of the base-forecast weights. The intercept, in contrast, absorbs
the level of the response. Since daily French electricity load is measured in
megawatts and is of order \(10^4\)--\(10^5\), the augmented norm
\(\|\bw_t^{(\gamma)}\|_2\) can be dominated by
\(|\alpha_t^{(\gamma)}|\) even when the combination weights remain stable.

\paragraph{Relation to the bounded-iterate assumption.}
Condition~(B4) is a sufficient condition used to make the pathwise bound
finite. In the augmented parametrisation, its numerical value mixes two
effects: the scale-free stability of the combination slope and the
scale-dependent magnitude of the intercept. The diagnostic below therefore
reports both the full augmented norm and the slope-only norm. The former checks
the literal quantity appearing in~(B4), while the latter isolates the part of
the iterate that governs online aggregation of base forecasts.

\paragraph{Empirical result.}
Figure~\ref{fig:bounded-iterate-diagnostic} shows the trajectories for the
\(K-1=15\) finite-forgetting EWLS experts over the 2019--2021 test period.
The no-forgetting endpoint \(\gamma=1\) is excluded because
Theorem~\ref{thm:rls-dynamic} is stated for finite forgetting factors.
Two observations are useful.

First, the full augmented norm is mainly driven by the intercept. Its running
supremum is \(1.50\times 10^4\), which is on the same scale as the response
level in megawatts. This behaviour reflects the scale of the intercept.

Second, after removing the intercept coordinate, the slope norm remains small:
\[
    \sup_{t,\gamma}\|\bm\beta_t^{(\gamma)}\|_2 = 1.69
\]
over all finite-forgetting experts and all \(T=746\) test days. The
fastest-forgetting experts respond most strongly around the onset of the
COVID-19 lockdown, where the running supremum is attained, while the
long-memory experts change more gradually. After the lockdown period, the slope
norm relaxes to a range around \(0.5\)--\(0.7\) across most of the grid.

Overall, this diagnostic supports the plausibility of the bounded-iterate
condition along the realised RTE-FR Load trajectory: the literal augmented
iterate remains finite, and the scale-free combination component remains
uniformly bounded by a small constant. We therefore use this figure as a
sanity check for the conditional pathwise guarantee, rather than as an
independent proof of~(B4).

\begin{figure}[t]
\centering
\includegraphics[width=\linewidth]{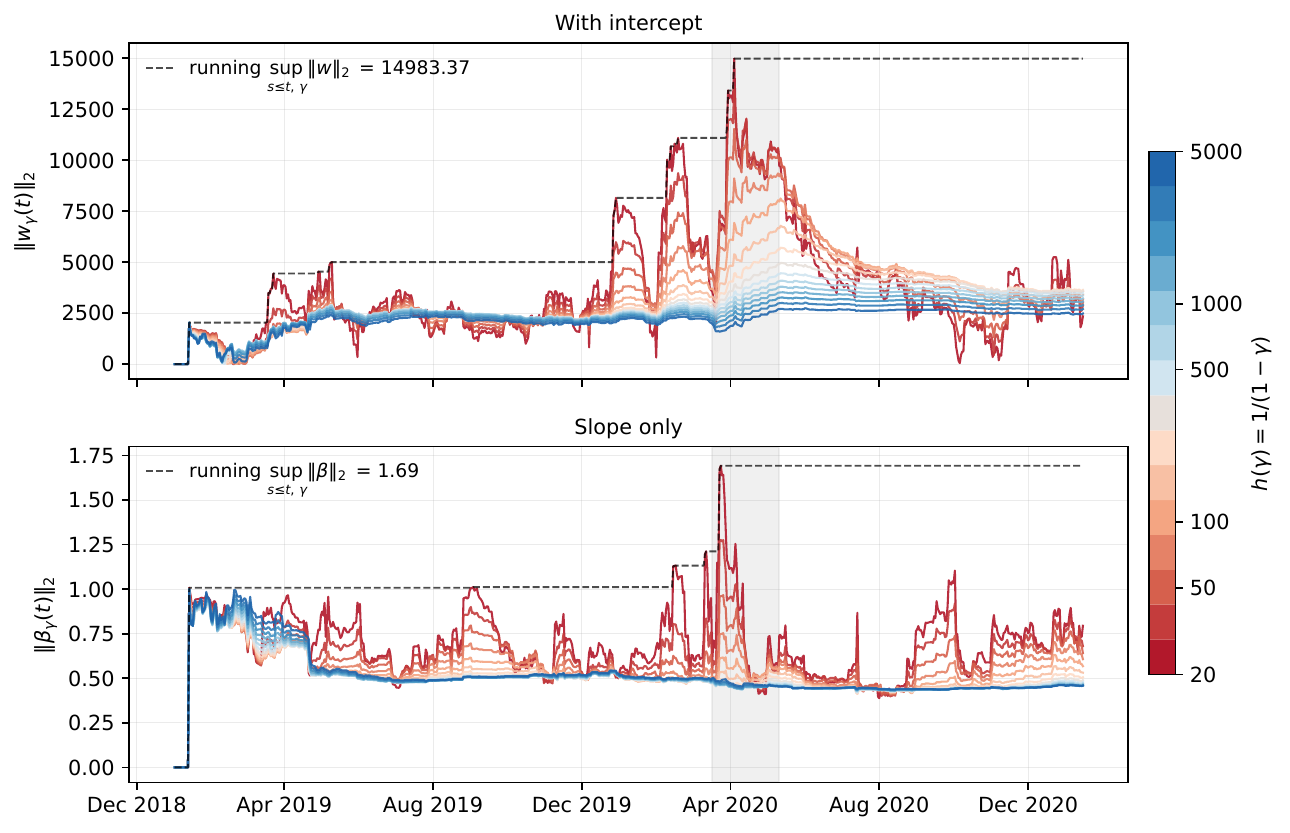}
\caption{\textbf{Empirical bounded-iterate diagnostic on RTE-FR Load.}
Trajectory of the EWLS coefficient norm for the \(K-1=15\)
finite-forgetting experts over the 2019--2021 test period. The
no-forgetting endpoint \(\gamma=1\) is outside the finite-forgetting setting
of Theorem~\ref{thm:rls-dynamic} and is excluded. Curves are coloured by
effective memory \(h(\gamma)=1/(1-\gamma)\), with short-memory experts in warm
colours and long-memory experts in cool colours; the shaded band marks the
COVID-19 lockdown period.
\textit{Top:} full augmented norm \(\|\bw_t^{(\gamma)}\|_2\). The running
supremum is \(1.50\times 10^4\), reflecting the scale of the intercept.
\textit{Bottom:} slope-only norm \(\|\bm\beta_t^{(\gamma)}\|_2\), with the
intercept coordinate removed. The running supremum is \(1.69\), attained near
the lockdown onset. Thus the EWLS slope iterate responds to the structural
break while remaining uniformly bounded across forgetting factors and test
days.}
\label{fig:bounded-iterate-diagnostic}
\end{figure}

% -----------------------------------------------------------------------------
\section{Detailed RTE-FR results and statistical significance}
\label{app:rte_detailed_results}
% -----------------------------------------------------------------------------

\subsection{Full per-regime RMSE table}
\label{app:rte_full_rmse}

Table~\ref{tab:rte_full_rmse} reports the full per-regime RMSE table behind
Table~\ref{tab:main_results}. The rows include the final aggregates, all standalone
EWLS/RLS correction experts, external foundation-model references, and all
individual base-pool models. The column ``Excl. warm-up'' reports the same
overall test metric after removing the initial \(M+5\) cold-start observations
used to initialise the EWLS states; all methods are evaluated on the same
remaining dates in that column.

\begin{table*}[t]
    \centering
    \caption{
    Full per-regime RMSE (MW) on RTE-FR Load. Rows are sorted by overall RMSE.
    ``Excl. warm-up'' removes the initial \(M+5\) EWLS cold-start observations from
    the overall evaluation window. External online TabICL/TabPFN variants are rerun
    daily on an expanding context window; ``+GRI'' denotes the additional
    regime-informed Oxford Government Response Index covariate.
    }
    \label{tab:rte_full_rmse}
    \scriptsize
    \setlength{\tabcolsep}{3.5pt}
    \renewcommand{\arraystretch}{0.96}
    \resizebox{\textwidth}{!}{%
    \begin{tabular}{llrrrrr}
        \toprule
        Group & Method & Pre-lockdown & Lockdown & Post-lockdown & Excl. warm-up & Overall \\
        \midrule
        Final agg. & \textsc{MELO} (Base+EWLS) & 623.1 & 1086.1 & 579.3 & 651.5 & 655.8 \\
        Final agg. & MLpol on EWLS only & 653.9 & 1073.6 & 584.1 & 656.9 & 673.2 \\
        \midrule
        EWLS expert & EWLS, $\gamma=0.998563$ & 657.9 & 1168.3 & 585.4 & 673.0 & 687.5 \\
        EWLS expert & EWLS, $\gamma=0.997868$ & 661.5 & 1138.2 & 592.8 & 673.0 & 687.9 \\
        EWLS expert & EWLS, $\gamma=0.999031$ & 656.2 & 1204.0 & 579.9 & 675.6 & 689.7 \\
        EWLS expert & EWLS, $\gamma=0.996838$ & 666.9 & 1110.9 & 602.1 & 675.2 & 690.3 \\
        EWLS expert & EWLS, $\gamma=0.995309$ & 673.9 & 1084.1 & 613.2 & 679.1 & 694.4 \\
        EWLS expert & EWLS, $\gamma=0.999347$ & 656.6 & 1247.8 & 576.1 & 681.1 & 694.7 \\
        EWLS expert & EWLS, $\gamma=0.993041$ & 682.2 & 1056.1 & 625.3 & 684.0 & 699.5 \\
        EWLS expert & EWLS, $\gamma=0.999560$ & 659.0 & 1301.9 & 574.0 & 689.9 & 702.8 \\
        EWLS expert & EWLS, $\gamma=0.989676$ & 691.2 & 1028.3 & 637.3 & 689.7 & 705.2 \\
        EWLS expert & EWLS, $\gamma=0.984685$ & 700.6 & 1006.9 & 648.2 & 696.1 & 711.7 \\
        EWLS expert & EWLS, $\gamma=0.999703$ & 663.3 & 1367.9 & 573.6 & 702.2 & 714.4 \\
        EWLS expert & EWLS, $\gamma=0.977280$ & 710.9 & 998.8 & 659.2 & 704.7 & 720.2 \\
        \midrule
        External & TabICL (online, +GRI) & 668.0 & 1132.6 & 699.5 & 710.2 & 723.4 \\
        EWLS expert & EWLS, $\gamma=0.999800$ & 669.2 & 1446.3 & 575.2 & 718.2 & 729.5 \\
        EWLS expert & EWLS, $\gamma=0.966295$ & 723.0 & 1009.1 & 673.6 & 717.8 & 732.7 \\
        EWLS expert & EWLS, $\gamma=0.950000$ & 738.4 & 1037.9 & 695.5 & 737.4 & 751.4 \\
        External & TabPFN (online, +GRI) & 670.9 & 1165.5 & 788.4 & 749.3 & 758.6 \\
        Final agg. & MLpol on Base only & 690.5 & 2452.7 & 907.0 & 1013.0 & 1004.0 \\
        External & TabICL (online) & 670.3 & 2943.6 & 687.4 & 1038.6 & 1036.2 \\
        External & TabPFN (online) & 669.3 & 2943.3 & 719.7 & 1047.8 & 1043.1 \\
        EWLS expert & EWLS, $\gamma=1.000000$ & 736.7 & 3212.4 & 878.5 & 1170.2 & 1163.2 \\
        \midrule
        Base pool & GAM & 946.6 & 3201.3 & 1076.2 & 1314.2 & 1298.3 \\
        External & TabPFN (offline) & 677.7 & 4103.1 & 1055.8 & 1396.5 & 1381.1 \\
        Base pool & ResNet & 771.2 & 3582.5 & 1336.7 & 1401.4 & 1382.5 \\
        External & TabICL (offline) & 683.5 & 4453.8 & 1137.9 & 1500.2 & 1482.3 \\
        Base pool & XGBoost & 842.8 & 3972.8 & 1347.4 & 1510.7 & 1486.8 \\
        Base pool & LightGBM & 821.9 & 4410.9 & 1544.3 & 1656.9 & 1629.7 \\
        Base pool & FT-Transformer & 1060.4 & 4131.1 & 1763.3 & 1744.8 & 1727.3 \\
        Base pool & Ridge & 1999.0 & 3116.8 & 1918.7 & 2067.4 & 2078.5 \\
        Base pool & Lag-1 & 3918.6 & 2626.3 & 3286.4 & 3574.4 & 3633.2 \\
        \bottomrule
    \end{tabular}%
    }
\end{table*}

\subsection{Bootstrap uncertainty for the main comparison}
\label{app:rte_bootstrap}

Because MLpol is deterministic given its inputs and the base experts are
fit once on a fixed training window, Table~\ref{tab:main_results} contains
no seed-level variability to report. The uncertainty that remains is
\emph{sampling variability} of the finite test trajectory itself: had the
test window been a different but comparable trajectory of comparable
length, how would the headline numbers move? We quantify this via a
paired moving block bootstrap on the per-step squared-loss sequences.

\paragraph{Protocol.}
For each method we treat the sequence of daily squared errors
$\{(\hat{y}_t - y_t)^2\}_{t=1}^T$ as the basic observation and resample
with a moving block bootstrap \citep{kunsch1989jackknife}. Within each
regime (overall, pre-lockdown, lockdown, post-lockdown) we draw
$\lceil T_{\mathrm{regime}}/B \rceil$ block start positions uniformly
from the valid start set, concatenate the resulting length-$B$ blocks,
and truncate to the regime length. The replicate RMSE is computed from
the resampled sequence. We use $B=14$ days on overall, pre-lockdown, and
post-lockdown (two weekly cycles), and $B=7$ on the 56-day lockdown
regime so that the set of valid block starts remains non-trivial. All
methods share the same block-start draws within a replicate
(\emph{paired} bootstrap), so the distribution of pairwise RMSE
differences is well-defined. Results below use $R=10{,}000$ replicates
and seed $0$; Monte Carlo variation in the reported CIs is under $1$\,MW
(verified against independent runs at $R=1{,}000$).

The bootstrap consumes the \emph{same} per-step squared-loss sequences
that produce Table~\ref{tab:main_results}; the point estimates therefore
match the main results exactly.

Given a method $m$ and anchor method $m_0$ (our full Base+EWLS
aggregate), we report:
\begin{itemize}[leftmargin=1.4em,itemsep=1pt,topsep=2pt]
    \item the point estimate $\mathrm{RMSE}_m$ (identical to
      Table~\ref{tab:main_results}) and the $2.5$--$97.5\%$ percentile
      bootstrap CI, as a marginal uncertainty on $m$'s performance;
    \item the paired difference $\Delta_{m}=\mathrm{RMSE}_m -
      \mathrm{RMSE}_{m_0}$, computed on the real data, and the
      $2.5$--$97.5\%$ percentile CI of the \emph{bootstrap distribution}
      of the difference. A CI excluding zero corresponds to a
      significant difference at the $5\%$ level under this resampling
      scheme.
\end{itemize}

\paragraph{Per-method RMSE with 95\% bootstrap CIs.}
Table~\ref{tab:boot_per_method} reports, for each method and regime, the
point estimate and the $95\%$ bootstrap CI. The CI widths scale as
expected with regime length (narrow on the 441-day pre-lockdown, very
wide on the 56-day lockdown) and with RMSE level (wider for the
offline baselines that are dominated by a few catastrophic lockdown
days).

\begin{table}[ht]
    \centering
    \small
    \setlength{\tabcolsep}{4pt}
    \renewcommand{\arraystretch}{0.95}
    \caption{Per-method RMSE (MW) with $95\%$ moving-block bootstrap CIs
    ($R = 10{,}000$, $B = 14$ for overall/pre/post, $B = 7$ for lockdown).
    Point estimates match Table~\ref{tab:main_results} exactly. CIs are
    rounded to the nearest MW for readability. The same bootstrap resamples are
    shared across methods and are used for the paired differences in
    Table~\ref{tab:boot_diffs}.}
    \label{tab:boot_per_method}
    \resizebox{\linewidth}{!}{%
    \begin{tabular}{
        l
        S[table-format=4.1] l
        S[table-format=4.1] l
        S[table-format=4.1] l
        S[table-format=4.1] l
    }
        \toprule
        & \multicolumn{2}{c}{Pre-lockdown}
        & \multicolumn{2}{c}{Lockdown}
        & \multicolumn{2}{c}{Post-lockdown}
        & \multicolumn{2}{c}{Overall} \\
        \cmidrule(lr){2-3}
        \cmidrule(lr){4-5}
        \cmidrule(lr){6-7}
        \cmidrule(lr){8-9}
        Method
        & {RMSE} & {95\% CI}
        & {RMSE} & {95\% CI}
        & {RMSE} & {95\% CI}
        & {RMSE} & {95\% CI} \\
        \midrule
        \textbf{MLpol on Base+EWLS (ours)}
          & 623.1 & {[532, 680]}
          & 1086.1 & {[600, 1323]}
          & 579.3 & {[466, 661]}
          & 655.8 & {[565, 745]} \\
        MLpol on EWLS only
          & 653.9 & {[546, 704]}
          & 1073.6 & {[673, 1315]}
          & 584.1 & {[473, 666]}
          & 673.2 & {[576, 749]} \\
        MLpol on Base only
          & 690.5 & {[595, 752]}
          & 2452.7 & {[1654, 2984]}
          & 907.0 & {[740, 1026]}
          & 1004.0 & {[780, 1245]} \\
        Best EWLS ($\gamma\!\approx\!0.9986$, hindsight)
          & 657.9 & {[552, 711]}
          & 1168.3 & {[867, 1367]}
          & 585.4 & {[475, 665]}
          & 687.5 & {[591, 761]} \\
        TabICL (online, +GRI)
          & 668.0 & {[564, 718]}
          & 1132.6 & {[579, 1425]}
          & 699.5 & {[544, 854]}
          & 723.4 & {[618, 816]} \\
        TabPFN (online, +GRI)
          & 670.9 & {[568, 723]}
          & 1165.5 & {[522, 1452]}
          & 788.4 & {[637, 936]}
          & 758.6 & {[644, 865]} \\
        TabICL (online)
          & 670.3 & {[566, 715]}
          & 2943.6 & {[1914, 3742]}
          & 687.4 & {[556, 790]}
          & 1036.2 & {[681, 1390]} \\
        TabPFN (online)
          & 669.3 & {[566, 711]}
          & 2943.3 & {[1854, 3714]}
          & 719.7 & {[559, 828]}
          & 1043.1 & {[683, 1388]} \\
        TabICL (offline)
          & 683.5 & {[581, 727]}
          & 4453.8 & {[3847, 5176]}
          & 1137.9 & {[810, 1333]}
          & 1482.3 & {[972, 1954]} \\
        TabPFN (offline)
          & 677.7 & {[575, 721]}
          & 4103.1 & {[3244, 4953]}
          & 1055.8 & {[744, 1213]}
          & 1381.1 & {[899, 1842]} \\
        \bottomrule
    \end{tabular}%
    }
\end{table}

\paragraph{Pairwise differences vs.\ our full method.}
Table~\ref{tab:boot_diffs} is the more informative summary: each row
reports $\Delta = \mathrm{RMSE}_m - \mathrm{RMSE}_{\mathrm{ours}}$ on the
real data, and the $95\%$ CI of its bootstrap distribution. A checkmark
in the \textit{sig.}\ column indicates that the CI excludes zero. The
paired structure concentrates the uncertainty on the \emph{difference},
which is considerably narrower than the individual marginal CIs in
Table~\ref{tab:boot_per_method} because the methods share common noise
on each day.

\begin{table}[ht]
    \centering
    \small
    \setlength{\tabcolsep}{3.2pt}
    \renewcommand{\arraystretch}{0.96}
    \caption{Paired $\Delta\mathrm{RMSE}=\mathrm{RMSE}_m-
    \mathrm{RMSE}_{\mathrm{ours}}$ (MW) with $95\%$ moving-block bootstrap CIs.
    Positive $\Delta$ favours our full Base+EWLS aggregate. $\checkmark$ marks
    CIs excluding zero. $R=10{,}000$, with block lengths as in
    Table~\ref{tab:boot_per_method}.}
    \label{tab:boot_diffs}

    \vspace{0.25em}
    \textbf{Panel A: Pre-lockdown and lockdown.}

    \resizebox{\linewidth}{!}{%
    \begin{tabular}{
        l
        S[table-format=+4.1]@{\;}c@{}S[table-format=-4.1]@{,\,}S[table-format=4.1]@{}c@{\;}c
        S[table-format=+4.1]@{\;}c@{}S[table-format=-4.1]@{,\,}S[table-format=4.1]@{}c@{\;}c
    }
        \toprule
        & \multicolumn{6}{c}{Pre-lockdown}
        & \multicolumn{6}{c}{Lockdown} \\
        \cmidrule(lr){2-7}
        \cmidrule(lr){8-13}
        Comparator $m$
        & {$\Delta$} & \multicolumn{4}{c}{95\% CI} & {sig.}
        & {$\Delta$} & \multicolumn{4}{c}{95\% CI} & {sig.} \\
        \midrule
        MLpol on EWLS only
         &  +30.8 & {[} &   -3.5 &   50.4 & {]} &
         &  -12.5 & {[} &  -30.5 &  112.2 & {]} & \\
        MLpol on Base only
         &  +67.4 & {[} &   27.5 &  112.1 & {]} & $\checkmark$
         & +1366.6 & {[} &  993.3 & 1772.5 & {]} & $\checkmark$ \\
        Best single EWLS ($\gamma\!\approx\!0.9986$, hindsight)
         &  +34.8 & {[} &    1.5 &   55.3 & {]} & $\checkmark$
         &  +82.2 & {[} &   16.5 &  316.7 & {]} & $\checkmark$ \\
        TabICL (online, +GRI)
         &  +44.9 & {[} &  -31.0 &   95.0 & {]} &
         &  +46.5 & {[} & -154.0 &  268.4 & {]} & \\
        TabPFN (online, +GRI)
         &  +47.8 & {[} &  -21.7 &   92.6 & {]} &
         &  +79.4 & {[} & -225.1 &  216.1 & {]} & \\
        TabICL (online)
         &  +47.2 & {[} &  -31.2 &   95.1 & {]} &
         & +1857.5 & {[} & 1237.3 & 2545.8 & {]} & $\checkmark$ \\
        TabPFN (online)
         &  +46.2 & {[} &  -28.3 &   87.4 & {]} &
         & +1857.2 & {[} & 1191.0 & 2485.5 & {]} & $\checkmark$ \\
        TabICL (offline)
         &  +60.4 & {[} &  -18.4 &  108.5 & {]} &
         & +3367.7 & {[} & 3064.1 & 4042.8 & {]} & $\checkmark$ \\
        TabPFN (offline)
         &  +54.6 & {[} &  -19.3 &   96.6 & {]} &
         & +3017.1 & {[} & 2543.0 & 3774.0 & {]} & $\checkmark$ \\
        \bottomrule
    \end{tabular}%
    }

    \vspace{0.75em}
    \textbf{Panel B: Post-lockdown and overall.}

    \resizebox{\linewidth}{!}{%
    \begin{tabular}{
        l
        S[table-format=+4.1]@{\;}c@{}S[table-format=-4.1]@{,\,}S[table-format=4.1]@{}c@{\;}c
        S[table-format=+4.1]@{\;}c@{}S[table-format=-4.1]@{,\,}S[table-format=4.1]@{}c@{\;}c
    }
        \toprule
        & \multicolumn{6}{c}{Post-lockdown}
        & \multicolumn{6}{c}{Overall} \\
        \cmidrule(lr){2-7}
        \cmidrule(lr){8-13}
        Comparator $m$
        & {$\Delta$} & \multicolumn{4}{c}{95\% CI} & {sig.}
        & {$\Delta$} & \multicolumn{4}{c}{95\% CI} & {sig.} \\
        \midrule
        MLpol on EWLS only
         &   +4.8 & {[} &   -4.8 &   13.8 & {]} &
         &  +17.4 & {[} &   -7.0 &   32.9 & {]} & \\
        MLpol on Base only
         & +327.8 & {[} &  197.2 &  442.1 & {]} & $\checkmark$
         & +348.3 & {[} &  185.9 &  527.6 & {]} & $\checkmark$ \\
        Best single EWLS ($\gamma\!\approx\!0.9986$, hindsight)
         &   +6.2 & {[} &  -14.1 &   25.5 & {]} &
         &  +31.7 & {[} &   -2.5 &   56.7 & {]} & \\
        TabICL (online, +GRI)
         & +120.2 & {[} &    7.7 &  269.6 & {]} & $\checkmark$
         &  +67.7 & {[} &    9.4 &  125.3 & {]} & $\checkmark$ \\
        TabPFN (online, +GRI)
         & +209.1 & {[} &  118.3 &  340.5 & {]} & $\checkmark$
         & +102.8 & {[} &   39.7 &  160.9 & {]} & $\checkmark$ \\
        TabICL (online)
         & +108.2 & {[} &   25.7 &  208.3 & {]} & $\checkmark$
         & +380.4 & {[} &   86.4 &  675.2 & {]} & $\checkmark$ \\
        TabPFN (online)
         & +140.4 & {[} &   43.7 &  219.3 & {]} & $\checkmark$
         & +387.3 & {[} &   90.8 &  667.1 & {]} & $\checkmark$ \\
        TabICL (offline)
         & +558.6 & {[} &  250.1 &  774.5 & {]} & $\checkmark$
         & +826.6 & {[} &  363.5 & 1254.6 & {]} & $\checkmark$ \\
        TabPFN (offline)
         & +476.6 & {[} &  198.0 &  649.6 & {]} & $\checkmark$
         & +725.3 & {[} &  288.5 & 1137.8 & {]} & $\checkmark$ \\
        \bottomrule
    \end{tabular}%
    }
\end{table}

\paragraph{What is and is not statistically established.}
We summarise the findings of Table~\ref{tab:boot_diffs}; all inferences
are at the $5\%$ level under the paired moving block bootstrap.

\begin{enumerate}[leftmargin=1.6em,itemsep=2pt,topsep=2pt]
    
    \item \textbf{The EWLS correction layer adds significant value in every
    regime.} Against MLpol on Base only, Ours wins by $+67.4$\,MW
    on pre-lockdown, $+1366.6$\,MW on lockdown, $+327.8$\,MW on
    post-lockdown, and $+348.3$\,MW overall, with every CI strictly
    positive. This is the cleanest empirical support of
    Section~\ref{sec:theory} and the statistical basis for the
    headline finding in Section~\ref{sec:main_result}.

    \item \textbf{Base+EWLS beats the hindsight-optimal
    single-\boldmath$\gamma$ EWLS expert in pre-lockdown and lockdown;
    the overall improvement is positive but marginally inside bootstrap
    noise.} Against the best single EWLS expert chosen in hindsight on
    overall RMSE ($\gamma\!\approx\!0.9986$, overall RMSE $687.5$), our
    full pool wins by $+34.8$\,MW on pre-lockdown (CI $[1.5, 55.3]$,
    $\checkmark$) and $+82.2$\,MW on lockdown (CI $[16.5, 316.7]$,
    $\checkmark$); the overall improvement is $+31.7$\,MW (CI
    $[-2.5, 56.7]$, marginally inside bootstrap noise) and the
    post-lockdown difference ($+6.2$\,MW, CI $[-14.1, 25.5]$) is
    essentially zero. 
    This is the expected shape: no single memory scale dominates the whole
    trajectory. The overall-selected single $\gamma$ remains locally competitive
    post-lockdown, but pays in regimes where the optimal effective memory differs.

    Per-regime hindsight optima make this point quantitatively: the
    best $\gamma$ is $\approx\!0.9990$ on pre-lockdown,
    $\approx\!0.977$ on lockdown, and $\approx\!0.9997$ on
    post-lockdown---an order-of-magnitude shift in optimal effective
    memory across the three regimes. In pre-lockdown, the full
    Base+EWLS aggregate ($623.1$) in fact beats every single EWLS
    expert in hindsight, including the per-regime oracle, because the
    base-expert channel carries predictive structure orthogonal to any
    memory-length choice.

    \item \textbf{Base experts contribute a small orthogonal signal,
    most visibly in pre-lockdown.} Adding base experts to MLpol on
    EWLS-only yields a positive pre-lockdown gain
    ($\Delta = +30.8$\,MW, CI $[-3.5, 50.4]$, marginally inside
    bootstrap noise); the overall gap ($+17.4$\,MW, CI $[-7.0, 32.9]$)
    and the post-lockdown gap ($+4.8$\,MW, CI $[-4.8, 13.8]$) are
    favourable but within bootstrap noise, while the lockdown
    comparison marginally favours EWLS-only ($-12.5$\,MW, CI
    $[-30.5, 112.2]$, not significant). The picture is coherent with
    the design: stable historical structure---which base experts
    encode and EWLS cannot recover from a fixed memory grid
    alone---contributes most where the base models' training
    distribution applies, i.e.\ pre-lockdown. Under the structural
    break, EWLS-only and Base+EWLS are statistically
    indistinguishable, confirming that the lockdown gain in
    Table~\ref{tab:main_results} is delivered by the EWLS layer rather
    than by the base pool.

    \item \textbf{Significantly beats TabICL+GRI overall and post-lockdown,
    despite using no external regime-response signal.} Against TabICL
    (online, +GRI)---the strongest competitor, supplied with a 
    regime-informed covariate---Ours wins significantly on
    overall RMSE ($\Delta = +67.7$\,MW, CI $[9.4, 125.3]$,
    $\checkmark$, a $9.4\%$ relative improvement) and in post-lockdown
    ($\Delta = +120.2$\,MW, CI $[7.7, 269.6]$, $\checkmark$, a
    $17.2\%$ relative improvement). 
    The pre-lockdown and lockdown margins are within bootstrap noise. This is
    consistent with the role of the GRI covariate: before the lockdown shock it
    provides little regime-shift information, while during lockdown it gives
    TabICL+GRI an immediate external response signal. The recovery regime is where
    this coincident signal is less directly informative and our method's
    memory-scale tracking takes over.
    
    \item \textbf{Significantly beats every no-GRI online adaptive baseline in
    the regimes where adaptation is stress-tested.}
    Against TabICL/TabPFN online (without GRI),
    every difference is significant in lockdown, post-lockdown, and
    overall; pre-lockdown is the lone null cell, where no method
    actually exercises adaptation and all reasonable predictors lie
    within $\sim\!50$\,MW of each other.
\end{enumerate}

\paragraph{Sensitivity to block length.}
We verified the robustness of the overall-regime CIs to block length by
repeating the bootstrap at $B \in \{7, 14, 28\}$. Point estimates are
identical across $B$ by construction. The width of each CI increases
with $B$ (as fewer independent blocks fit inside the trajectory), but
the qualitative \textit{sig.}\ column of Table~\ref{tab:boot_diffs} is
stable---no entry flips between $\checkmark$ and blank over this range.

\paragraph{Implementation.}
The bootstrap procedure is implemented as
\texttt{scripts/bootstrap\_ci.py} in the released code base and runs in
under $30$\,s for $R=10{,}000$ on a laptop; it consumes only the saved
aggregation outputs and requires no pipeline re-execution.

% -----------------------------------------------------------------------------
\section{Adaptive-filter and state-space baselines}
\label{app:adaptive_kf}
% -----------------------------------------------------------------------------

This section clarifies the relationship between the EWLS/RLS update used in
Algorithm~\ref{alg:main} and Kalman filtering. The purpose is interpretive:
our method does not attempt to estimate a latent process-noise covariance.
Instead, each forgetting factor defines a fixed time-scale objective, and the
outer MLpol layer selects among these objectives according to realised
forecasting loss.

Consider first the ideal EWLS recursion without the inflation term. We write
\(\bw_{t|t-1}\) and \(\bP_{t|t-1}\) for the state mean and covariance used before
observing \(y_t\), and \(\bw_{t|t}\) and \(\bP_{t|t}\) for the corresponding
quantities after observing \(y_t\). In the forgotten-RLS interpretation,
\[
    \bP_{t|t-1}=\gamma^{-1} \bP_{t-1|t-1},\qquad
    \bw_{t|t-1}=\bw_{t-1|t-1}.
\]
With unit observation-noise scale, the measurement update is
\[
    S_t = 1+\tilde \bz_t^\top \bP_{t|t-1}\tilde \bz_t,\qquad
    \bm K_t = \bP_{t|t-1}\tilde \bz_t  S_t^{-1},
\]
\[
    \bw_{t|t}
    =
    \bw_{t|t-1}
    +
    \bm K_t\bigl(y_t-\tilde \bz_t^\top \bw_{t|t-1}\bigr),
\]
\[
    \bP_{t|t}
    =
    \bP_{t|t-1}
    -
    \bm K_t\tilde \bz_t^\top \bP_{t|t-1}.
\]
Equivalently, since \(\bP_{t|t-1}=\gamma^{-1} \bP_{t-1|t-1}\), this can be written
in the standard RLS form
\[
    s_t=\gamma+\tilde \bz_t^\top \bP_{t-1|t-1}\tilde \bz_t,\qquad
    \bm K_t=\frac{\bP_{t-1|t-1} \tilde \bz_t}{s_t},
\]
\[
    \bw_{t|t}
    =
    \bw_{t-1|t-1}
    +
    \bm K_t\bigl(y_t-\tilde \bz_t^\top \bw_{t-1|t-1}\bigr),
\]
\[
    \bP_{t|t}
    =
    \gamma^{-1}
    \left(
    \bP_{t-1|t-1}
    -
    \frac{
    \bP_{t-1|t-1}\tilde \bz_t\tilde \bz_t^\top \bP_{t-1|t-1}
    }{s_t}
    \right).
\]
The through-\(t\) coefficient \(\bw_{t|t}\) is the minimiser of
\[
\sum_{s=1}^{t}\gamma^{t-s}
(y_s - \bw^\top \tilde \bz_s)^2
+
\gamma^t\delta_0\|\bw\|_2^2.
\]
Thus the coefficient used to predict at round \(t\) is \(\bw_{t|t-1}\), while
\(\bw_{t|t}\) is available only after \(y_t\) has been observed.

This recursion can be read as a Kalman filter for the linear observation model
\[
y_t=\tilde \bz_t^\top \bw_t + \varepsilon_t,
\]
with random-walk state equation
\[
\bw_t = \bw_{t-1} + \bm \eta_t.
\]
In this interpretation, forgetting replaces the previous posterior covariance
\(\bP_{t-1|t-1}\) by the prior covariance
\[
    \bP_{t|t-1}=\gamma^{-1} \bP_{t-1|t-1}.
\]
Hence the implicit process-noise covariance is
\[
\bQ_t^{(\gamma)}
=
\bP_{t|t-1} - \bP_{t-1|t-1}
=
(\gamma^{-1}-1) \bP_{t-1|t-1}.
\]

Thus smaller $\gamma$ corresponds to a larger prior covariance inflation and
therefore to a filter that is willing to move its coefficients more rapidly.
Conversely, $\gamma$ close to one imposes a long-memory, slowly moving
coefficient path.

The implementation adds a small isotropic inflation term after the measurement
update,
\[
\bP_t = \bP_{t|t}+\varepsilon_k\bI .
\]
Under the Kalman interpretation, this can be viewed as carrying a small
additional process-noise component into the next round. Its role in our
experiments is numerical stabilisation, especially for fast-forgetting experts
whose effective design matrices may become poorly conditioned.

This interpretation is different from adaptive Kalman filtering. Methods such
as adaptive Kalman filters try to infer or update a process-noise covariance
from innovations, likelihoods, or variational criteria. Our construction does
not estimate $\bQ_t$. Instead, it defines a finite family of fixed EWLS
objectives indexed by $\gamma$, equivalently a finite family of implicit
time-scale assumptions. The choice among these assumptions is deferred to the
outer MLpol aggregation layer, which is driven directly by realised squared
forecasting loss. In this sense, the forgetting factors are not latent
parameters to be estimated inside a single state-space model; they are
competing online objectives exposed to expert aggregation.

% -----------------------------------------------------------------------------
\subsection{RLS, Kalman filtering, and time-scale modelling}
\label{app:rls-kalman}
% -----------------------------------------------------------------------------

A natural question is whether the separation of roles in our framework---a
fixed $\gamma$-grid of EWLS time scales at the expert layer, followed by
loss-driven MLpol aggregation---could instead be replaced by a
Kalman-filter-style adaptive layer that updates process-noise uncertainty
online. We therefore compare against three representative adaptive-filter
alternatives on the French electricity benchmark. The goal of this appendix is
not to rule out all possible adaptive state-space designs, but to test whether
standard adaptive-\(Q\) mechanisms, tuned under the same validation protocol,
provide the same empirical behaviour as the \(\gamma\)-grid + MLpol
construction.

We benchmark \textit{$\gamma$-grid + MLpol} against three adaptive-KF
baselines that cover complementary modelling choices:

\begin{itemize}
  \item \textbf{VIKING} \citep{devilmarest2024viking}: a variational Kalman
    filter with isotropic process noise $\bQ_t = e^{b_t} \bI$. A Gaussian
    posterior over the log-variance $b_t$ is updated online, and a scalar parameter 
    $\rho_b$ controls the variability of this log-variance
    process.
  \item \textbf{VBAKF} \citep{huang2018vbakf}: a variational Bayesian adaptive
    Kalman filter. An inverse-Wishart prior is placed on the predicted state
    covariance and an inverse-gamma prior on the scalar
    observation noise; a forgetting factor $\rho \in (0, 1]$ controls 
    how quickly the covariance posterior discounts past information.
  \item \textbf{IMM-KF} \citep{blom1988imm}: an interacting multiple-models filter.
    It maintains several Kalman filters on a fixed grid of $\bQ_i = q_i \bI$,
    mixes their state through a Markov transition matrix, and updates mode probabilities
    by Gaussian innovation likelihoods. IMM is the closest state-space analogue of a
    multi-scale construction, but its combination rule is Bayesian Markov mixing rather than regret minimisation over predictions exposed as separate experts.
\end{itemize}

These baselines vary along two axes: the parameterisation of process 
uncertainty (isotropic scalar, full covariance, or discrete grid) and the
adaptation rule used to update and combine that uncertainty. 
This makes them a useful stress test for whether the explicit expert-layer /
aggregation-layer separation is empirically useful on this benchmark.

\paragraph{Setup and tuning protocol.}
For a like-for-like comparison, each adaptive-KF baseline is paired with the
same \(M=7\) direct base experts and combined by the same outer MLpol
aggregator as our main method. Only the adaptive expert-layer component varies.
For VIKING and VBAKF, the responsiveness axis is swept as a small pool for MLpol
to aggregate over. IMM-KF already contains an internal grid of \(Q\) scales, so
we use a single IMM instance with that grid:

\begin{itemize}[leftmargin=*]
  \item \textbf{VIKING pool}: three instances at
    $\rho_b \in \{10^{-4}, 10^{-3}, 10^{-2}\}$, spanning
    conservative $\to$ responsive log-variance random-walk rates.
  \item \textbf{VBAKF pool}: three instances at
    $\rho \in \{0.95, 0.98, 0.995\}$, spanning
    responsive $\to$ conservative IW-concentration forgetting.
  \item \textbf{IMM}: a single instance with a five-mode grid
    $q \in \{10^{-8}, 10^{-6}, 10^{-4}, 10^{-2}, 10^0\}$. The IMM
    architecture already integrates a grid of $Q$ scales internally, so a
    pool of multiple IMMs would be redundant.
\end{itemize}
All remaining nuisance hyperparameters (VIKING's $\hat{b}_0$,
$\Sigma_0$, VB iteration count, and initial covariance scale; VBAKF's
$\tau_P$, $\tau_R$, VB iterations, and $\delta$; IMM's $\pi_{\mathrm{stay}}$
and $\delta$) are tuned on the 2018 out-of-sample window via the same
walk-forward protocol used for our method's $\varepsilon_0$ selection
(Appendix~\ref{app:eps_selection}). In this protocol, base models are fit on 2012--2016, the
aggregator runs online over 2017--2018, and the 2018 RMSE of the corresponding MLpol 
pipeline is minimised over a grid of nuisance-parameter
combinations. The real test period (2019-01-01 onward) is never used
for selection. Thus all adaptive-filter baselines are tuned under the same information
constraints as our method. Observation-noise priors are initialised from the training residual variance of the uniform base-expert mean and are not tuned. Selected hyperparameters are listed in Table~\ref{tab:kf_tuned}.

\begin{table}[t]
  \centering
  \small
  \caption{Nuisance-parameter search grids and selected values for the
    three adaptive-KF baselines. All values selected by minimising the
    2018-out-of-sample RMSE of MLpol on Base$+$pool, using the
    walk-forward protocol of Appendix~\ref{app:eps_selection}. The
    responsiveness-axis grids (VIKING $\rho_b$, VBAKF $\rho$, and IMM's
    internal $q$-grid) are listed in the Setup of
    \S\ref{app:adaptive_kf} and are \emph{not} tuned: they play the role
    of the pool that MLpol aggregates over. Observation-noise priors
    ($\hat{a}_0$ for VIKING, $R_0$ for VBAKF/IMM) are initialised from
    train-residual variance of the uniform base-expert mean and are also
    not tuned. Selected values at grid boundaries were confirmed by
    boundary-extension sweeps.}
  \label{tab:kf_tuned}
  \begin{tabular}{@{}llll@{}}
    \toprule
    Parameter & Role & Search grid & Selected \\
    \midrule
    \multicolumn{4}{@{}l}{\textbf{VIKING}\hspace{0.6em}\small
      (fixed: \texttt{learn\_sigma} $=$ \texttt{learn\_Q} $=$ True,
      $\rho_a = 0$)} \\
    $\hat{b}_0$       & prior mean of $\log Q$         & $\{-14, -12, -10, -8, -6, -4, -2, 0\}$         & $-8$      \\
    $\Sigma_0$        & prior variance of $\log Q$     & $\{10^{-5}, 10^{-4}, 10^{-3}, 10^{-2}, 10^{-1}, 1\}$ & $10^{-1}$ \\
    $N_\mathrm{VB}$   & Newton steps per observation   & $\{1, 2, 3, 5, 8\}$                            & $8$        \\
    $\delta$          & initial covariance $\bP_0 = \delta \bI$ & $\{10^{-3}, 10^{-2}, 10^{-1}, 1, 10, 10^{2}\}$ & $10^{-1}$  \\
    \midrule
    \multicolumn{4}{@{}l}{\textbf{VBAKF}} \\
    $\tau_P$          & prior IW dof on $P$            & $\{3, 4, 5, 7, 10, 15, 25, 50\}$               & $3$        \\
    $\tau_R$          & prior IG dof on $R$            & $\{3, 4, 5, 7, 10, 15, 25, 50\}$               & $50$        \\
    $N_\mathrm{VB}$   & variational iterations/step    & $\{3, 5\}$                                     & $3$        \\
    $\delta$          & initial $\bP_0 = \delta \bI$       & $\{10^{-2}, 10^{-1}, 1, 10, 10^{2}, 10^{3}, 10^{4}\}$ & $10^{-2}$ \\
    \midrule
    \multicolumn{4}{@{}l}{\textbf{IMM-KF}} \\
    $\pi_\mathrm{stay}$ & Markov self-transition prob. & $\{0.30, 0.50, 0.60, 0.70, 0.80, 0.90, $ &       \\

                        &                              & \,\, $  0.95, 0.98, 0.99, 0.995, 0.999\}$      & $0.5$ \\
    $\delta$          & initial $\bP_0 = \delta \bI$ per mode & $\{10^{-3}, 10^{-2}, 10^{-1}, 1, 10, 10^{2}\}$ & $10^{-2}$ \\
    \bottomrule
  \end{tabular}
\end{table}

\paragraph{Main comparison.}

Table~\ref{tab:adaptive_kf_comparison} reports RMSE by regime. Under the
matched walk-forward tuning protocol, MELO has the lowest overall RMSE among
all compared variants. Relative to the adaptive-KF pools that are also exposed
to the same direct base experts and the same outer MLpol aggregator, MELO
reduces overall RMSE by \(3.6\%\) against Base+VIKING, \(7.5\%\) against
Base+IMM, and \(9.3\%\) against Base+VBAKF. These overall gains are also
confirmed by the paired moving-block bootstrap reported below.

The regime decomposition gives a more nuanced picture. MELO is the best
point estimate before lockdown and after lockdown, and remains close to the
best adaptive variants during the short lockdown window. In that window,
EWLS-only and VBAKF-only have slightly lower point estimates, while IMM-only is
essentially tied with MELO. Thus the advantage of MELO in this comparison is
not that it uniformly dominates every adaptive filter in every regime, but that
it gives the best overall balance across the full 2019--2021 online test
period.

The filter-only rows provide a diagnostic on how the adaptive-filter
outputs interact with the direct base experts. Adding direct experts improves
overall RMSE for VIKING, VBAKF, and our EWLS pool, while it worsens the overall
RMSE of IMM. The regime decomposition shows that this interaction is not
uniform: during lockdown, adding direct experts helps VIKING, slightly hurts
the EWLS pool, and hurts IMM and VBAKF more substantially. This suggests that
the benefit of exposing direct base forecasts to the outer aggregator depends
on whether the adaptive component remains complementary to the base pool during
the break. We return to this point in the IMM diagnostic below.

\begin{table}[t]
  \centering
  \small
  \caption{Per-regime and overall RMSE (MW) of MELO against three adaptive
    Kalman-filter baselines on the 2019-01-01 to 2021-01-15 test period. All
    nuisance hyperparameters of each baseline were selected on 2018
    out-of-sample data with the same walk-forward protocol used for our own
    $\varepsilon_0$ selection (Appendix~\ref{app:eps_selection}); the real test
    period was never touched during selection. filter-only rows drop direct
    base experts from the outer MLpol pool. \textbf{Bold}: best per column;
    \underline{underline}: second best.}
  \label{tab:adaptive_kf_comparison}
  \begin{tabular}{lrrrr}
    \toprule
    Method & Pre-lockdown & Lockdown & Post-lockdown & Overall \\
    \midrule
    \multicolumn{5}{l}{\textit{Our method}} \\
    \quad MELO (MLpol on Base$+$EWLS) & \textbf{623.1} & 1086.1 & \textbf{579.3} & \textbf{655.8} \\
    \quad MLpol on EWLS only          & 653.9          & \underline{1073.6} & \underline{584.1} & \underline{673.2} \\
    \midrule
    \multicolumn{5}{l}{\textit{Adaptive KF baselines}} \\
    \quad MLpol on Base$+$VIKING      & \underline{634.0} & 1185.3 & 599.7 & 680.5 \\
    \quad MLpol on VIKING only        & 676.8             & 1293.0 & 595.0 & 717.3 \\
    \quad MLpol on Base$+$IMM         & 636.5             & 1276.7 & 648.6 & 708.7 \\
    \quad MLpol on IMM only           & 670.8             & 1086.5 & 605.0 & 690.5 \\
    \quad MLpol on Base$+$VBAKF       & 665.1             & 1227.4 & 665.4 & 722.8 \\
    \quad MLpol on VBAKF only         & 869.3             & \textbf{1066.9} & 633.6 & 816.2 \\
    \midrule
    \multicolumn{5}{l}{\textit{Reference}} \\
    \quad MLpol on Base only          & 690.5             & 2452.7 & 907.0 & 1004.0 \\
    \bottomrule
  \end{tabular}
\end{table}

\paragraph{Bootstrap significance.}
We extend the paired moving-block bootstrap of
Appendix~\ref{app:rte_bootstrap} to the three adaptive-KF baselines
with direct base experts exposed to the same outer MLpol aggregator. We use
\(R=10{,}000\) bootstrap replicates, with block length \(B=14\) for the
overall, pre-lockdown, and post-lockdown evaluations, and \(B=7\) for the
short lockdown regime. Table~\ref{tab:adaptive_kf_boot} reports the paired
\[
    \Delta\mathrm{RMSE}
    =
    \mathrm{RMSE}_m - \mathrm{RMSE}_{\mathrm{MELO}},
\]
where positive values favour MELO, together with the \(95\%\) percentile CI of
the bootstrap distribution.

\begin{table}[t]
  \centering
  \small
  \setlength{\tabcolsep}{4pt}
  \caption{Paired \(\Delta \mathrm{RMSE} = \mathrm{RMSE}_m -
    \mathrm{RMSE}_{\mathrm{MELO}}\) (MW) for the three adaptive-KF baselines
    with direct base experts included in the outer MLpol pool, with \(95\%\)
    paired moving-block bootstrap CIs. Positive \(\Delta\) favours MELO.
    \(\checkmark\) indicates that the CI excludes zero. The corresponding
    adaptive-filter-only ablations are discussed in the text.}
  \label{tab:adaptive_kf_boot}
  \resizebox{\linewidth}{!}{%
  \begin{tabular}{lrcrcrcrc}
    \toprule
    & \multicolumn{2}{c}{Pre-lockdown} & \multicolumn{2}{c}{Lockdown} & \multicolumn{2}{c}{Post-lockdown} & \multicolumn{2}{c}{Overall} \\
    \cmidrule(lr){2-3}\cmidrule(lr){4-5}\cmidrule(lr){6-7}\cmidrule(lr){8-9}
    Comparator \(m\)
      & \(\Delta\) / 95\% CI & sig.
      & \(\Delta\) / 95\% CI & sig.
      & \(\Delta\) / 95\% CI & sig.
      & \(\Delta\) / 95\% CI & sig. \\
    \midrule
    Base$+$VIKING
      & \(+10.9\) \([-2.2,\,25.6]\)     &
      & \(+99.2\) \([-144.0,\,180.8]\)  &
      & \(+20.5\) \([5.0,\,43.3]\)      & \checkmark
      & \(+24.7\) \([1.9,\,51.1]\)      & \checkmark \\
    Base$+$IMM-KF
      & \(+13.4\) \([-5.5,\,33.2]\)     &
      & \(+190.6\) \([-35.5,\,315.4]\)  &
      & \(+69.4\) \([16.9,\,108.8]\)    & \checkmark
      & \(+52.9\) \([24.8,\,86.5]\)     & \checkmark \\
    Base$+$VBAKF
      & \(+42.0\) \([12.0,\,76.3]\)     & \checkmark
      & \(+141.4\) \([-89.5,\,282.3]\)  &
      & \(+86.2\) \([25.3,\,119.6]\)    & \checkmark
      & \(+67.0\) \([32.0,\,113.6]\)    & \checkmark \\
    \bottomrule
  \end{tabular}
  }
\end{table}

The bootstrap results support the same overall conclusion as the point
estimates. Each of the three adaptive-KF baselines with direct experts is
significantly worse than MELO on overall RMSE: the paired difference is
\(+24.7\) MW for Base$+$VIKING (CI \([1.9,\,51.1]\)), \(+52.9\) MW for
Base$+$IMM-KF (CI \([24.8,\,86.5]\)), and \(+67.0\) MW for Base$+$VBAKF
(CI \([32.0,\,113.6]\)). The regime-level results are more mixed. MELO is
significantly better than all three baselines after lockdown and significantly
better than Base$+$VBAKF before lockdown, but the shorter 56-day lockdown
window has wide uncertainty intervals and none of the three pool-plus-direct
lockdown comparisons excludes zero. We therefore interpret the per-regime
diagnostics as qualitative evidence about behaviour, and the overall paired
bootstrap as the main significance statement.

The adaptive-filter-only ablations reach the same qualitative verdict on
overall RMSE for VIKING and VBAKF: VIKING-only and VBAKF-only are significantly
worse than MELO, with paired differences of \(+61.5\) MW
(CI \([18.9,\,92.6]\)) and \(+160.4\) MW (CI \([55.6,\,253.1]\)),
respectively. IMM-only is closer: the point estimate still favours MELO
(\(+34.7\) MW), but the overall CI crosses zero (CI \([-4.0,\,52.6]\)).
This is consistent with the diagnostic in \S\ref{app:adaptive_kf_imm}: IMM's
Bayesian mixture remains competitive when isolated, but appears less
complementary to the direct experts under the outer MLpol aggregator.

The diagnostics below examine why these adaptive-filter alternatives behave
differently from the fixed-scale EWLS grid on this benchmark. They should be
read as empirical interpretations of this comparison, rather than as general
negative results about adaptive Kalman filtering.
% -----------------------------------------------------------------------------
\subsection{IMM diagnostics}
\label{app:adaptive_kf_imm}
% -----------------------------------------------------------------------------

IMM is informative because its internal mode probabilities respond visibly to
the lockdown transition. Figure~\ref{fig:imm_mu_covid} plots the mode
probabilities \(\mu_{i,t}\) over the full test period and zooms in on the
lockdown onset.

\begin{figure}[t]
  \centering
  \includegraphics[width=\linewidth]{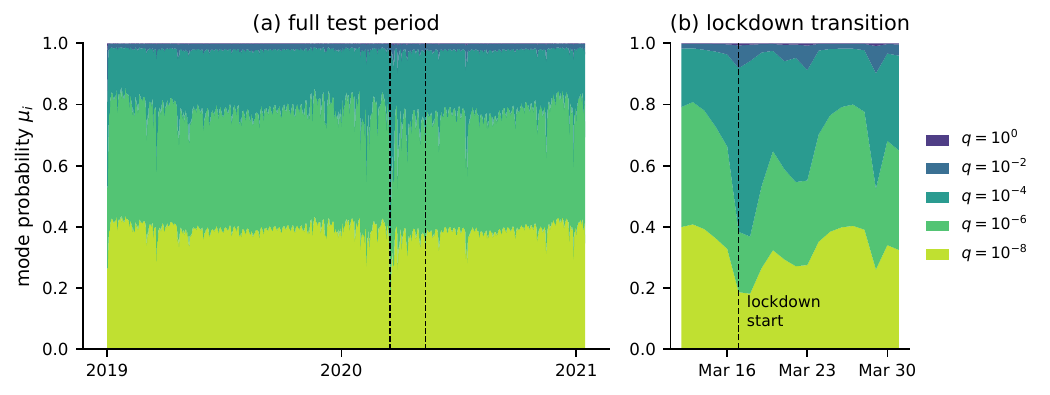}
  \caption{IMM-KF mode probabilities $\mu_{i,t}$ on a five-mode grid
    $q \in \{10^{-8}, \ldots, 10^0\}$. Colours encode modes by $Q$-scale,
    from slow to fast. \emph{(a)}~Full test period. Dashed vertical lines mark
    the COVID-19 lockdown window. \emph{(b)}~Zoom on the lockdown
    transition. With the tuned \(\pi_{\mathrm{stay}}=0.5\), IMM is relatively
    willing to change mode under the selected validation protocol. Around the
    lockdown onset, posterior mass shifts away from the two slowest modes and
    toward the intermediate mode \(q=10^{-4}\): from
    $\mu(\text{2020-03-14}) = [0.39, 0.39, 0.20, 0.02, 0.00]$ to
    $\mu(\text{2020-03-17}) = [0.19, 0.20, 0.53, 0.07, 0.01]$,
    ordered from \(q=10^{-8}\) to \(q=10^0\). The \(q=10^{-4}\) mode
    absorbs about \(53\%\) of the posterior mass on 2020-03-17 and peaks at
    \(57\%\) on 2020-03-18.}
  \label{fig:imm_mu_covid}
\end{figure}

\paragraph{Mode probabilities respond to the break.}
Outside the lockdown window, the IMM posterior places most of its mass on the
two slowest modes (\(q=10^{-8}\) and \(q=10^{-6}\)), consistent with relatively
stable pre- and post-lockdown behaviour at the daily frequency. Around the
lockdown onset, the posterior shifts toward the intermediate mode
\(q=10^{-4}\): its mass rises to \(53\%\) on 2020-03-17 and peaks at \(57\%\)
on 2020-03-18. Thus IMM does react to the change in the innovation sequence,
even though it receives no lockdown-aware covariate.

\paragraph{Detection does not fully translate into aggregate performance.}
Despite this visible mode shift, Base$+$IMM trails MELO by \(52.9\) MW
overall, a \(7.5\%\) reduction in RMSE relative to Base$+$IMM. The lockdown
period is more nuanced: IMM-only is essentially tied with MELO in point
estimate, whereas Base$+$IMM performs worse in that short window, with wide
bootstrap uncertainty. The underperformance of Base$+$IMM is therefore not
simply due to an absence of mode response at the break. Rather, the way IMM
compresses its adaptive behaviour into a single Bayesian mixture prediction
appears less favourable for the outer MLpol aggregator on this benchmark.

One interpretation is that IMM performs its time-scale selection internally:
it mixes states, updates mode probabilities, and then exposes only the resulting
Bayesian mixture forecast to the outer aggregator. Once posterior mass
concentrates around an intermediate operating point, the output behaves much
like a single adaptive filter at that effective scale. By contrast, the EWLS
construction exposes all \(K=16\) fixed-scale predictions separately to MLpol,
so the outer aggregator can place weight directly on short-memory corrections
during the break while downweighting both slow EWLS experts and raw base
forecasts that are no longer useful.

\paragraph{Internal mixing may reduce the diversity visible to MLpol.}
This distinction is apparent in the filter-only ablations. IMM-only is
competitive during lockdown (RMSE \(1086.5\)), but adding direct base experts
worsens lockdown RMSE to \(1276.7\). In MELO, adding the base experts to the
EWLS pool also slightly hurts lockdown RMSE, from \(1073.6\) to \(1086.1\),
but the deterioration is much smaller (\(+12.5\) MW), while the improvement
before and after lockdown is enough to give the best overall RMSE. This
suggests that IMM's internal Markov mixing leaves fewer distinct adaptive
signals for the outer MLpol layer to combine. The issue is not that IMM fails
to react to the break; rather, its reaction is largely compressed into a single
state-space mixture prediction before the regret aggregator observes it.
% -----------------------------------------------------------------------------
\subsection{Adaptive-\texorpdfstring{$Q$}{Q} pool diagnostics}
\label{app:adaptive_kf_collapse}
% -----------------------------------------------------------------------------

A second diagnostic concerns the pools of VIKING and VBAKF filters with
different responsiveness settings. Figure~\ref{fig:pool_collapse} reports the
average MLpol weight allocated to each member of these pools, broken down by
regime.

\begin{figure}[t]
  \centering
  \includegraphics[width=\linewidth]{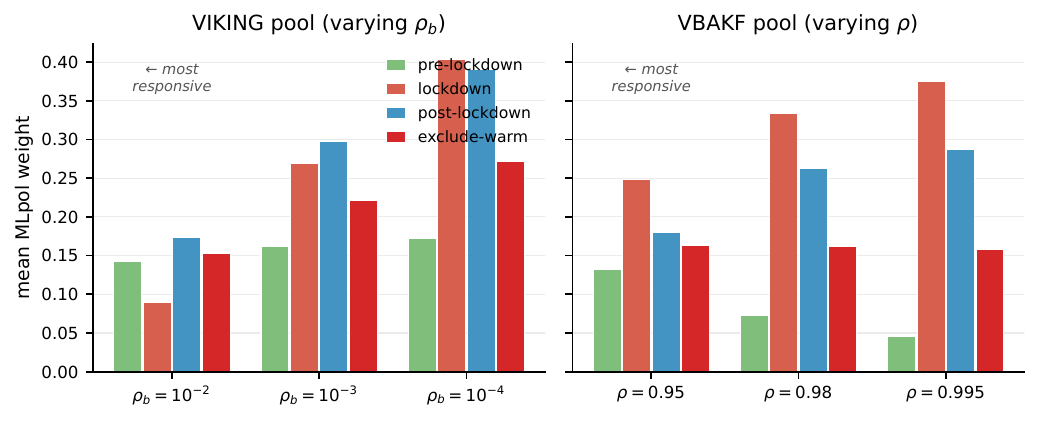}
  \caption{Mean MLpol weight allocated to each member of the VIKING pool
    (left) and the VBAKF pool (right), broken down by regime. For VIKING,
    smaller \(\rho_b\) is more conservative; for VBAKF, larger \(\rho\) is
    more conservative. Under lockdown, both pools shift weight away from the
    most responsive configuration toward more conservative ones, although the
    effect is sharper for VIKING than for VBAKF. The two pools use different
    internal adaptation mechanisms---an isotropic log-variance random walk for
    VIKING and a variational covariance/noise update for VBAKF---yet both
    exhibit this same qualitative reallocation.}
  \label{fig:pool_collapse}
\end{figure}

One might expect the most responsive configurations to receive more weight
under the lockdown break. Empirically, however, the outer MLpol layer does not
favour the most responsive setting. In the VIKING pool, the most conservative
configuration \(\rho_b=10^{-4}\) rises from \(17.2\%\) pre-lockdown to
\(40.3\%\) during lockdown, while the most responsive \(\rho_b=10^{-2}\)
receives only \(9.0\%\) in lockdown. The intermediate setting
\(\rho_b=10^{-3}\) also gains weight, reaching \(27.0\%\). Thus the lockdown
allocation within VIKING is concentrated toward the more conservative end of
the pool.

The VBAKF pool shows the same tendency, though less sharply. Before lockdown,
the most responsive configuration \(\rho=0.95\) receives the largest weight
among the three VBAKF members (\(13.2\%\), versus \(7.2\%\) for \(\rho=0.98\)
and \(4.6\%\) for \(\rho=0.995\)). During lockdown, however, the weight shifts
toward the slower settings: \(\rho=0.98\) and \(\rho=0.995\) rise to
\(33.4\%\) and \(37.5\%\), respectively, while the more responsive
\(\rho=0.95\) remains non-negligible at \(24.8\%\) but is no longer dominant.
At the same time, the direct-expert share collapses from \(74.9\%\) before
lockdown to only \(4.3\%\) during lockdown, so the outer aggregator is indeed
relying almost entirely on the VBAKF pool in that regime---just not primarily
on its most responsive member.

This pattern should not be read as a failure of MLpol. Rather, it suggests that
the more aggressive adaptive-\(Q\) configurations did not consistently deliver
lower realised squared loss on this trajectory. A plausible explanation is a
bias--variance trade-off: aggressive process-noise adaptation can reduce bias
after a break, but it can also increase prediction variability. Since MLpol is
driven by realised squared loss,
\[
  r_{t,j}
  =
  2(\hat{y}_t-y_t)(\hat{y}_t-\tilde y_{t,j}),
\]
it downweights experts whose additional responsiveness does not translate into
lower loss. The fact that this qualitative pattern appears for both VIKING and
VBAKF suggests that, on this benchmark, filter-internal adaptive-\(Q\)
objectives need not align with the outer squared-loss objective.

The contrast with the EWLS grid is that each EWLS expert has a fixed update
scale. The outer aggregator therefore compares predictions generated by stable,
pre-specified adaptation rules, rather than predictions whose internal
process-noise scale is itself being re-estimated online. In our main method,
the lockdown period assigns \(87.9\%\) of the total MLpol weight to the EWLS
pool, with \(43.9\%\) on fast experts (\(\gamma<0.99\)),
\(30.1\%\) on intermediate experts (\(\gamma \in [0.99,0.999)\)), and only
\(13.9\%\) on the slowest experts (\(\gamma \ge 0.999\)). Thus, in our
implementation, the EWLS grid leaves the candidate time scales explicit and
separate, allowing MLpol to place substantial weight directly on short-memory
corrections during the break.

% -----------------------------------------------------------------------------
\subsection{VBAKF diagnostics}
\label{app:adaptive_kf_vbakf}
% -----------------------------------------------------------------------------

\begin{figure}[t]
  \centering
  \includegraphics[width=\linewidth]{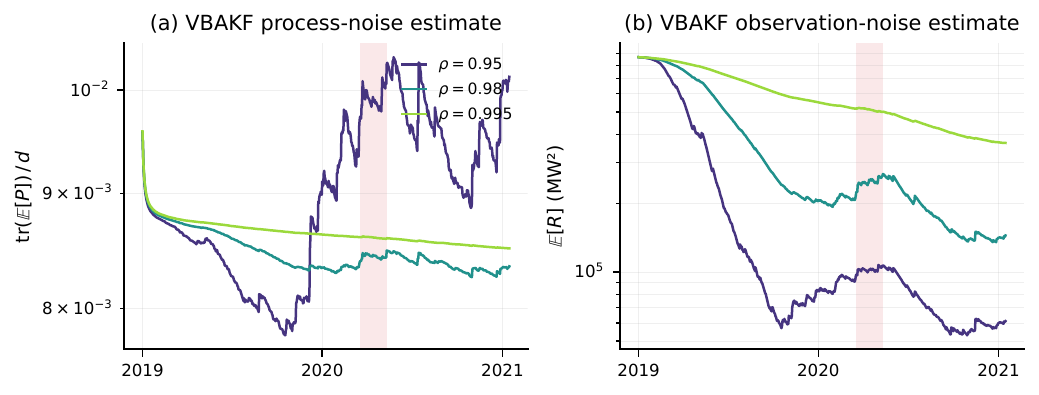}
  \caption{VBAKF's estimated state-uncertainty scale
  \(\mathrm{tr}(\mathbb{E}[\bP])/d\) (\emph{left}) and observation-noise scale
  \(\mathbb{E}[R]\) (\emph{right}) over the test period, for three values of
  the IW-concentration forgetting factor \(\rho\). The lockdown window is
  shaded. The most responsive setting \(\rho=0.95\) shows a pronounced rise in
  state uncertainty around the transition, whereas the two slower settings
  remain nearly flat or slightly lower relative to their pre-lockdown levels.
  The observation-noise estimates also evolve in a strongly
  configuration-dependent way, especially for smaller \(\rho\). Thus VBAKF does
  not expose a common, configuration-consistent signal of increased
  state-evolution uncertainty at the break.}
  \label{fig:vbakf_uncertainty}
\end{figure}

VBAKF shows an additional behaviour that helps interpret its weaker overall
performance in this experiment. Figure~\ref{fig:vbakf_uncertainty} plots the
estimated state-uncertainty scale \(\mathrm{tr}(\mathbb{E}[\bP])/d\) and the
estimated observation-noise scale \(\mathbb{E}[R]\) over the test period. A
natural response to the lockdown break would be a robust increase in
state-evolution uncertainty across responsiveness settings. The fitted VBAKF
variants do not show such a common pattern. The most responsive setting
\(\rho=0.95\) rises sharply around the transition and remains volatile during
and after lockdown, but the two slower settings \(\rho=0.98\) and
\(\rho=0.995\) remain essentially flat or slightly lower relative to the
pre-lockdown regime.

Quantitatively, the lockdown-to-pre-lockdown ratio of
\(\mathrm{tr}(\mathbb{E}[\bP])/d\) is \(1.161\) for \(\rho=0.95\), but only
\(0.987\) for \(\rho=0.98\) and \(0.988\) for \(\rho=0.995\). The corresponding
post-lockdown-to-pre-lockdown ratios are \(1.124\), \(0.979\), and \(0.982\).
Thus the VBAKF state-uncertainty posterior provides a response for the most
responsive configuration, but not a consistent cross-configuration signal of
increased state volatility around the break. At the same time, the estimated
observation-noise scale is far from constant: the coefficient of variation of
\(\mathbb{E}[R]\) is \(1.188\), \(0.670\), and \(0.266\) for
\(\rho=0.95,0.98,0.995\), respectively. This suggests that the adaptive update
partly reallocates uncertainty between state and observation noise in a
configuration-dependent way, which need not align with the squared-loss
objective used by the outer MLpol layer.

One possible mechanism is that the variational covariance update can become
path-dependent. In stable periods, the state updates \(\Delta\theta\) are small
and the posterior covariance after the Kalman update is also small. The
sufficient statistic added to the inverse-Wishart update,
\[
    \bA_{\mathrm{IW}}
    =
    \bP_{\mathrm{new}}
    +
    (\Delta\boldsymbol{\theta})(\Delta\boldsymbol{\theta})^\top .
\]
is therefore small, which can reinforce a tight state posterior. When a break
arrives, the filter may respond through a mixture of state-uncertainty and
observation-noise adaptation, rather than by exposing a clean set of alternative
time-scale predictions to the outer aggregator. This behaviour is consistent
with the empirical results: VBAKF-only is competitive during the short lockdown
window in point estimate, but the Base+VBAKF pool is significantly worse than
MELO overall, before lockdown, and after lockdown.

% -----------------------------------------------------------------------------
\subsection{Synthesis}
\label{app:adaptive_kf_synth}
% -----------------------------------------------------------------------------

The adaptive-filter comparison gives a consistent empirical message on this
benchmark. We consider three representative adaptive-KF alternatives:
isotropic variational adaptation (VIKING), inverse-Wishart covariance/noise
adaptation (VBAKF), and Bayesian multi-model mixing (IMM). Under the same
walk-forward tuning protocol, all three trail \(\gamma\)-grid + MLpol on
overall RMSE. The paired bootstrap confirms that the overall differences are
statistically significant for all three adaptive-KF baselines when their
outputs are combined with the direct base experts by the same outer MLpol layer
(Table~\ref{tab:adaptive_kf_boot}).

The diagnostics suggest three complementary interpretations. First, IMM reacts
to the lockdown transition in its mode probabilities, but its internal Bayesian
mixing exposes only a single combined prediction to the outer aggregator. This
may reduce the diversity that MLpol can exploit. Second, in the VIKING and
VBAKF pools, the outer aggregator does not simply put all lockdown mass on the
most responsive adaptive-\(Q\) configuration. VIKING shifts strongly toward the
more conservative end of its pool, while VBAKF shifts toward slower settings
and away from the direct base experts. This suggests that greater internal
process-noise responsiveness need not translate into lower realised squared
loss. Third, the tuned VBAKF variants do not expose a common,
configuration-consistent signal of increased state-evolution uncertainty at the
break: the most responsive setting reacts strongly, whereas the slower settings
remain nearly flat or slightly lower relative to their pre-lockdown levels.
Part of the adaptation is instead absorbed by the estimated observation-noise
scale, in a configuration-dependent way.

Together, these diagnostics support the design choice made in \melo{}: keep the
adaptation scales fixed and explicit at the expert layer, and use the outer
MLpol layer to select among the resulting predictions according to realised
squared loss. The comparison does not imply that adaptive Kalman filters are
unsuitable in general, nor that no alternative state-space design could close
the gap. It shows that, on this deployment-shift benchmark and under the same
walk-forward tuning protocol, the adaptive-\(Q\) mechanisms considered here do
not match the overall performance of the explicit multi-scale EWLS +
regret-aggregation separation.

% -----------------------------------------------------------------------------
\section{Cross-dataset behavior on TabReD industrial datasets}
\label{app:tabred}
% -----------------------------------------------------------------------------

The empirical analysis in Section~\ref{sec:experiments} and
Appendix~\ref{app:adaptive_kf} is centered on RTE-FR Load, a French national
electricity-load benchmark spanning the COVID-19 lockdown. This appendix
evaluates the framework on two additional industrial regression datasets from
the TabReD benchmark~\citep{rubachev2024tabred}\footnote{
  TabReD comprises eight tabular datasets curated for studying temporal robustness
  with realistic industrial feature pipelines and time-based train/val/test
  splits. See \url{https://github.com/yandex-research/tabred} for code and data.}:
\textbf{Weather} (predicting station temperature, $T_{\text{test}} = 40{,}840$
instances over July 2023) and \textbf{Delivery-ETA} (predicting food-delivery
arrival time, $T_{\text{test}} = 36{,}927$ instances).

The goal of this appendix is not to claim uniformly large gains across all
chronological tabular tasks. Rather, it clarifies the boundary of the framework:
MELO is most useful when two conditions hold simultaneously: (1) there is
exploitable non-stationarity, and (2) the base pool contains residual diversity that
a covariance-aware combination layer can use. The TabReD datasets provide informative
negative-control settings in which the base-model residuals are highly
correlated on early held-out test segments; consistently with this diagnostic,
the EWLS layer remains mildly beneficial but produces much smaller gains than on
RTE-FR Load.

% -----------------------------------------------------------------------------
\subsection{Setup}
\label{app:tabred-setup}
% -----------------------------------------------------------------------------

\paragraph{Base pool.}
For each TabReD dataset we use four base models: LightGBM, XGBoost, MLP, and
MLP-PLR (MLP with periodic-linear numerical embeddings;
\citealp{gorishniy2022embeddings}). Hyperparameters are taken directly from the
TabReD authors' tuned configurations committed in their public repository,
without re-tuning. Each model is trained at three random seeds for Weather and
one seed for Delivery-ETA; predictions are averaged across seeds.

\paragraph{Online evaluation.}
We use TabReD's default time-based split. Base models are trained on the train
split with early stopping on the val split; the test split is then traversed
once in chronological order, with EWLS and MLpol updated after each observation.
Forgetting factors and $\delta_0$ are inherited from the RTE-FR Load main experiment
without modification ($K = 16$, $\gamma \in [0.95, 0.9998]$, equivalently
$h \in [20, 5000]$, $\delta_0 = 10^{-3}$). The covariance-inflation scale
$\varepsilon_0$ is selected by walk-forward validation on the val split using
the protocol of Appendix~\ref{app:eps_selection}.

\paragraph{Patch to the TabReD pipeline.}
The TabReD code base saves test-set predictions to disk by default, but does not
distinguish them by random seed in the prediction filename; we extracted the
per-seed predictions from each evaluation directory's \texttt{predictions.npz}.
No model code or hyperparameters were modified.
% -----------------------------------------------------------------------------
\subsection{Cross-dataset results}
\label{app:tabred-results}
% -----------------------------------------------------------------------------

\begin{table}[ht]
    \centering
    \caption{Cross-dataset summary. RMSE gains are measured on the full test
    period. ``Mean off-diag corr'' is the average off-diagonal entry of the
    residual correlation matrix among base models, computed only on an initial
    diagnostic segment of the chronological test stream: the first 300 test days
    for RTE-FR Load and the first 20{,}000 test instances for the two TabReD
    datasets. These correlations are reported only as a post-hoc diversity
    diagnostic; they are not used for model selection, hyperparameter tuning,
    expert-pool construction, or online updating. RTE-FR Load uses the 7-base
    pool of Section~\ref{sec:experiments}; TabReD entries use the 4-base pool
    described in Section~\ref{app:tabred-setup}.}
    \label{tab:cross-dataset-headline}
    \begin{tabular}{lccc}
        \toprule
        Dataset & Mean off-diag corr & \melo{} gain & $T_{\text{test}}$ \\
        \midrule
        RTE-FR Load                             & \textbf{0.43} & $+34.7\%$ & 746 \\
        TabReD Weather                          & 0.97          & $+0.30\%$ & 40{,}840 \\
        TabReD Delivery-ETA                     & 0.98          & $+0.24\%$ & 36{,}927 \\
        \bottomrule
    \end{tabular}
\end{table}

Table~\ref{tab:cross-dataset-headline} reports the headline finding. \melo{}
improves over MLpol on Base alone in all three datasets we evaluate, but the
magnitude of the improvement varies by two orders of magnitude. This pattern is
not a contradiction of the main RTE-FR Load result; it identifies the regime in
which the correction layer is most useful. In our experiments, large gains appear to require both a non-stationary test
trajectory and a base pool whose residuals contain multiple directions of error. When diagnostic-segment residuals are nearly collinear, as
in the two TabReD datasets, the affine combination space has little additional
signal to exploit, and the observed gain is correspondingly small.

Per-method test-set RMSE for the two TabReD datasets is given in
Table~\ref{tab:tabred-detail}. Both datasets show the same qualitative pattern
as RTE-FR Load: MLpol on Base only improves marginally over a uniform simple
average, and Base+EWLS adds a further small improvement. The relative size of
these improvements is two orders of magnitude smaller than on RTE-FR Load,
consistent with the high diagnostic-segment residual correlations reported in
Table~\ref{tab:cross-dataset-headline}.

\begin{table}[ht]
\centering
\caption{Per-method RMSE on the two TabReD test sets. Lower is better. Bold:
best per column.}
\label{tab:tabred-detail}
\begin{tabular}{lcc}
\toprule
Method & Weather (RMSE) & Delivery-ETA (RMSE) \\
\midrule
Simple average over base pool   & 1.4668          & 0.5449 \\
MLpol on Base only              & 1.4577          & 0.5447 \\
\textbf{\melo{}}                & \textbf{1.4533} & \textbf{0.5434} \\
\midrule
Best base predictor (single)    & 1.4586 (LightGBM) & 0.5467 (XGBoost) \\
$\Delta$ \melo{} vs Base only   & $+0.30\%$       & $+0.24\%$ \\
\bottomrule
\end{tabular}
\end{table}

% -----------------------------------------------------------------------------
\subsection{Residual correlation as a diversity diagnostic}
\label{app:tabred-corr}
% -----------------------------------------------------------------------------

The two-orders-of-magnitude gap between the RTE-FR Load and TabReD gains is
consistent with the amount of base-pool diversity visible in an initial segment
of each chronological test stream. Figure~\ref{fig:cross-dataset-corr} measures
this diversity through pairwise residual correlations of the base models. This
diagnostic is computed after the online protocols, hyperparameters, and expert
pools are fixed. It is not used to select hyperparameters, construct the expert
pool, or update the online aggregation pipeline for the reported comparisons;
it is reported only to explain whether affine combination had a rich enough
signal space to exploit. In a prospective deployment, the analogous diagnostic
would be computed on a held-out pre-deployment validation period or on an
initial calibration stream not used for model selection.

\begin{figure}[ht]
    \centering
    \includegraphics[width=\textwidth]{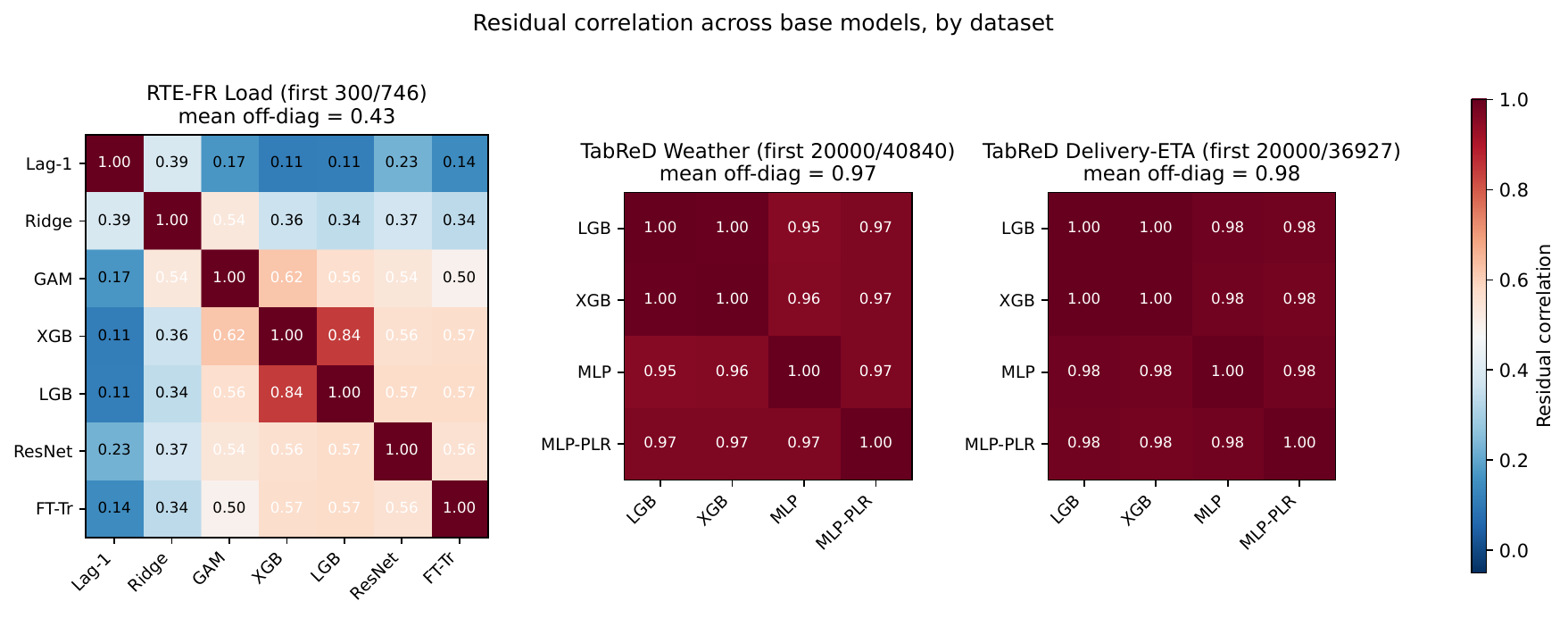}
    \caption{\textbf{Early-window residual correlation matrices across the
    three datasets, sorted left-to-right by base-pool homogeneity.}
    Correlations are computed on an initial diagnostic segment of each
    chronological test stream: the first 300 test days for RTE-FR Load and the
    first 20{,}000 test instances for the two TabReD datasets. Each cell reports
    the Pearson correlation between two base models' residuals (prediction minus
    truth), and panel titles report mean off-diagonal correlation. These
    correlations are post-hoc diagnostics only and are not used by \melo{} or by
    any compared online method. RTE-FR Load exhibits visibly lower and more
    structured residual correlation, whereas TabReD Weather and Delivery-ETA
    remain almost uniformly correlated across base models.}
    \label{fig:cross-dataset-corr}
\end{figure}

RTE-FR Load's 7-base pool has substantially more residual diversity than the
TabReD pools. In the early-window diagnostic, Lag-1 has low correlation with
most feature-based models, and the feature-based models themselves separate
into weaker and stronger correlation blocks. This structural asymmetry is
important: Lag-1 uses the previous-day target as its information source, whereas
the other models map the contemporaneous feature vector to load. The EWLS layer
can exploit such inductive-bias diversity because the time-varying affine
combination $\tilde \bz_t^\top \bw_t^{(\gamma)}$ is useful only when the base
forecasts make meaningfully different errors.

This interpretation matches the role of the comparator in
Theorem~\ref{thm:mlpol-oracle}. The theorem bounds the loss of the aggregate
against a time-varying affine comparator path, but the comparator class is only
empirically useful when the base-prediction vector spans several error
directions. The residual correlation matrix is therefore a simple empirical
proxy for the spanning quality of the base pool.

By contrast, the two TabReD datasets show near-uniformly high residual
correlation. The four base models consist of two GBDT families and two
MLP-based families. They appear to converge to nearly the same mapping from
engineered features to target, leaving residuals dominated by shared
per-instance errors. In this setting, the effective affine combination space is
narrow, so the EWLS layer has little room to improve over base-only
aggregation. These results clarify a boundary of the method rather than
undermining the main RTE-FR Load result. Large gains require both
deployment-time non-stationarity and a base pool with non-redundant error
directions.

\paragraph{Ruling out tuning convergence as the cause.}
A natural concern is that the high TabReD correlations are an artifact of all
four models being hyperparameter-tuned toward the same loss-landscape optimum,
and that under-tuned models would expose hidden diversity. We test this by
training a default-hyperparameter XGBoost model with 100 trees, depth 6, and
learning rate 0.3 on each TabReD dataset, then measuring its validation
residual correlation with the four tuned models. Results are summarised in
Table~\ref{tab:tabred-default-xgb}.

\begin{table}[ht]
\centering
\caption{\textbf{Default-XGBoost probe on the TabReD validation splits.}
Despite materially worse single-model test RMSE, the default model's validation
residuals remain highly correlated with the four tuned models, and adding it to
the pool lowers mean off-diagonal validation correlation by less than $0.01$ in
both datasets.}
\label{tab:tabred-default-xgb}
\begin{tabular}{lcc}
\toprule
& Weather & Delivery-ETA \\
\midrule
Default-XGB test RMSE          & 1.5643 & 0.5534 \\
Best tuned-model test RMSE     & 1.4614 (LightGBM) & 0.5467 (XGBoost) \\
\midrule
Default-XGB vs LightGBM corr   & 0.944 & 0.986 \\
Default-XGB vs XGBoost corr    & 0.944 & 0.986 \\
Default-XGB vs MLP corr        & 0.905 & 0.973 \\
Default-XGB vs MLP-PLR corr    & 0.912 & 0.969 \\
\midrule
Mean off-diag, 4 tuned bases    & 0.946  & 0.983 \\
Mean off-diag, +default-XGB (5) & 0.938  & 0.981 \\
$\Delta$                       & $-0.008$ & $-0.002$ \\
\bottomrule
\end{tabular}
\end{table}

The default-XGBoost probe argues against the tuning-convergence hypothesis. An
explicitly under-tuned model whose test RMSE is $7\%$ worse on Weather and
$1.2\%$ worse on Delivery-ETA still produces validation residuals close to
those of the tuned models. The high TabReD correlations therefore appear to
reflect a structural property of the datasets and feature pipelines: the
heavily engineered features leave little room for inductive-bias diversity,
regardless of how each model is configured.

% -----------------------------------------------------------------------------
\subsection{A diagnostic for expected aggregation gain}
\label{app:tabred-diagnostic}
% -----------------------------------------------------------------------------

The empirical regularity in Figure~\ref{fig:cross-dataset-corr} and
Table~\ref{tab:cross-dataset-headline} suggests a practical diagnostic, but the
figures reported above should be read as post-hoc diagnostics because they use
initial segments of the test streams. The online evaluation itself is unchanged:
the full test stream is traversed once chronologically, and no residual-correlation
statistic enters the online updates. In a prospective deployment, the analogous
quantity would be computed on a held-out pre-deployment validation period or on
an initial calibration stream not used for model selection. Low residual
correlation indicates that the pool contains multiple error directions, making
affine combination more promising. High residual correlation indicates that
the base models make nearly the same errors, so \melo{} has limited room
to improve over base-only aggregation.

These results suggest a qualitative diagnostic. Mean off-diagonal correlation above $0.95$ suggests
that large gains from affine combination are unlikely unless the deployment
period introduces a new source of residual diversity. In this regime, the
framework may still be useful as a conservative adaptation layer, but the
expected improvement over base-only aggregation should be modest. Mean
off-diagonal correlation below $0.7$ suggests that the base pool contains
several distinct error directions, so larger gains are plausible, particularly
when the deployment period also involves non-stationarity.

This diagnostic complements the path-length analysis in
Theorem~\ref{thm:rls-dynamic}. Empirical gains depend on both the variability of
the useful comparator path and the spanning quality of the base-prediction
vector. Either condition can be the binding constraint. RTE-FR Load combines
residual diversity with a severe deployment-time shift; the two TabReD datasets
have highly correlated residuals, and the corresponding gains are small.

\paragraph{Limitations of the diagnostic.}
The diagnostic is directional, not a calibrated estimator. The gain magnitude
depends on the joint distribution of base errors, the realised deployment path,
and the losses used by the online aggregation layer, not only on marginal
pairwise correlations. Three datasets span a useful range but do not identify a
precise functional relationship. A more refined diagnostic would use the
spectrum of the residual covariance matrix, i.e., the effective number of
independent error directions; on RTE-FR Load this spectrum has approximately
four non-trivial directions, whereas on TabReD it concentrates near one.

% -----------------------------------------------------------------------------
\subsection{Discussion}
\label{app:tabred-discussion}
% -----------------------------------------------------------------------------

The cross-dataset comparison is consistent with the intended deployment profile
of the framework. In the RTE-FR Load stress test, where residual drift is
exploitable and the base pool contains diverse error directions, \melo{}
delivers large gains through the COVID-19 lockdown. On the two TabReD datasets
considered here, where residuals are highly correlated, the same mechanism has
limited room to improve over base-only aggregation.

This should be read as a boundary condition rather than a failure mode. The
oracle inequality gives a conservative comparison with the original base
experts and the added EWLS experts up to the aggregation overhead; it does not
imply finite-sample monotonic improvement whenever experts are added. 
Empirically, Base+EWLS is positive relative to Base-only on all three datasets
considered here, but the TabReD results show that the magnitude of the benefit
is governed by detectable residual diversity. The diagnostic of
Section~\ref{app:tabred-diagnostic} therefore helps set realistic expectations
before deployment when computed on a genuine pre-deployment validation or
calibration period.

% -----------------------------------------------------------------------------
\section{Robustness to base-pool diversity: a heterogeneity sweep}\label{app:heterogeneity}
% -----------------------------------------------------------------------------

The main experiment (Section~\ref{sec:experiments}) uses $F=7$ base
predictors spanning linear, tree, additive, and deep families---an
intentionally heterogeneous pool. A natural question is how the
aggregation gain reported there depends on this heterogeneity. We
answer it here by sweeping the number of distinct families in the pool
and enumerating, at each level, every possible subset.

\paragraph{Setup.}
For each \emph{heterogeneity level} $H \in \{1, 2, \dots, 7\}$, we
enumerate all $\binom{7}{H}$ size-$H$ subsets of the seven families
(127 subsets in total). For each subset we rerun the complete
aggregation pipeline: MLpol on the base subset only, MLpol on the
EWLS pool constructed from that subset, and MLpol on the combined
Base+EWLS pool. All other hyperparameters ($\gamma$ grid,
$\varepsilon_0$, $\alpha$) and the test period
(2019-01-01 to 2021-01-15) are identical to the main experiment. For
each $(H,$ method$)$ we report the mean and the $2.5\%$--$97.5\%$
percentile band of overall RMSE across the $\binom{7}{H}$ subsets at
that level, so the shown spread captures the sensitivity of the
result to \emph{which} families happen to be in the pool.

\paragraph{RMSE vs.\ heterogeneity.}
Figure~\ref{fig:het_rmse} reports the three MLpol variants (Base only,
EWLS only, Base+EWLS) and the best standalone single-$\gamma$ EWLS
(hindsight reference) as a function of $H$. The headline pattern is
sharp:
\begin{itemize}[leftmargin=1.4em,itemsep=1pt,topsep=2pt]
    \item MLpol on Base only needs pool heterogeneity to work. At
    $H = 1$ (one family, one predictor) it is trivially the best single
    base and sits at mean RMSE $\approx 1900$\,MW; adding families
    drives it down monotonically, reaching $1004$\,MW at $H = 7$ (the
    main experiment).
    \item Both EWLS-augmented methods are already within $110$\,MW of
    their $H = 7$ value by $H = 2$, and within $35$\,MW by $H = 3$.
    The marginal return to further heterogeneity is small.
    \item Already at H=2, a Base+EWLS pool outperforms the full seven-family Base-only
    pool.
    \item The $p2.5$--$p97.5$ spread shrinks steeply with $H$: at
    $H = 1$, the choice of family determines whether overall RMSE is
    $750$ or $3200$\,MW; by $H = 4$ every subset sits within $100$\,MW
    of the mean. Thus small pools are sensitive to which model family
    happens to be chosen, whereas larger pools are much more stable.
    Since increasing $H$ only adds base forecasts to the MLpol pool and
    has negligible computational cost relative to fitting the base
    models, this variability favours using as broad a base pool as is
    available.
\end{itemize}

\begin{figure}[ht]
\centering
\includegraphics[width=0.88\linewidth]{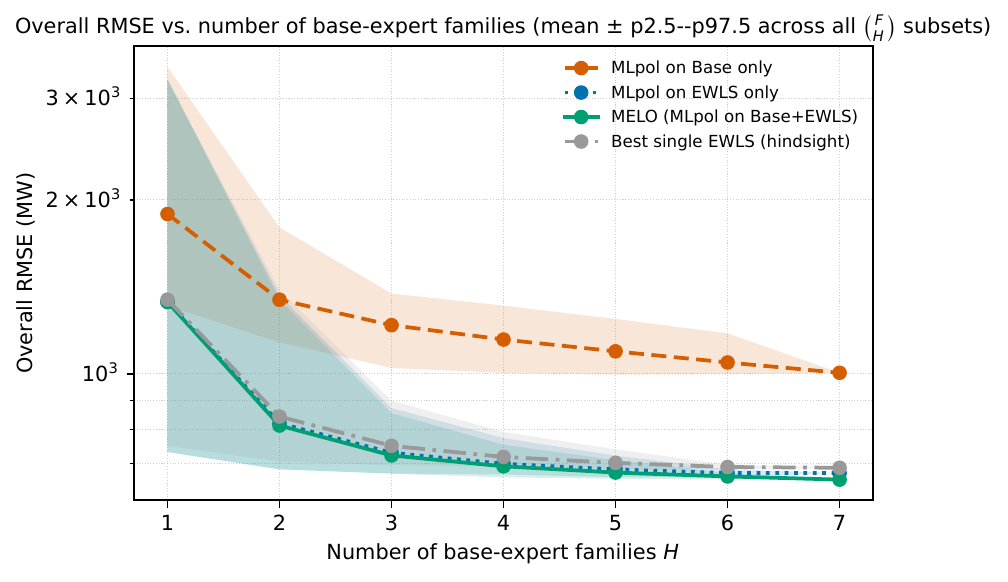}
\caption{Overall RMSE as a function of the number of base-expert
families $H$ (log scale). Each curve is the mean over all
$\binom{7}{H}$ size-$H$ subsets at that level; the shaded band is the
$2.5\%$--$97.5\%$ percentile range across subsets. MLpol on Base only
(vermilion) needs pool heterogeneity to become competitive; both
EWLS-augmented methods (blue, green) plateau within $35$\,MW of their
$H=7$ value by $H=3$. The grey dash-dot line is the hindsight-best
standalone single $\gamma$ per subset, shown as a reference for what
a ``best single memory length'' would achieve without aggregation.}
\label{fig:het_rmse}
\end{figure}

\paragraph{The two sources of diversity are substitutable, not
additive.}
Figure~\ref{fig:het_gain} decomposes the aggregation gain into three
pairwise contrasts. Two observations:

(i)~\emph{EWLS substitutes for family heterogeneity, monotonically.}
The gain Base+EWLS~$-$~Base-only is $568$\,MW at $H = 1$ and falls
monotonically to $325$\,MW at $H = 7$ (the main-experiment value).
The EWLS-only vs.\ Base-only gain traces the same curve, offset by a
near-constant small amount. This is the core substitutability signal:
as family heterogeneity accumulates in the base pool, the marginal
value of adding the EWLS layer declines---precisely because the
base pool itself is now contributing the diversity EWLS would
otherwise provide.

(ii)~\emph{Base experts contribute an orthogonal signal of roughly
constant size.} The gain Base+EWLS~$-$~EWLS-only is $\approx 12$\,MW
for all $H \geq 2$, with $p2.5$--$p97.5$ bands tightly concentrated
around the mean. The base-expert channel contributes a small but
persistent signal that is not captured by any convex combination of
single-memory EWLS experts, and whose magnitude is essentially
independent of how many families are in the pool.

Together (i) and (ii) characterise a clear trade-off: \emph{most} of
the value of the two-layer framework at low $H$ comes from EWLS (it
can rescue a homogeneous pool), and a roughly constant residual value
at every $H$ comes from the complementary base channel (orthogonal to
memory-length aggregation).

\begin{figure}[ht]
\centering
\includegraphics[width=\linewidth]{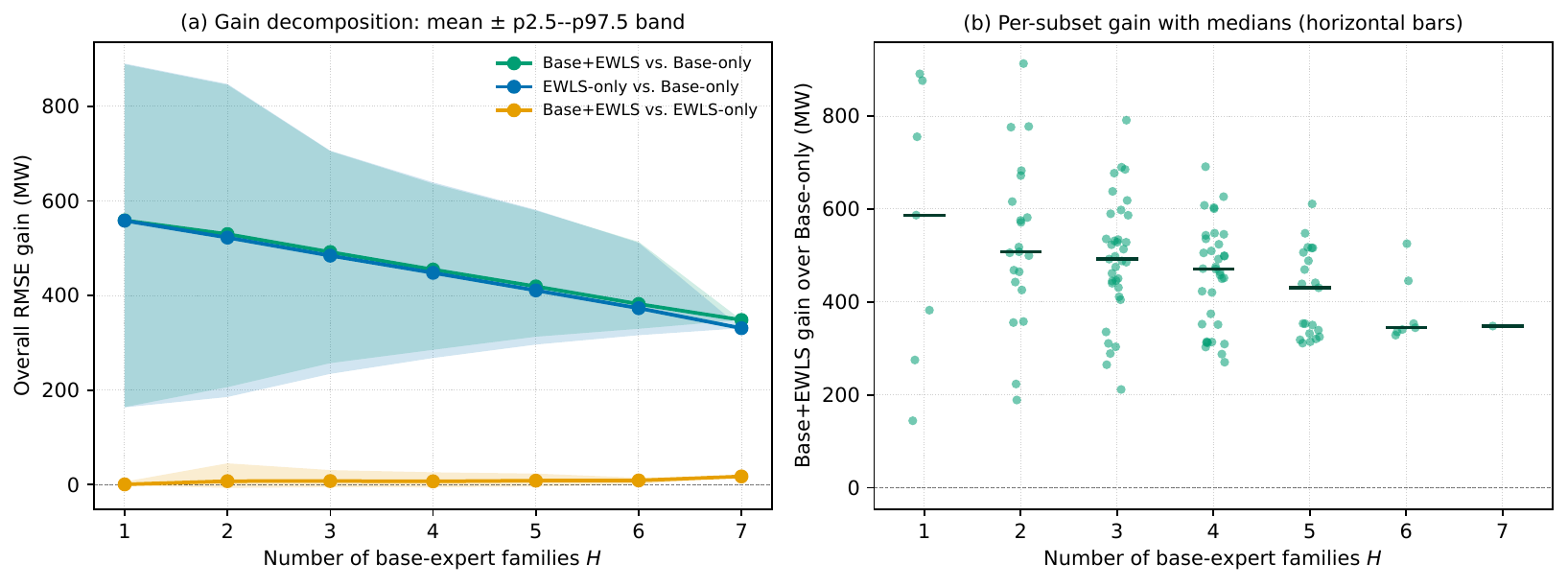}
\caption{\emph{(a)} Decomposition of the aggregation gain into three
pairwise contrasts as a function of $H$. Base+EWLS vs.\ Base-only
(green) and EWLS-only vs.\ Base-only (blue) trace a near-identical
monotone decline from $\sim\!570$\,MW at $H = 1$ to $\sim\!320$\,MW at
$H = 7$, the signature of substitutability between EWLS and family
heterogeneity. Base+EWLS vs.\ EWLS-only (orange) is flat at
$\approx 12$\,MW for all $H \geq 2$: base experts contribute a small
orthogonal signal whose magnitude is independent of pool diversity.
\emph{(b)} Per-subset values of the Base+EWLS gain (green scatter) with
per-$H$ medians (horizontal bars). Positive at every subset we
evaluate.}
\label{fig:het_gain}
\end{figure}

\paragraph{Family heterogeneity does not replace adaptation under drift.} 
Figure~\ref{fig:het_regime} splits the results by regime. The
pre-lockdown and post-lockdown panels show the same
diminishing-returns pattern as the overall metric: Base-only catches
up as $H$ grows, while Base+EWLS is close to asymptote by $H = 3$.
The lockdown panel is qualitatively different---in the structural
break, Base+EWLS RMSE is essentially \emph{flat} in $H$ (median
$1151$\,MW at $H=2$; $1099$\,MW at $H=7$, a $5\%$ reduction), while
Base-only plateaus at around $2400$\,MW even at $H = 7$. This matches
the theoretical picture of Section~\ref{sec:theory}:
under heavy drift, static base experts tend to fail together even when diverse, 
so adding families contributes no tracking
capability---only
the EWLS layer provides explicit online adaptation. Family heterogeneity is a
partial remedy for poor stable-regime predictions; it is not a
substitute for multi-time-scale adaptation under drift.

\begin{figure}[ht]
\centering
\includegraphics[width=\linewidth]{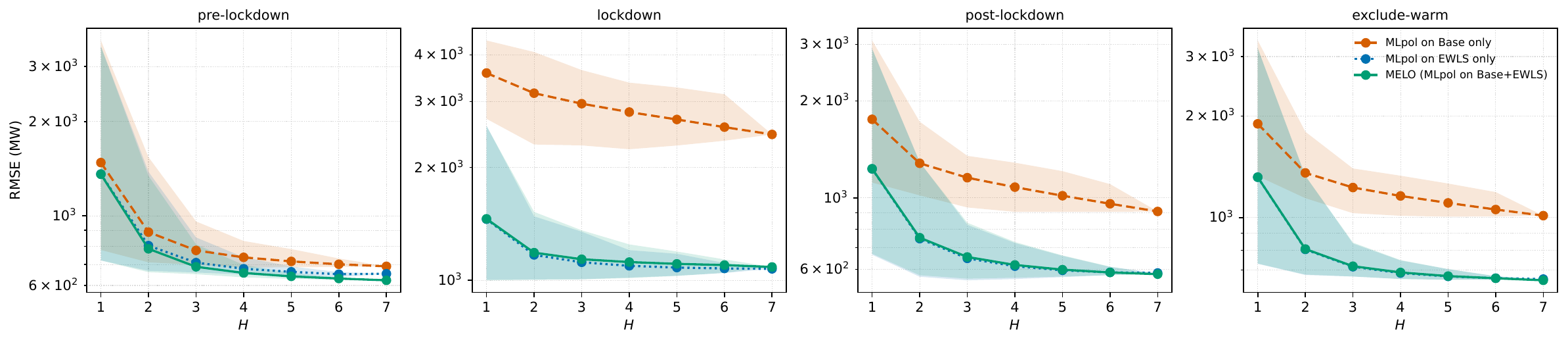}
\caption{Per-regime RMSE as a function of $H$ (log scale). In stable
regimes (pre- and post-lockdown), Base-only catches up as
heterogeneity grows. Under the lockdown structural break, the
Base+EWLS curve is nearly flat in $H$ while Base-only plateaus at
$\sim\!2400$\,MW: the drift-tracking contribution of the framework
comes from EWLS and cannot be replaced by pool heterogeneity.}
\label{fig:het_regime}
\end{figure}

\paragraph{Summary.}
Sweeping the $F=7$ base families yields three findings that together
characterise how the two-layer framework interacts with base-pool
construction:

\begin{enumerate}[leftmargin=1.4em,itemsep=1pt,topsep=2pt]
  \item \textbf{EWLS and family heterogeneity are partial
    substitutes.} The main-experiment Base+EWLS~$-$~Base-only gain of
    $325$\,MW grows monotonically to $568$\,MW as the pool narrows to
    a single family. EWLS provides an alternative source of diversity, expressed along the
    time-scale axis rather than the model-family axis.
  \item \textbf{A homogeneous EWLS-augmented pool outperforms a
    heterogeneous static one.} A two-family Base+EWLS pool already
    beats the seven-family Base-only pool overall. Practitioners
    without access to a rich model zoo can still obtain most of the
    benefit reported in the main experiment from a small pool, as
    long as the EWLS layer is present.
  \item \textbf{Base experts contribute a small, approximately $H$-independent complementary signal.} Adding base experts to an EWLS-only pool improves RMSE
    by $\approx 12$\,MW at any pool diversity, so the two layers are
    complementary rather than redundant---even when diversity is
    abundant.
\end{enumerate}
The headline aggregation gain reported in the main experiment is
therefore not driven by the structural richness of the heterogeneous
base pool: 
it is present at every pool composition tested, and is especially large
in the low-diversity configurations and drift regimes where static
predictors struggle most.
Because the identity of the chosen family matters substantially at
small $H$, and because admitting additional available families is cheap
at aggregation time, the sweep also supports the practical default of
using the broadest causal base pool available.

% -----------------------------------------------------------------------------
\section{Weight dynamics: per-expert details}
\label{app:weight_details}
% -----------------------------------------------------------------------------

Section~\ref{sec:adaptation_diagnostics} presented the MLpol weight trajectory with
base experts aggregated into a single band and EWLS experts shown as
coloured strata, which emphasises the \emph{across-pool}
base/EWLS reallocation. Two further phenomena are important for
interpreting the regret dynamics in Figure~\ref{fig:regret_and_weights}
but cannot be read off that view: (i)~the \emph{within-base}
reallocation, where Lag-1 absorbs much of the remaining base-expert
mass under the break, and (ii)~the \emph{within-EWLS}
reshuffling, where the weight mass moves across time scales from slow
to fast and back. 

% ---------------------------------------------------------------------
\paragraph{Fully disaggregated view.}
Figure~\ref{fig:weights_full_detail} shows the individual weight
trajectory of all $N\!=\!M\!+\!K\!=\!23$ experts, with base experts
coloured individually and EWLS experts coloured by $\gamma$ on the
RdYlBu ramp. The three macro-level claims of
Section~\ref{sec:adaptation_diagnostics}---base dominance in pre-lockdown, abrupt
lockdown reallocation, and medium-memory dominance in the recovery---are
all visible, but two finer patterns become apparent as well.
First, the within-base mass at lockdown onset is not simply spread
uniformly across the original leaders: Lag-1, which had negligible
weight before the break, becomes the dominant member of the remaining
base allocation. Second, within the EWLS strata the reddest bands
(smallest $\gamma$) widen abruptly at lockdown onset at the expense of
the bluest bands (near-static $\gamma\!\to\!1$), consistent with the
optimal memory length $h^{\ast}$ shrinking under drift
(Theorem~\ref{thm:rls-dynamic}). We unpack each of these two axes next.

\begin{figure}[t]
  \centering
  \includegraphics[width=\linewidth]{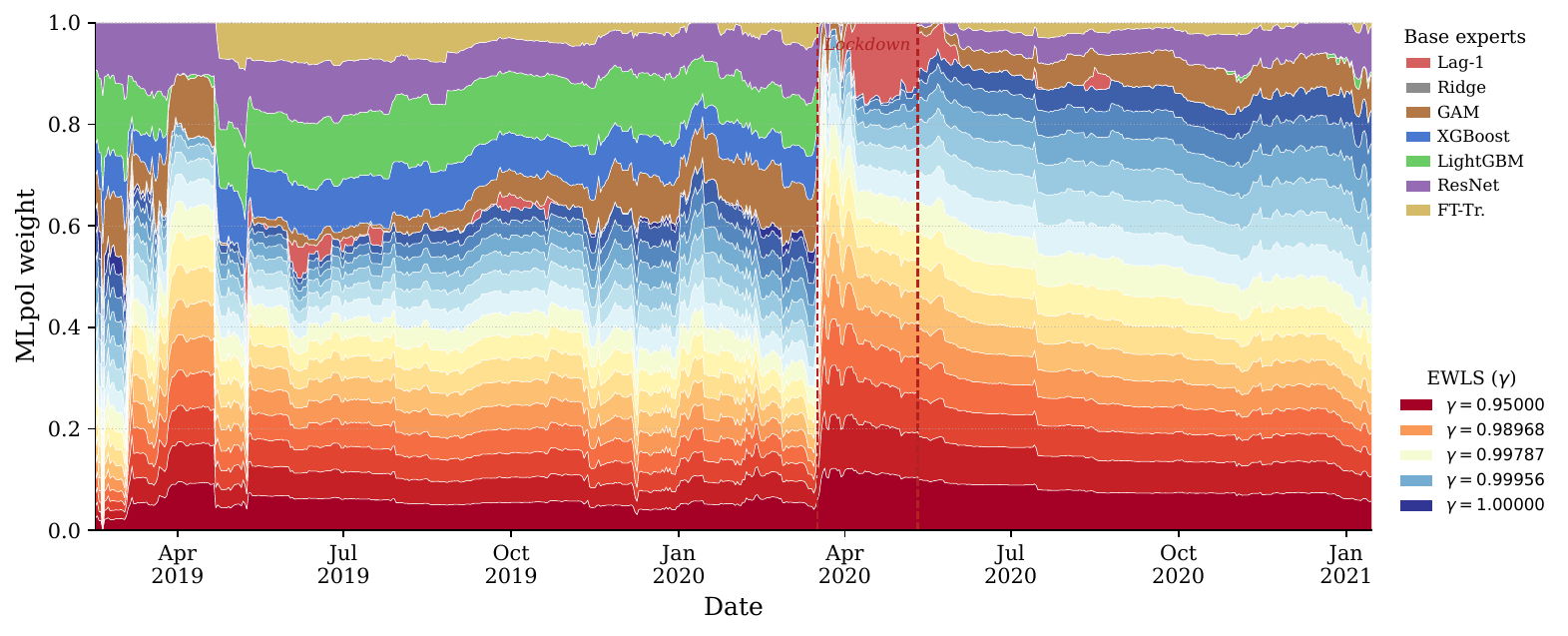}
  \caption{\textbf{Fully disaggregated MLpol weight trajectory.} All
    $N\!=\!M\!+\!K\!=\!23$ experts are shown as individual strata.
    Base experts (bottom of the stack) are coloured individually; EWLS
    experts (top) are coloured by forgetting factor $\gamma$ on the
    RdYlBu ramp (red $=$ fast forgetting, $\gamma\!=\!0.95$; blue $=$
    slow / near-static, $\gamma\!\to\!1$). Dashed red vertical lines
    mark the COVID-19 lockdown window (2020-03-17 to 2020-05-11). The
    plot begins on 2019-02-15 for readability, omitting the initial
    uniform-weight transient. The EWLS legend shows five representative
    $\gamma$'s; all $K\!=\!16$ values appear in the plot.}
  \label{fig:weights_full_detail}
\end{figure}

% ---------------------------------------------------------------------
\paragraph{Within-base reallocation: Lag-1 takes over during the break.}
Figure~\ref{fig:base_pool_detail} isolates the $M\!=\!7$ base experts.
The top panel shows their raw MLpol weight; the bottom panel
normalises to the base sub-simplex (each expert's share of the total
base weight), which removes the confounding effect of base-pool
shrinkage during the lockdown and isolates \emph{internal} reallocation.

\begin{figure}[t]
  \centering
  \includegraphics[width=\linewidth]{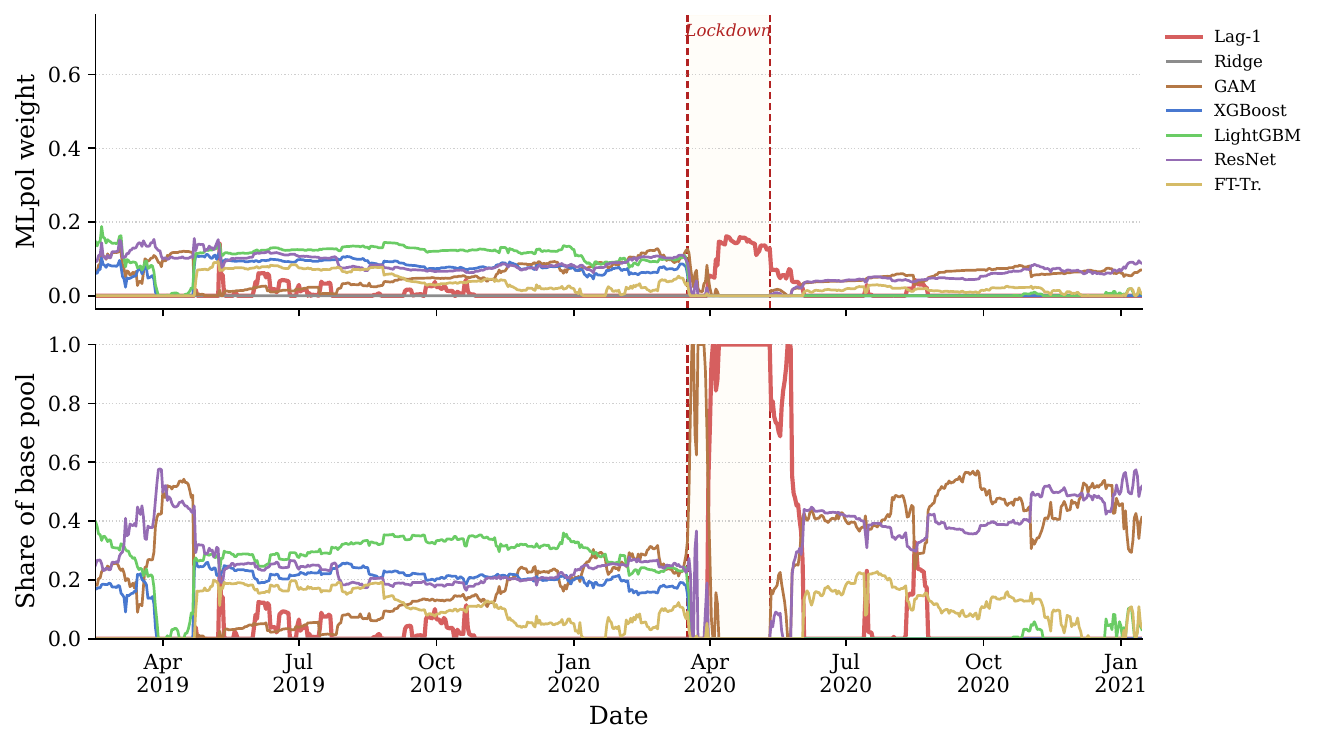}
  \caption{\textbf{Within-base reallocation.} \textit{Top:} raw MLpol
    weight of each of the $M\!=\!7$ base experts. \textit{Bottom:} the
    same weights normalised to sum to one across the base pool, so
    one reads each expert's share of the base allocation independently
    of the base-pool total. The first six weeks covering the
    uniform-weight warm-up transient are clipped. Dashed red vertical
    lines mark the COVID-19 lockdown window. Under the structural
    break the structurally-trained base experts
    (GAM/XGBoost/LightGBM/ResNet/FT-Transformer) collapse jointly
    while Lag-1, which had been receiving essentially zero weight in
    pre-lockdown, takes over the majority of the base sub-simplex.}
  \label{fig:base_pool_detail}
\end{figure}

Three regularities stand out:
\begin{itemize}[leftmargin=1.4em,itemsep=1pt,topsep=2pt]
  \item In the stable pre-lockdown regime, the base allocation is
    dominated by the structurally trained models, roughly in the order
    LightGBM $>$ ResNet $>$ XGBoost $>$ GAM, each receiving
    6--13\% of the total MLpol weight. Lag-1
    and Ridge sit at essentially zero.

  \item During lockdown, the structurally trained base experts collapse
    jointly: GAM, XGBoost, LightGBM, ResNet, and FT-Transformer fall
    from a combined 41.3\% pre-lockdown weight to 2.4\% during
    lockdown. Lag-1 moves in the opposite direction, rising from
    0.5\% to 9.7\%. Thus the base channel does not disappear entirely;
    rather, most of the remaining base-expert mass is reassigned to
    the one-step autoregressive baseline. On the base sub-simplex,
    Lag-1 accounts for roughly 80\% of the lockdown base allocation.
    This is consistent with the role of Lag-1 as an extremely
    short-memory raw expert, analogous in the base channel to the
    fast-forgetting EWLS experts in the correction channel.

  \item Post-lockdown, the structurally trained experts partially
    recover, but not to their pre-lockdown levels: their combined
    weight rises to 12.0\%, compared with 41.3\% before lockdown.
    GAM and ResNet recover most visibly, while XGBoost and LightGBM
    remain near zero. The base pool has learned, in effect, that its
    offline-trained members carry less reliable structure after the
    regime shift.
\end{itemize}

% -----------------------------------------------------------------------------
\section{Extensibility: signal-informed forecasts as additional experts}
\label{app:extensibility}
% -----------------------------------------------------------------------------

The main experiments deliberately use frozen pre-lockdown base models and no
external regime signal, in order to isolate the effect of the online correction
layer. The framework is not tied to this particular base pool. Any causal
forecast whose loss can be evaluated online can be added as an expert in the
final MLpol layer, and can optionally be included in the EWLS design vector.

This gives a simple oracle-inequality interpretation of extensibility. Suppose
that a new causal predictor \(m_X\) is added to an existing pool of \(N\)
experts. Applying Lemma~\ref{lem:mlpol-regret} to the enlarged pool gives an
aggregate \(\hat y^{N+1}\) satisfying
\[
    \sum_{t=1}^T \ell_t(\hat y^{N+1}_t)
    \;\le\;
    \sum_{t=1}^T \ell_t(m_X(t)) + G\sqrt{(N+1)T},
\]
where \(G\) depends on the clipping radius, under the same boundedness conditions as in the
MLpol oracle inequality. Thus adding a signal-informed forecast gives the
aggregate a new comparator, at the cost of increasing the finite-pool overhead
from order \(\sqrt{NT}\) to order \(\sqrt{(N+1)T}\). This does not imply
finite-sample monotonic improvement over the previous pool, but it means that
useful auxiliary signals can be incorporated without changing the algorithmic
structure.

Table~\ref{tab:tabicl-extension} tests this extensibility by adding a strong
daily-retrained TabICL forecast augmented with the Oxford Government Response
Index (TabICL+GRI). This is a stronger-information setting than the main
experiment, so we report it as a diagnostic rather than as the primary
comparison. Three increasingly rich integrations all reduce overall RMSE:
adding TabICL+GRI to a base-only MLpol pool reduces RMSE from \(1004.03\) to
\(669.14\); adding it only as an extra raw expert to the Base+EWLS pool reduces
RMSE from \(655.78\) to \(621.06\); and additionally including it in the EWLS
design vector reaches \(601.86\), with the largest gain during lockdown.

\begin{table}[ht]
    \caption{Extensibility experiment with a regime-informed foundation-model
    forecast. TabICL+GRI denotes the daily-retrained TabICL forecast augmented
    with the Oxford Government Response Index. ``Raw only'' adds TabICL+GRI only
    as an expert in the final MLpol pool. ``Via EWLS'' also includes TabICL+GRI
    among the inputs to the EWLS correction layer. Lower is better.}
    \label{tab:tabicl-extension}
    \centering
    \small
    \setlength{\tabcolsep}{4.5pt}
    \begin{tabular}{lrrrr}
        \toprule
        Method & Pre-lockdown & Lockdown & Post-lockdown & Overall \\
        \midrule
        MLpol on Base only                         & 690.54 & 2452.70 & 907.04 & 1004.03 \\
        TabICL+GRI                                 & 667.96 & 1132.62 & 699.46 & 723.43 \\
        MLpol on Base + TabICL+GRI                 & 607.81 & 1192.50 & 606.08 & 669.14 \\
        MLpol on Base+EWLS                         & 623.10 & 1086.07 & 579.25 & 655.78 \\
        MLpol on Base+EWLS + TabICL+GRI raw only   & 594.90 & 980.20  & 559.19 & 621.06 \\
        MLpol on Base+EWLS + TabICL+GRI via EWLS   & 582.98 & 895.87  & 550.27 & 601.86 \\
        \bottomrule
    \end{tabular}
\end{table}

This supports the intended deployment interpretation of \melo{}: it is not a
closed ensemble of weak static models, but an online adaptation layer around any
available causal forecasts. Stronger base forecasts can be incorporated
directly, and the EWLS layer can still exploit residual structure in their joint
prediction vector.

\end{document}